\documentclass[11pt]{article}

\usepackage[final]{acl}

\usepackage{times}
\usepackage{latexsym}

\usepackage[T1]{fontenc}

\usepackage[utf8]{inputenc}

\usepackage{microtype}

\usepackage{inconsolata}

\usepackage{graphicx}
\usepackage{amsmath}
\usepackage{amssymb}
\usepackage{paralist}
\usepackage{amsthm}
\usepackage{thm-restate}
\usepackage{thmtools}
\usepackage{algorithm}
\usepackage{algorithmic}
\usepackage{booktabs}
\usepackage{multirow}

\newtheorem{proposition}{Proposition}

\makeatletter
\newcommand{\qualification}[1]{\ifthmt@thisistheone#1\fi}
\makeatother

\title{Mind Which Bird You Favour: Parameterizing Adequacy–Fluency Balance in Meta-Evaluation of Machine Translation}

\author{Behzad Shayegh \and
  Niloofar Kazemi \\
  Independent Researchers}

\begin{document}
\maketitle
\begingroup
\renewcommand\thefootnote{$\star$}
\footnotetext{Correspondence to: \href{mailto:the.shayegh@gmail.com}{the.shayegh@gmail.com}}
\endgroup
\begin{abstract}
There is a tradeoff in machine translation meta-evaluation between prioritizing alignment with adequacy versus fluency. The balance depends on the combination of translation systems in the meta-evaluation dataset. This system set is a small, filtered sample whose characteristics change heavily across years and language pairs; it does not represent the true system distribution. Consequently, the adequacy–fluency balance is often unrepresentative and subject to change. For sensitive domains, controlling this balance is critical. We expose this balance as a tunable choice. To achieve a target balance, we reweight existing systems while minimizing distortion from uniform weighting, ensuring the evaluated systems remain real and representative. We provide an exact optimization algorithm with theoretical guarantees and pruning mechanisms to compute these weights. To validate meta-evaluation internal consistency, we design a scorer-augmentation framework that establishes a known relative identity for the scorers. Results demonstrate that our reweighting method effectively controls the adequacy–fluency balance and preserves the internal consistency, outperforming prior approaches. Finally, we analyze the performance of popular scorers across a sweep of this parameter.\footnote{Code available at \href{https://github.com/TheShayegh/adequacy-fluency-param}{https://github.com/TheShayegh/adequacy-fluency-param}}
\end{abstract}

\section{Introduction}

Machine translation (MT) is typically judged by automatic metrics, which we refer to as \emph{scorers}.\footnote{We use the term ``scorer'' instead of ``metric'' to avoid ambiguity with meta-metrics.} Deciding which scorer to trust is itself an evaluation problem known as \emph{meta-evaluation}~\citep{callison-burch-etal-2007-meta}. Concretely, meta-evaluation fixes a set of translation systems and scores them once with a human protocol taken as the ground truth~\citep[Multidimensional Quality Metrics, MQM;][]{freitag-etal-2021-experts}, and once with the scorer under evaluation. It then asks how well the scorer's system ranking agrees with that of the human, as measured by a \emph{meta-metric} such as pairwise accuracy~\citep[PA;][]{kocmi-etal-2021-ship} or soft pairwise accuracy~\citep[SPA;][]{thompson-etal-2024-improving}. The WMT~metrics shared task \citep{freitag-etal-2023-results} performs this meta-evaluation every year, and it sits at the top of the MT stack: it identifies the best scorers, which subsequent research adopts and reports, in turn shaping the next generation of MT systems. Both the scorers and the MT systems therefore silently inherit whatever bias the meta-evaluation imposes.

\begin{figure}[t]
\begin{center}
\includegraphics[width=1\linewidth]{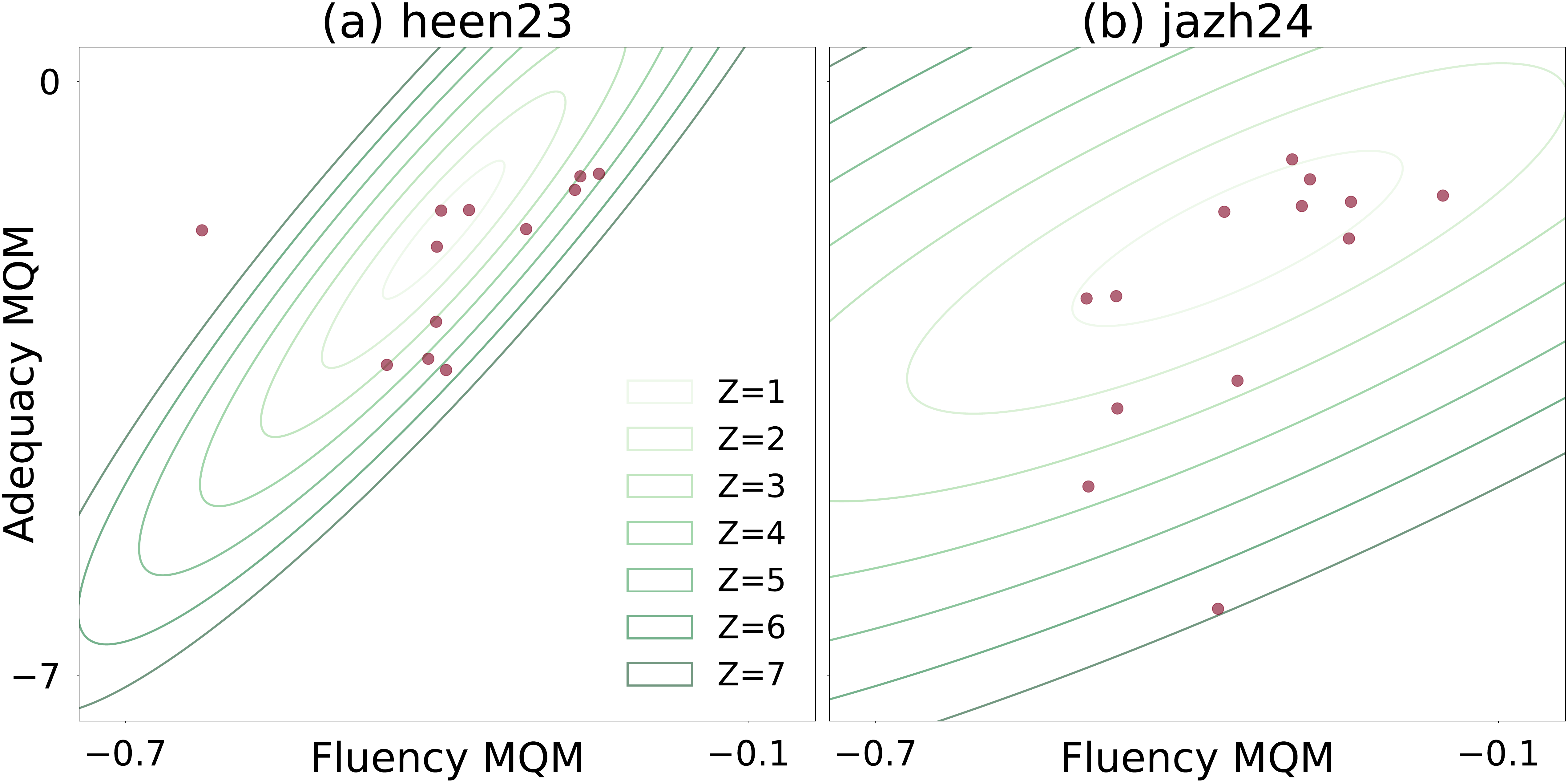}
\end{center}
\vspace{-8pt}
\caption{System-level Adequacy vs.\ Fluency MQM for (a)~heen23 and (b)~jazh24 WMT datasets. Each dot is an MT system. Green contours are robust Mahalanobis-distance levels~\citep{gnanadesikan-kettenring-1972}.}
\label{fig:1}
\vspace{-5pt}
\end{figure}

The MQM ground truth is a multi-component quantity~\citep{freitag-etal-2021-experts}. A translation is penalized both for {adequacy} errors (failing to convey the meaning of the source) and for {fluency} errors (being ungrammatical or unnatural). \citet{flamich2025feedbirdsscoreaccuracynaturalness} show that these two aspects trade off against each other, and \citet{shayegh-etal-2025-feeding} show that the tradeoff reaches the {meta-evaluation} level: the overall human ranking is dominated by whichever aspect varies more across the systems being compared, so the scorers judged best are those aligned with that aspect. In other words, if the systems differ mostly in adequacy, human rankings will be driven primarily by adequacy as well; consequently, metrics sensitive to adequacy will perform better in the meta-evaluation. Crucially, this lean is heavily driven by {system selection}~\citep{shayegh-etal-2025-feeding}.

This is especially concerning because the included systems are a small, filtered subset, thus likely a biased sample of the true, unknown distribution. Figure~\ref{fig:1} shows that the lean differs substantially across WMT datasets. Consequently, meta-evaluation produces uncontrolled scorer rankings, risking the endorsement of undesired scorers that misdirect future model development. This motivates the need for an explicit control mechanism.

\citet{shayegh-etal-2025-feeding} attempt to address this problem through the lens of ``{debiasing}.'' They frame system-selection imbalance as undesirable {extrinsic bias} and introduce a system-synthesis method to reduce it. However, not all system-selection imbalance is undesirable; only its deviation from the true system distribution is. As the true distribution remains unknown, this attribute is unmeasurable, leading us to abandon debiasing in favor of control. Furthermore, we show that the system synthesis produces unrepresentative systems, yielding a distorted meta-evaluation.

In this work, we expose the previously uncontrolled adequacy--fluency balance as a deliberate, tunable choice. Specifically, we parameterize the evaluation pool using the ratio of weighted marginal variances of adequacy and fluency, a parameter we name $\beta$. We treat the systems in the MQM data as a biased sample and reweight them to achieve a target $\beta$ while minimizing distortion from uniform weighting. Because this reweighting is applied at the system level, the evaluated systems remain real and representative, preserving the meta-evaluation's original characteristics. This approach is an instance of importance sampling~\citep{tokdar2010importance} with minimal weight distortion~\citep{deville-sarndal-1992}.

The reweighting is a nontrivial optimization problem, for which we design an exact algorithm, pruning mechanisms that improve efficiency, and certificates that verify optimality. We prove the soundness and completeness of the algorithm, pruning mechanisms, and certificates, and provide an empirical computational-cost analysis.

Validating meta-evaluation is challenging without a ground truth. To address this, we introduce a scorer-augmentation framework that injects controlled signals into real base scorers, preserving the inherent statistical signatures of real scorers while providing known \emph{relative identity}. Using these augmented scorers, we assess a meta-evaluation's (1)~adequacy-over-fluency preference, (2)~MQM adherence, and (3)~explainability by MQM. The latter measures noise avoidance while retaining the natural scorer--MQM relationship. Our reweighting method effectively controls the first measure while preserving the other two, outperforming \citet{shayegh-etal-2025-feeding}'s synthesis method.

Finally, we report and analyze the meta-evaluation scores of popular scorers, including different variants of xCOMET~\citep{10.1162/tacl_a_00683} and MetricX~\citep{juraska-etal-2024-metricx}, across a sweep of the exposed balance parameter $\beta$.

In summary, we: (1) formulate the adequacy--fluency balance of MT meta-evaluation as a tunable parameter; (2) provide an exact algorithm that controls this parameter by importance-weighting the systems in the meta-evaluation, minimizing the distortion from uniform weighting so as to maximally preserve the meta-evaluation's other characteristics; and
(3) propose a controlled scorer-augmentation method to assess a meta-evaluation's adequacy-over-fluency preference, MQM adherence, and explainability by MQM.

\section{Background and Related Work}

\subsection{Meta-Evaluation}

A translation scorer is evaluated through meta-evaluation, which measures its agreement with human annotations across a fixed pool of translation systems \citep{callison-burch-etal-2007-meta}. Expert error annotation under Multidimensional Quality Metrics~\citep[MQM;][]{mariana2015multidimensional} has been established as the standard ground truth for the WMT metrics shared task~\citep{freitag-etal-2021-experts, freitag-etal-2023-results, kocmi-federmann-2023-large, kocmi-etal-2025-findings-wmt25}. Our work focuses on MQM-based meta-evaluation following recent standards.

\paragraph{Meta-metrics.}
While scorer--human agreement was traditionally measured via Pearson~\cite{machacek-bojar-2013-results, machacek-bojar-2014-results}, \citet{mathur-etal-2020-tangled} demonstrate that such a correlation is highly sensitive to outlier systems in the evaluation pool. To address this, \citet{kocmi-etal-2021-ship} introduce pairwise accuracy~(PA), which measures the fraction of system pairs whose human ordering is correctly reproduced. To overcome PA's coarse resolution and tie-handling limitations at the segment level \citep{deutsch-etal-2023-ties}, \citet{thompson-etal-2024-improving} develop soft pairwise accuracy~(SPA), which replaces PA's hard binary comparisons with a continuous measure based on statistical significance. SPA serves as the official system-level meta-metric of WMT24~\citep{freitag-etal-2024-llms}. We follow the line of research and focus on SPA as our main metric. We additionally report results for Pearson correlation in Appendix~\ref{app:pearson}.

\paragraph{Meta-evaluation analysis.}
Recent work critically examines meta-evaluation protocols, studying annotation stability \citep{knowles-2021-stability}, budget allocation \citep{riley-etal-2024-finding}, test-set item selection \citep{zouhar-etal-2025-select}, and baseline sensitivity via degenerate sentinel metrics \citep{perrella-etal-2024-guardians}. These studies treat pool-induced variation as noise or bias to be detected. Our work treats one special axis of that variation, the adequacy--fluency balance, as a tunable quantity that a practitioner should be able to explicitly configure.

\subsection{Adequacy and Fluency}

The distinction between translation adequacy and fluency originates in early evaluation campaigns~\citep{white-etal-1994-arpa} before being largely collapsed into error-weighted penalties \citep{lommel-etal-2014-multidimensional}. Information-theoretic analyses prove that adequacy and fluency genuinely trade off, making a single scalar summary insufficient \citep{flamich2025feedbirdsscoreaccuracynaturalness}. Moreover, fluency impacts downstream user trust independently of adequacy \citep{martindale-carpuat-2018-fluency}.
This trade-off extends to the meta-evaluation level: meta-evaluation inherits the pool's lean, favoring adequacy-oriented scorers in recent WMT MQM datasets due to skewed system selection~\citep{shayegh-etal-2025-feeding}.

We propose a principled parameterization for the adequacy–fluency balance in meta-evaluation, accompanied by a reweighting mechanism to achieve any target balance precisely.

\section{Approach}\label{sec:approach}

We control the adequacy--fluency balance by reweighting systems while explicitly minimizing the collision probability of the assigned weights. Practically, this keeps the meta-evaluation close to the natural dataset distribution and prevents over-indexing a handful of systems.

To control the adequacy--fluency balance, we need a reference for the two terms. MQM annotations provide fine-grained error categories. We follow \citet{flamich2025feedbirdsscoreaccuracynaturalness}  and group these into adequacy and fluency subsets, and compute Adequacy~MQM and Fluency~MQM at the system level accordingly. Their sum, Total~MQM, is the meta-evaluation ground truth.\footnote{Following \citet{shayegh-etal-2025-feeding}, we exclude rare errors that fall outside the two categories. The effect is negligible.}

 \subsection{The Balance Parameter}

A meta-evaluation dataset fixes $K$ translation systems. Write $a{=}(a_1,\cdots,a_K)$ and $f{=}(f_1,\cdots,f_K)$ for their system-level Adequacy~MQM and Fluency~MQM, so that system $k$ is the point $(a_k,f_k)$ and Total~MQM is $t{=}a{+}f$. We assign each system $k$ a weight $w_k$, with $w{=}(w_1,\dots,w_K)$ normalized so that $\sum_k w_k = \mathbf 1^\top w = 1$ and $w_k \ge 0$. Reweighting acts on these $K$ weights, so every system in the meta-evaluation remains real; the distribution of systems, which is likely biased, changes.

Let $\Sigma(w) = \mathrm{diag}(w) - ww^\top$. We have weighted mean and variance of scores $a$ as
\begin{align}
  \mu(a;w) &= w^\top a \\
  \sigma(a;w)^2 &= a^\top \Sigma(w)\, a \notag\\&= \textstyle\sum_k w_k a_k^2 - \mu(a;w)^2
\end{align}
respectively, and likewise for $f$.
We also have $a$--$f$ weighted covariance:
\begin{align}
  \Sigma(a,f;w) = a^\top \Sigma(w)\, f
\end{align}
Setting $w$ to the uniform weighting $\tfrac1K\mathbf 1$ recovers the unweighted sample mean and variance:
\begin{align}
  \mu(a) = \mu(a;\tfrac1K\mathbf 1) ;\;\;
  \sigma(a)^2 = \sigma(a;\tfrac1K\mathbf 1)^2
\end{align}

We assume throughout that no two systems share the same $(a_k,f_k)$; such pairs can be seen as one system in the weighting pool, with the assigned weight uniformly distributed among them.

We define the balance parameter\footnote{An equally valid choice for the balance parameter is $\beta_\text{std} = \sigma(a;w)/(\sigma(a;w)+\sigma(f;w))$, which is a monotonic reparametrization $\beta_\text{std} = (1 + \sqrt{\beta^{-1} - 1})^{-1}$.}
\begin{align}
  \beta(w) = \frac{\sigma(a;w)^2}{\sigma(a;w)^2 + \sigma(f;w)^2} \;\in\; [0,1]\tag{*}\label{eq:beta}
\end{align}
We specifically name the natural balance of the dataset $\beta_0 = \beta(\tfrac1K\mathbf 1)$.

Typically, $a$ and $f$ are partially correlated, thus Total~MQM's variance decomposes as $\sigma(t;w)^2 = \sigma(a;w)^2 + \sigma(f;w)^2 + 2\,\boldsymbol{\Sigma(a,f;w)}$. Our $\beta$ uses only the two marginal terms and leaves the cross term $\Sigma(a,f;w)$ out: $\beta$ answers how much adequacy's spread contributes relative to fluency's. In the next part, we show how meta-evaluation responds to this marginal quantity.

\subsection{Meta-Evaluation's Response to \texorpdfstring{$\boldsymbol\beta$}{Beta}}
\label{sec:pairwise-gaps}

A meta-metric typically judges scorers by aggregating agreement with the human ranking across system pairs. Therefore, meta-evaluation's lean towards adequacy or fluency is mainly driven by how much each component's gaps vary across the system pairs. We show next that this pairwise-gap variance is equal to the row-level variance, licensing its control through $\beta$.

\begin{proposition}
\label{lem:pairwise-gap}
For any vector $x{=}(x_1,\dots,x_K)$ and weights $w{=}(w_1,\dots,w_K)$ with $\mathbf 1^\top w = 1$
\begin{align}
  \sum_{k<k'} w_k w_{k'} (x_k - x_{k'})^2 \;=\; \sigma(x;w)^2
\end{align}
\end{proposition}
\begin{proof}
\begin{align}
  \sum_{k,k'} &w_kw_{k'} (x_k-x_{k'})^2 \notag\\&= \sum_{k,{k'}} w_kw_{k'}\big(x_k^2 - 2x_kx_{k'} + x_{k'}^2\big) \nonumber\\
  &= 2\sum_k w_kx_k^2 - 2\Big(\sum_k w_kx_k\Big)^2 \nonumber\\
  &= 2\,\sigma(x;w)^2 \label{eq:8}
\end{align}
The summand $w_kw_{k'}(x_k{-}x_{k'})^2$ is symmetric under $k{\leftrightarrow}{k'}$ and vanishes when $k{=}{k'}$, so splitting the sum by $k{<}{k'}$ gives $2\sum_{k<{k'}} w_kw_{k'}(x_k-x_{k'})^2$. Dividing \eqref{eq:8} by 2 gives the claim.
\end{proof}

Proposition~\ref{lem:pairwise-gap} lets us read the row-level variance $\sigma(x;w)^2$ as a pairwise quantity. Now let $q$ be the mean-squared pairwise gap
\begin{align}
  q(x;w) = \frac{\sum_{k<{k'}} w_kw_{k'} (x_k-x_{k'})^2}{\sum_{k<{k'}} w_kw_{k'}}\label{eq:9}
\end{align}
By Proposition~\ref{lem:pairwise-gap}, the numerator in \eqref{eq:9} is $\sigma(x;w)^2$ and the denominator is $x$-independent. Therefore
\begin{align}
  \frac{q(a;w)}{q(a;w){+}q(f;w)} = \frac{\sigma(a;w)^2}{\sigma(a;w)^2{+}\sigma(f;w)^2} = \beta(w)
\end{align}
In other words, $\beta$ is the balance of mean-squared pairwise gaps, and controlling $\beta$ effectively drives the meta-evaluation's lean.

\subsection{Weighted Soft Pairwise Accuracy}
\label{sec:weighted-spa}

Because SPA is our primary meta-metric, we explicitly define its weighted form. For every system pair~$(k,k')$, SPA compares the confidence that $k$ outranks $k'$ under human judgments against the confidence under a given scorer. This confidence is estimated via paired-permutation $p$-values, $p^{H}_{kk'}$ and $p^{m}_{kk'}$, the latter for scorer $m$; neither depends on $w$. Classical SPA averages the absolute gap between these $p$-values uniformly over all $\binom{K}{2}$ pairs. Weighting systems by $w$ assigns a weight of $w_kw_{k'}$ to each pair. Normalizing these weights yields
\begin{align}
  \pi_{kk'} = \frac{w_kw_{k'}}{\sum_{l<l'}w_lw_{l'}}\quad;
  \quad \sum_{k<k'}\pi_{kk'}=1
\end{align}
which defines the weighted meta-metric
\begin{align}
  \mathrm{SPA}(m;w) = 1 - \sum_{k<k'}
     \pi_{kk'}\,\bigl|\,p^{H}_{kk'} - p^{m}_{kk'}\,\bigr|
  \label{eq:wspa}
\end{align}
Uniform weighting yields $\pi_{kk'} = \binom{K}{2}^{-1}$, recovering classical SPA.

\subsection{Reweighting as an Optimization Problem}

Among all weight vectors $w$ that can realize a given target~$\hat\beta$, we find the one closest to uniform: it disturbs the natural view least beyond what~$\hat\beta$ requires, preserving most about the meta-evaluation.

We measure distance from uniform by collision probability $\mathrm{P}_\mathrm{collis}(w) = \sum_k w_k^2= w^\top w$. Minimizing this distance directly maximizes {effective sample size} $\mathrm{ESS}(w) = \mathrm{P}_\mathrm{collis}(w)^{-1}$, a standard diagnostic for weight concentration. In our work, we report $\mathrm{ESS}(w)/K$ to quantify how broadly the weights are distributed, with $K$ being the size of the original system set.

Writing the balance target as a constraint
\begin{align}
  b(w,\hat\beta) = (1-\hat\beta)\sigma(a;w)^2 - \hat\beta\sigma(f;w)^2\label{eq:b}
\end{align}
so that $b(w,\hat\beta) = 0 \iff \beta(w)=\hat\beta$, we formalize the reweighting as an optimization problem:
\begin{align}
  w^\star = \operatorname*{argmin}_w&\ \tfrac12\ \mathrm{P}_\mathrm{collis}(w) \tag{$\mathcal P$}\label{problem}\\
  \text{s.t.}&\ \mathbf 1^\top w = 1,\ \ b(w,\hat\beta) = 0,\ \ w_k \ge 0 \nonumber
\end{align}

\paragraph{Reachable range.}
For two systems $k,k'$, write
\begin{align}
  \beta_{kk'} = \frac{(a_k-a_{k'})^2}{(a_k-a_{k'})^2+(f_k-f_{k'})^2}\label{eq:bkk}
\end{align}
and let $\beta_{\min}$ and $\beta_{\max}$ be the smallest and largest such pairwise balance over all pairs of systems in the dataset.

\begin{restatable}[Reachable range]{theorem}{theoremrange}
\label{thm:range}
Problem~\eqref{problem} is feasible for target $\hat\beta$ iff $\hat\beta \in [\beta_{\min},\beta_{\max}]$. \qualification{ (Proof in Appendix~\ref{app:proof-range}.)}
\end{restatable}

\subsection{Our Algorithm}\label{sec:alg1}

In this section, we develop and theoretically justify our solver algorithm for Problem~\eqref{problem}. We provide a high-level narrative for the algorithm, followed by auxiliary theorems to improve its efficiency.

Define $S(w) = \{k : w_k > 0\}$ as the support system set of $w$.
Call a support $S$ \emph{admissible} for target $\hat\beta$ if there exists a feasible weighting $w_S$ with support $S$ for target $\hat\beta$, i.e., it is non-negative and $b(w_S, \hat\beta) = 0$.

\begin{restatable}[Bounded candidate count]{theorem}{theoremcandidates}
For an admissible support $S$, let
\label{thm:candidates}
\begin{align}
  \min_w&\ \tfrac12\ \mathrm{P}_\mathrm{collis}(w) \tag{$\mathcal{P}_S$}\label{eq:syspoolw}\\
  \text{s.t.}&\ \mathbf 1^\top w=1,\ \ b(w,\hat\beta)=0\ \ \ \text{\small(no non-negativity)} \nonumber
\end{align}
 be solved on $S$'s coordinates, with $w_k{=}0$ for $k{\notin}S$. Then (a) Problem~\eqref{eq:syspoolw} has at most five stationary points $w_S$, each obtainable in closed form by solving a polynomial equation of degree $\leq5$; and (b)~if $w^\star$ solves~\eqref{problem} and $S=S(w^\star)$, then $w^\star$ is one of these $w_S$. \qualification{ (Proof in Appendix~\ref{app:proof-candidates}.)}
\end{restatable}

This gives the abstract shape of our algorithm: for every support $S$, compute its (at most five) stationary points $w_S$, discard any with a negative entry, and among everything that survives across every support, return the one of smallest $\mathrm{P}_\mathrm{collis}(w_S)$. Part~\textit{(b)} guarantees $w^\star$ is never missed by this process, thus our algorithm is complete.

Next, we show the enumeration can be largely pruned without losing this guarantee.

For a support $S$ with $|S|{\ge}2$, let $\beta_{\min}(S)$ and $\beta_{\max}(S)$ be the smallest and largest pairwise balance $\beta_{kk'}$ among pairs $k,k'\in S$.

\begin{restatable}[Admissible supports]{theorem}{theoremadmissible}
\label{thm:admissible}
If $w_S$ with support $S$ is a feasible weighting for target $\hat\beta$, then $\beta_{\min}(S) \le \hat\beta \le \beta_{\max}(S)$. \qualification{ (Proof in Appendix~\ref{app:proof-admissible}.)}
\end{restatable}

Equivalently, a support whose pairwise balances do not bracket $\hat\beta$ can be safely skipped during the search.

\begin{restatable}[Support size bound]{theorem}{theoremsizebound}
\label{thm:size-pruning}
For any weighting $w$ supported on $S$, $\mathrm{P}_\mathrm{collis}(w) \ge 1/|S|$, with equality iff $w$ is uniform on $S$. \qualification{ (Proof in Appendix~\ref{app:proof-size-bound}.)}
\end{restatable}

Theorem~\ref{thm:size-pruning} suggests that once a feasible weighting with collision probability $p$ has been found, no support of size at most $1/p$ can carry a feasible weighting of strictly smaller collision probability, and such supports may be skipped by the search. Particularly, our  search process visits supports from largest to smallest, and the search stops the first time Theorem~\ref{thm:size-pruning}'s pruning applies.

Theorems~\ref{thm:admissible} and~\ref{thm:size-pruning} offer complementary pruning mechanisms: each is strongest where the other is weakest. When $\hat\beta$ is close to the dataset's natural balance $\beta_0$, the optimum stays close to uniform and Theorem~\ref{thm:size-pruning} prunes heavily. When $\hat\beta$ is close to $\beta_{\min}$ or $\beta_{\max}$, it falls out from many supports' admissible range, thus Theorem~\ref{thm:admissible} prunes heavily.

\begin{restatable}[Exit certificate]{theorem}{theoremcertificate}
\setlength{\arraycolsep}{3pt}
\label{thm:certificate}
Let $w$ be feasible for Problem~\eqref{problem}. Let $\tilde{a}=a-\mu(a)\mathbf 1,\ \tilde{f}=f-\mu(f)\mathbf 1$ be both nonzero, and $\rho = \frac{\tilde{a}^\top\tilde{f}}{\|\tilde{a}\|\,\|\tilde{f}\|}$ the adequacy--fluency Pearson correlation across systems. If there exist $\eta,\lambda \in \mathbb R$ such that
\begin{align}
&k \in S(w): \quad w_k + \eta\,\dfrac{\partial b(w,\hat\beta)}{\partial w_k} + \lambda = 0 \tag{i}\\
&k \notin S(w): \quad \lambda + \eta\,\dfrac{\partial b(w,\hat\beta)}{\partial w_k} \ge 0 \tag{ii}\\
&{\scriptscriptstyle
I_2{+}2\eta
\begin{pmatrix}
\scriptscriptstyle\hat\beta\|\tilde f\|^2{-}(1{-}\hat\beta)\rho^2\|\tilde a\|^2 ,& \scriptscriptstyle-(1{-}\hat\beta) \rho\sqrt{1{-}\rho^2} \|\tilde a\|^2 \\[5pt]
\scriptscriptstyle-(1{-}\hat\beta)\rho\sqrt{1{-}\rho^2}\|\tilde a\|^2         ,& \scriptscriptstyle-(1{-}\hat\beta)(1{-}\rho^2)\|\tilde a\|^2 \end{pmatrix}{\succeq}0 \tag{iii}}
\end{align}
where $\succeq0$ means positive semidefinite, then $w$ is a global optimum of Problem~\eqref{problem}. \qualification{ (Proof in Appendix~\ref{app:proof-certificate}.)}
\end{restatable}

Theorem~\ref{thm:certificate} grants our algorithm a short exit: the moment one stationary point on a visited support passes the three conditions, the search stops.

\begin{algorithm}[t]
\caption{Exact reweighting for target $\hat\beta$}
\label{alg:reweight}
\begin{algorithmic}[1]
\REQUIRE $\hat\beta \in [\beta_{\min},\beta_{\max}]$ \COMMENT{Theorem~\ref{thm:range}}
\STATE\textbf{register} $w^\star$
\FOR{$S \subseteq \{1,\dots,K\}$, $|S| \ge 2$, in decreasing order of $|S|$}
  \IF{$|S| \le 1/\mathrm{P}_\mathrm{collis}(w^\star)$}
    \STATE \textbf{break} \COMMENT{Theorem~\ref{thm:size-pruning}}
  \ENDIF
  \IF{$S$ not admissible at $\hat\beta$}
    \STATE \textbf{continue} \COMMENT{Theorem~\ref{thm:admissible}}
  \ENDIF
  \FOR{each stationary candidate $w$ on $S$ \COMMENT{Theorem~\ref{thm:candidates}a}}
    \IF{$w$ passes the exit certificate}
      \STATE \textbf{return} $w$ \COMMENT{Theorem~\ref{thm:certificate}}
    \ENDIF
    \IF{$w$ feasible and $\mathrm{P}_\mathrm{collis}(w) < \mathrm{P}_\mathrm{collis}(w^\star)$}
      \STATE $w^\star \gets w$ \COMMENT{Theorem~\ref{thm:candidates}b}
    \ENDIF
  \ENDFOR
\ENDFOR
\STATE \textbf{return} $w^\star$
\end{algorithmic}
\end{algorithm}

Algorithm~\ref{alg:reweight} concludes our exact search algorithm.
Theorems~\ref{thm:range}--\ref{thm:certificate} guarantee the soundness and completeness of our algorithm unconditionally: pruned branches are guaranteed to not include better solutions, short exit guarantees returning the global optimum, and the algorithm guarantees to search all branches (finite) carrying a possible solution. We also cross-validated its output against brute-force exhaustive search over all supports on small system sets, where full enumeration remains tractable, and found exact agreement throughout.\footnote{Theorem~\ref{thm:candidates} bounds the stationary points by the roots of a degree-five polynomial, but our implementation solves the equivalent rational secular equation directly rather than rooting that polynomial, whose coefficients become badly scaled as $\theta_2 \to 0$. This and all numerical thresholds are fixed in the released code.}
In Appendix~\ref{sec:comp}, we report how much computation Theorems~\ref{thm:admissible}--\ref{thm:certificate} save us in practice.

\section{Experiments}\label{sec:experiments}

In this section, we analyze the response of meta-evaluation to different
leanings that scorers may have. Specifically, we measure
each meta-evaluation's adequacy-over-fluency preference, MQM adherence, and explainability by MQM. To this end, we develop a controlled testbed that reveals signals for these aspects. Then, we benchmark our reweighting
mechanism on this testbed, comparing it against the unweighted
meta-evaluation and the system-synthesis-based meta-evaluation of
\citet{shayegh-etal-2025-feeding}.

\subsection{Scorer Augmentation}

Evaluation and analysis of scorers needs a ground truth for their true
underlying character (e.g., how much of a scorer's assessment is driven by
the fluency of the translation), which is not available. To obtain a replacement, we propose to augment scorers in a
controlled way, so that the augmented scorers' relative identities are known by construction, even though no scorer's absolute
identity is. We use this relative identity as pseudo ground truth for evaluation
and analysis. Importantly, we augment real scorers, staying close to their natural
biases and characteristics, except for those we control in particular.

We augment a base scorer $s_0$ in three directions, giving three
families of augmented scorers, each to test a different way a
meta-evaluation's preference among scorers could change. Given $a$ and $f$ being gold Adequacy~MQM and Fluency~MQM, the three directions are:
\begin{compactitem}
  \item MQM-adherence family: $s_0 \leftrightarrow a+f$
  \item Adequacy--fluency family: $s_0 \leftrightarrow a-f$
  \item Explainability-by-MQM family: $s_0 \leftrightarrow$ part of~$s_0$ only explained by $(a,f)$
\end{compactitem}
An ideal meta-evaluation should typically prefer an augmented scorer that better adheres to MQM or prefers explainability by MQM.

We realize each family with a dial $d$: a base scorer always has $d=0$,
and as $d$ increases, the augmented scorer approaches its family's
target. Each value of $d$ yields a new augmented scorer whose relative
identity to other members of its family is known through their relative
$d$ values. Note that we allow negative $d$, i.e., going backwards in each
direction.

\paragraph{MQM-adherence family.} This is constructed by moving linearly towards $t=a+f$ at the system level. For
segment~$i$ of system~$k$, and base scorer~$s_0$, the augmented scorer is
\begin{align}
  s_{d}(k, i) = s_0(k, i) + \frac{d \cdot \sigma(s_0)}{\sigma(z)} z_k
\end{align}
where $z$ is the system-level $z$-score of $t$, and $\sigma$ is standard deviation across systems. In simple words, we move towards $z$ in
steps of $d \times s_0$'s scale.\footnote{The same approach can produce Adequacy~MQM-adherence or Fluency~MQM-adherence families of synthetic scorers by considering $z$ as the system-level $z$-score of $a$ or $f$, respectively. We exclude these families from our analysis as the partial correlation between the two aspects~\citep{shayegh-etal-2025-feeding} confounds their interpretation.}

The above construction conserves characteristics of real scorers: the two
ends of the direction ($d=0$ and $d\rightarrow\infty$) are the base scorer
itself and the MQM ground truth, both considered valid
scorers. Moreover, for every finite $d$, the augmented scorer preserves
the base scorer's own system- and segment-specific ($k,i$-related) biases
and patterns. Note that the above construction does keep the
within-system variance $\sigma_i^2(s_d(k,\cdot))$ of each system $k$
intact. This is the signal SPA reads as a scorer's confidence. Keeping it
unchanged avoids confounding the MQM-adherence signal with a change in SPA's confidence signal.

\paragraph{Adequacy--fluency family.}
Similar to MQM-adherence family, we write
\begin{align}
    s_d(k, i){=}s_0(k,i){+}\frac{d\cdot\sigma(s_0)}{\sigma(z^{(a)})} z^{(a)}_k{-}\frac{d\cdot\sigma(s_0)}{\sigma(z^{(f)})} z^{(f)}_k
\end{align}
where $z^{(a)}$ and $z^{(f)}$ are the system-level $z$-scores of $a$ and $f$, respectively, and $\sigma$ is the standard deviation across systems. This construction is particularly clean: with correlation $\rho(a, f){<}1$, it establishes a known adequacy-over-fluency preference along the increasing dial $d$. If $\rho = 1$, there is no trade-off, and the construction has no effect.

\begin{figure}[t]
\begin{center}
\includegraphics[width=1\linewidth]{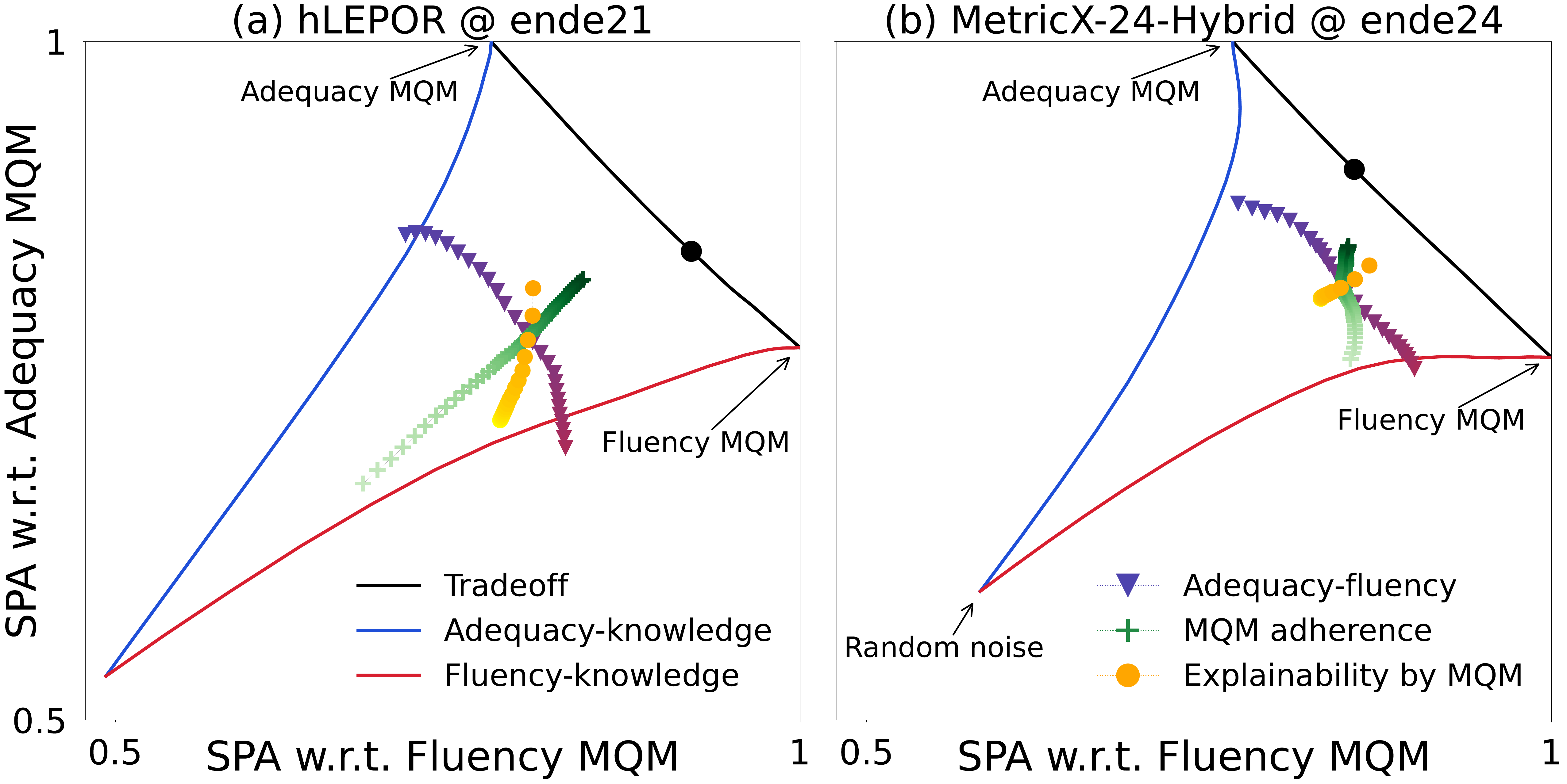}
\end{center}
\vspace{-10pt}
\caption{The SPA plane of \citet{shayegh-etal-2025-feeding}, with
the augmented scorers of all three families plotted as points, for
(a)~hLEPOR~\citep{han-etal-2013-language} in ende21 dataset and (b)~MetricX-24-Hybrid~\citep{juraska-etal-2024-metricx} in ende24 dataset as the base scorer.}
\label{fig:2}
\vspace{-5pt}
\end{figure}

\paragraph{Explainability-by-MQM family.}
This family's direction is given by the part of the base scorer not
explained by the joint $(a,f)$ pair: moving one way along it reduces this
unexplained part, approaching what $(a,f)$ alone explains, while moving
the other way amplifies it.
It is different from the other two as instead of injecting/removing
specific signals, it magnifies/reduces any $(a,f)$-unrelated signal the base scorer may have.
We find the direction as follows. Write $a_{ki}$ and $f_{ki}$ for segment
$i$ of system $k$'s gold Adequacy~MQM and Fluency~MQM scores, and let
$\bar{s}_0(a, f)$ be the $s_0$'s mean score among segments sharing the exact
$(a, f)$:
\begin{align}
  \bar{s}_0(a,f) = \mathbb{E}_{i,k:\;a_{ki}=a,\;f_{ki}=f} \;\; [s_0(k, i)]
\end{align}
Let
\begin{align}
  \bar e_k = \mathbb{E}_i \big[ s_0(k, i) - \bar{s}_0(a_{ki}, f_{ki}) \big]
\end{align}
be system $k$'s \emph{offset}: the part of its score that the joint pair
does not explain. We construct the augmented scorer
\begin{align}
  s_{d}(k, i) = s_0(k, i) - d \cdot \bar e_k.
\end{align}
In simple words, we partially remove each system's offset beyond what its
$(a, f)$ explains. For
negative $d$, we add more of it.

To visually illustrate the three families,
Figure~\ref{fig:2} plots them, applied to two different
base scorers, on the SPA plane of \citet{shayegh-etal-2025-feeding}.

\subsection{Evaluation Setup}
We analyze the performance of a meta-evaluation separately in each of
three directions. Specifically, we compare the meta-evaluation's assessment of
different augmented scorers within each family, and count for how many
adjacent-in-$d$ scorer pairs the
meta-evaluation correctly identifies the dial direction. We report this
success rate separately for each family as three measures.

\paragraph{Datasets.} We use MQM data published by~\citet{freitag-etal-2021-experts}.\footnote{\href{https://github.com/google/wmt-mqm-human-evaluation}{github.com/google/wmt-mqm-human-evaluation}; MQM data for 2025~\citep{lavie-etal-2025-findings} has not been published by the time of this study.} We restrict our analysis to datasets from 2021 onward~\citep{freitag-etal-2021-results,freitag-etal-2022-results,freitag-etal-2023-results,freitag-etal-2024-llms} due to limited standardization and support for older datasets. We specifically exclude ende23 because many of its MQM annotations are categorized as ``other'' rather than being properly assigned to adequacy or fluency categories. This leaves 11 datasets: ende24, enes24, jazh24, heen23, zhen23, ende22, zhen22, ende21, zhen21, ted\_ende, and ted\_zhen. Across all datasets, we evaluate all available systems except for M2M100\_1.2B-B4 and MSLC, which are officially flagged for exclusion. Detailed statistics are in Appendix~\ref{app:datasetstats}.

\begin{figure*}[t]
\begin{center}
\includegraphics[height=5.5cm]{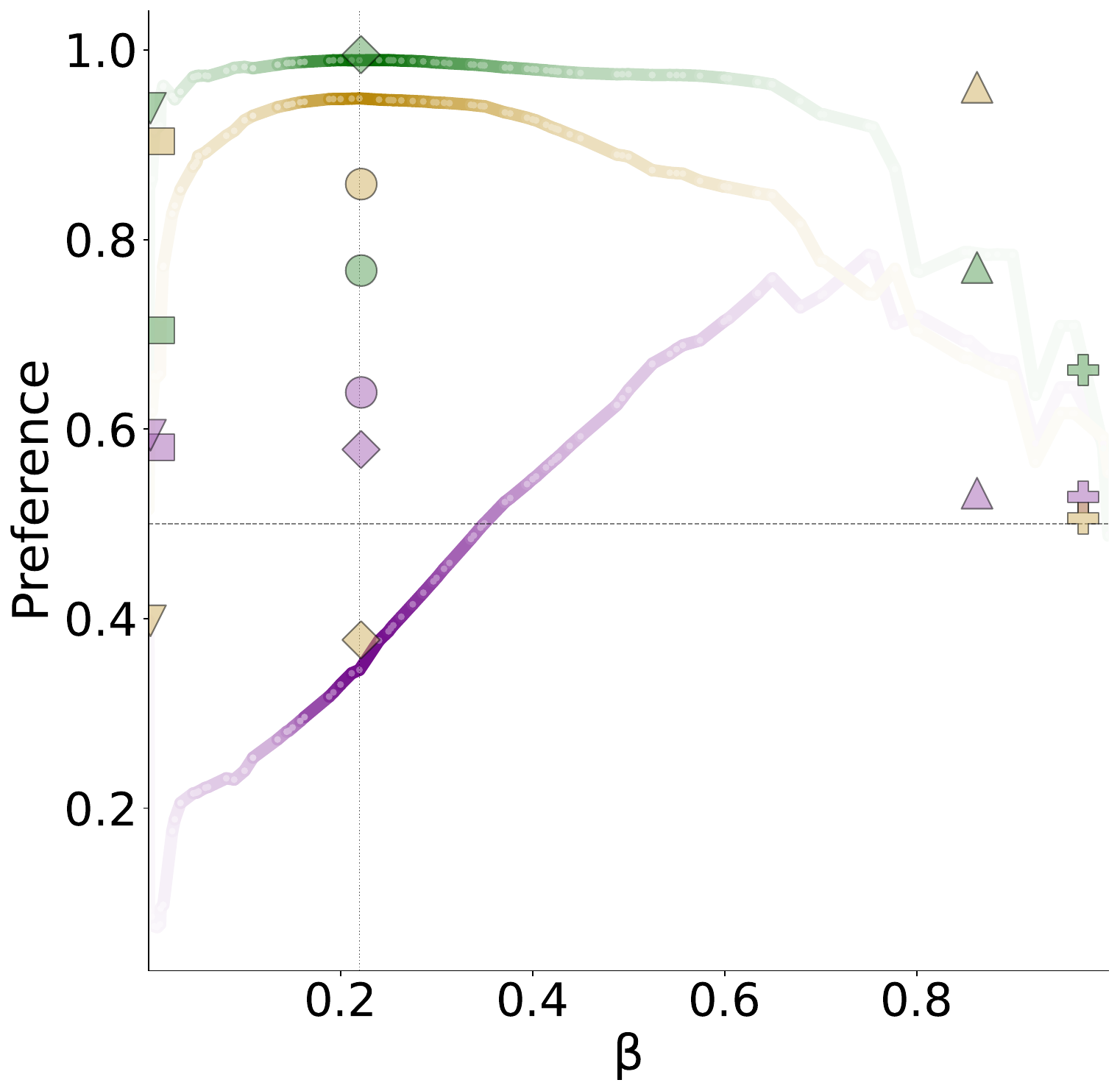}
\hfill
\includegraphics[height=5.5cm]{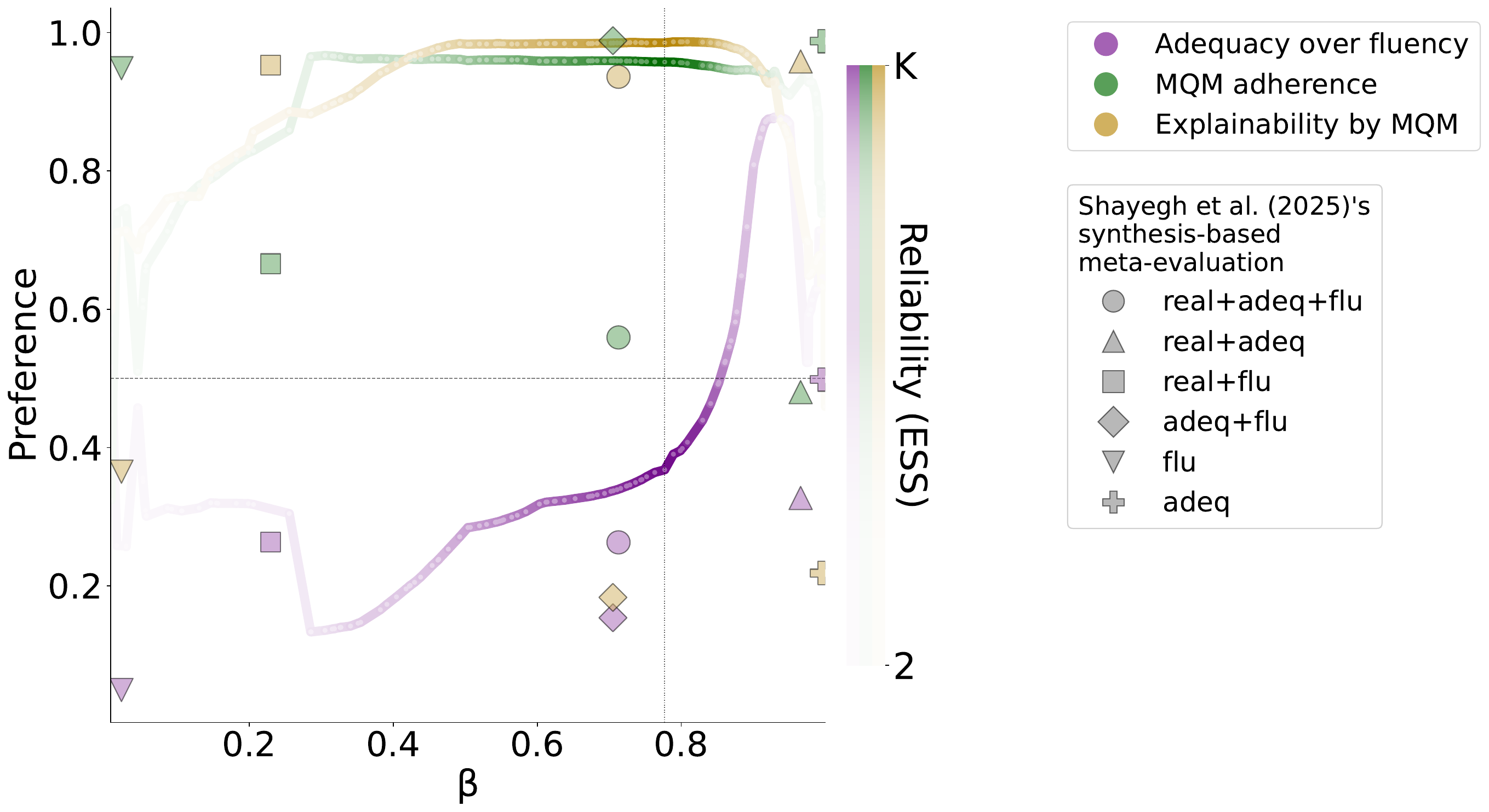}
\end{center}
\vspace{-12pt}
\caption{Main results on ende21 (left) and ende24 (right). Curve color intensity represents the effective sample size~($\mathrm{ESS}$); a lower $\mathrm{ESS}$ indicates fewer actively involved systems, reducing representativeness. The strictly increasing adequacy-over-fluency curve demonstrates that $\beta$ is an effective driver of this balance. The natural balance $\beta_0$ for each dataset is marked by a vertical dotted line, with its intersections indicating the original meta-evaluation's performance. Higher MQM adherence and explainability by MQM indicate better meta-evaluation internal consistency. Markers show the performance of the six synthesis-based meta-evaluation setups from~\citet{shayegh-etal-2025-feeding}, placed at their synthesized datasets' respective $\beta_0$.}
\label{fig:3}
\vspace{-5pt}
\end{figure*}

\paragraph{Base scorers.}
We consider automatic scorers (omitting human references) in the MQM datasets as the base scorers. This includes the competitive scorers as well as the diagnostic ones included by WMT. We exclude those without segment-level scores (e.g., corpus-level chrF). We also deduplicate perfectly duplicated scorer pairs. We provide the complete list of the base scorers in Appendix~\ref{app:datasetstats}. For each meta-evaluation and each dataset, we report the
average scores across all base scorers.

\paragraph{Dial grid.} The closer two $d$ values are, the harder it is for a meta-evaluation to recognize their relative identity. We empirically choose the $d$ grid for the best readability of results. For MQM adherence, $d$ takes $-1.5+3^{n/200}$, and for explainability by~MQM, $-3.5+4^{n/200}$, where $n\in\mathbb{N}: 0~\leq~n~\leq~200$, resulting in $200$ adjacent pairs per base scorer. For adequacy-over-fluency preference, $d$ ranges from $-0.5$ to $0.5$ with a step size of $0.02$, resulting in $50$ adjacent pairs per base scorer.

\subsection{Main Results}
\label{sec:results}

Figures\footnote{Results for some datasets are deferred to the Appendix~\ref{app:all_datasets} due to space constraints. All the claims are written with respect to all datasets, including those deferred to the appendix.}~\ref{fig:3} and~\ref{fig:results_all} show our method's performance across datasets over a $\beta$ sweep spanning $[\beta_{\min}, \beta_{\max}]$, plus explicitly every $\beta_{kk'}$ in the dataset. The monotonically increasing adequacy-over-fluency curves confirm that our reweighting mechanism effectively controls this balance. Furthermore, within a range around $\beta_0$ where $\mathrm{ESS}$ remains large, our method successfully preserves the original meta-evaluation's MQM adherence and explainability by MQM.

The strong correlation between $\mathrm{ESS}$ and both MQM adherence and explainability by MQM confirms $\mathrm{ESS}$ as a useful proxy for reliability, preventing severe evaluation distortion. Practically, $\mathrm{ESS}$ identifies the reachable range for controlling the adequacy--fluency balance without degrading meta-evaluation quality.

The permissible deviation from $\beta_0$ is fundamentally constrained by the dataset itself. In datasets like jazh24, most $\beta_{kk'}$ values concentrate within a narrow band, limiting reweighting capacity. Consequently, if a dataset lacks systems that intrinsically support a fluency-heavy evaluation, this imposes a strict limit on the achievable balance. 

Compared to \citet{shayegh-etal-2025-feeding}, our method consistently yields superior MQM adherence and/or explainability by MQM across all datasets, whether controlled by $\beta$ or by adequacy-over-fluency preference. The only exceptions occur at extreme adequacy-over-fluency values (e.g., in zhen23, ende22, and ende24) which their synthetic method reaches but our reweighting cannot. However, explainability by MQM heavily drops for both methods in these regimes, rendering those extremes highly distorted and practically infeasible anyway. Overall, our method strictly improves upon \citet{shayegh-etal-2025-feeding} in preserving the internal consistency of the meta-evaluation.

\begin{figure*}[t]
\begin{center}
\includegraphics[height=7.1cm]{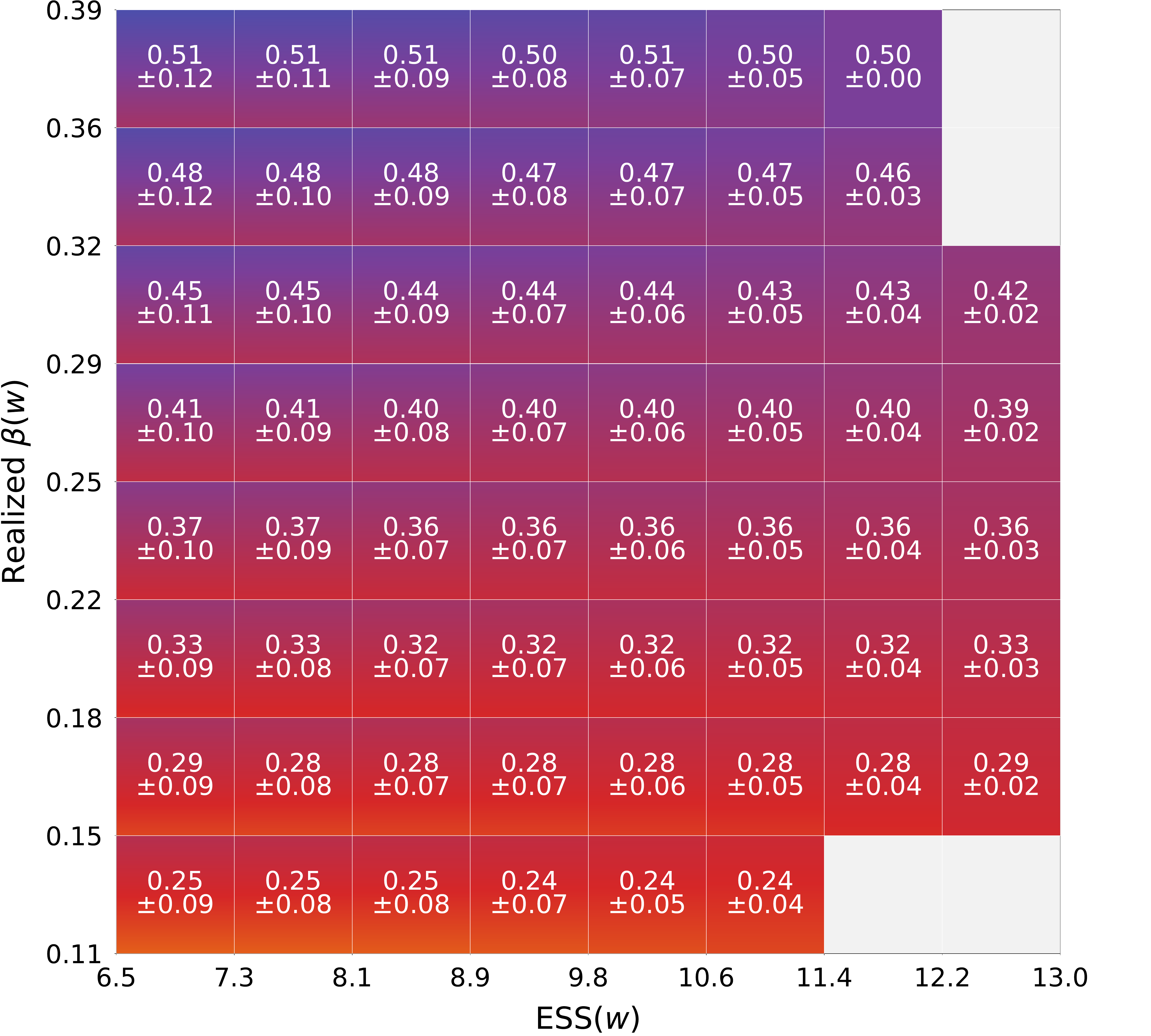}
\hfill
\includegraphics[height=7.1cm]{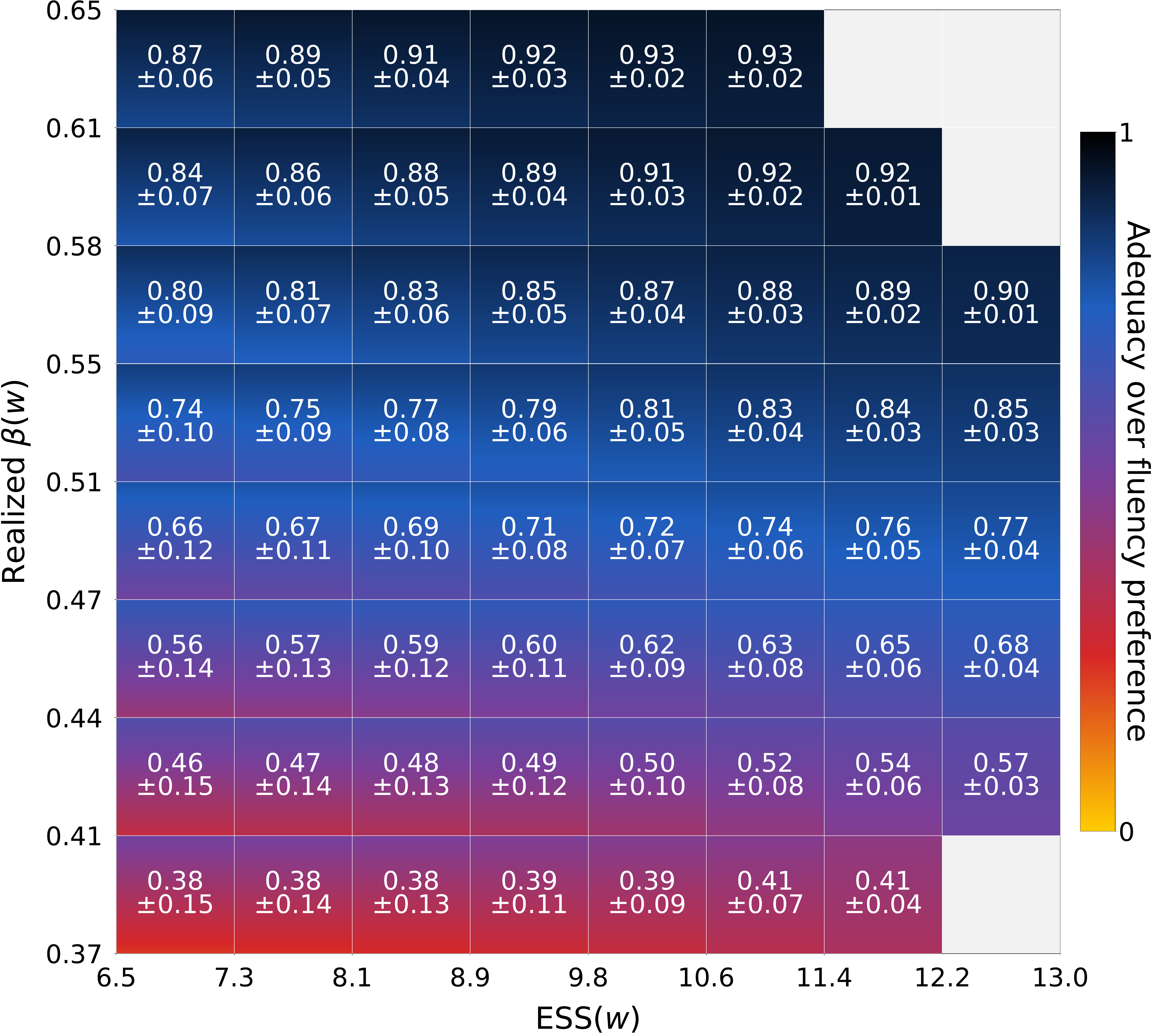}
\end{center}
\vspace{-10pt}
\caption{Adequacy-over-fluency preference across the $(\mathrm{ESS}, \beta)$ plane for ende21 (left) and ted\_zhen~(right). Each cell gives the mean and standard deviation over the weightings realizing that $(\mathrm{ESS}, \beta)$, color-ramped from $\mathrm{mean}{-}\mathrm{std}$ at its bottom edge to $\mathrm{mean}{+}\mathrm{std}$ at its top. Grey cells are unreachable.}
\label{fig:heatmap}
\vspace{-5pt}
\end{figure*}

Analyzing the markers in Figure~\ref{fig:3} reveals a critical limitation of synthetic systems. In \citet{shayegh-etal-2025-feeding}'s setups that retain real systems, explainability by MQM is preserved but MQM adherence degrades. Conversely, in their fully synthesized pools, MQM adherence is preserved while explainability by MQM collapses. Explainability by MQM measures the meta-evaluation's response to scorer noise while preserving natural biases; it is informed, but not strictly dictated, by MQM. This failure point of fully synthetic pools confirms our hypothesis that synthetic systems are unrepresentative. This stems from their construction: segments are regrouped to maximize between-system variance. By the law of total variance~\cite{soch2025statproofbook}, since total pool variance is fixed, altering the between-system component inevitably alters the within-system component. This yields systems with unrealistically low score variances, being unrepresentative and also distorting SPA's confidence estimations.

In Appendix~\ref{app:pearson}, we report similar experiments using Pearson correlation as the meta-metric. The observed trends remain consistent, further supporting the effectiveness of our reweighting method.

\subsection{Additional Analysis}

\paragraph{Variance Analysis.}
We investigate whether weight vectors with similar $\beta$ values can yield substantially different adequacy--fluency balances due to uncontrolled factors, which would invalidate $\beta$ as a reliable control parameter. The results in~\S\ref{sec:results} are limited to $\mathrm{P}_\mathrm{collis}$-minimized weight vectors for each $\beta$, hiding any same-$\beta$ variance. To address this, we sample numerous weight vectors on the simplex $w \ge 0:\mathbf 1^\top w = 1$, each realizing some $\beta(w)$, $\mathrm{ESS}(w)$, and adequacy-over-fluency preference. By grouping these into buckets of~$(\mathrm{ESS}, \beta)$ pairs, we compare within-bucket preference variances (indicating $\beta$-independent shifts) against across-bucket average trends (indicating the $\beta$-effect), demonstrating the reliability of $\beta$ as a control parameter.

We restrict the experiment to $w: \mathrm{ESS}(w)\in[\frac{K}{2}, K]$, given the unreliability of weighting with lower $\mathrm{ESS}$. We draw around $3.5$M weight vectors per dataset. These sets provide varying $\beta$ coverage across datasets. We partition $\beta$ and $\mathrm{ESS}$ ranges into $8$ equal-sized buckets each. We compute the adequacy-over-fluency preference for each weighting, presenting the results as a heatmap over this $8\times8$ grid. Because these weightings are unoptimized, they remain unconstrained beyond the $\beta$ and $\mathrm{ESS}$. Consequently, a cell's standard deviation strictly isolates the variation in SPA-based adequacy-over-fluency preference attributable to everything about $w$ {except}~$\beta$.

\begin{figure*}[t]
\begin{center}
\includegraphics[width=1\linewidth]{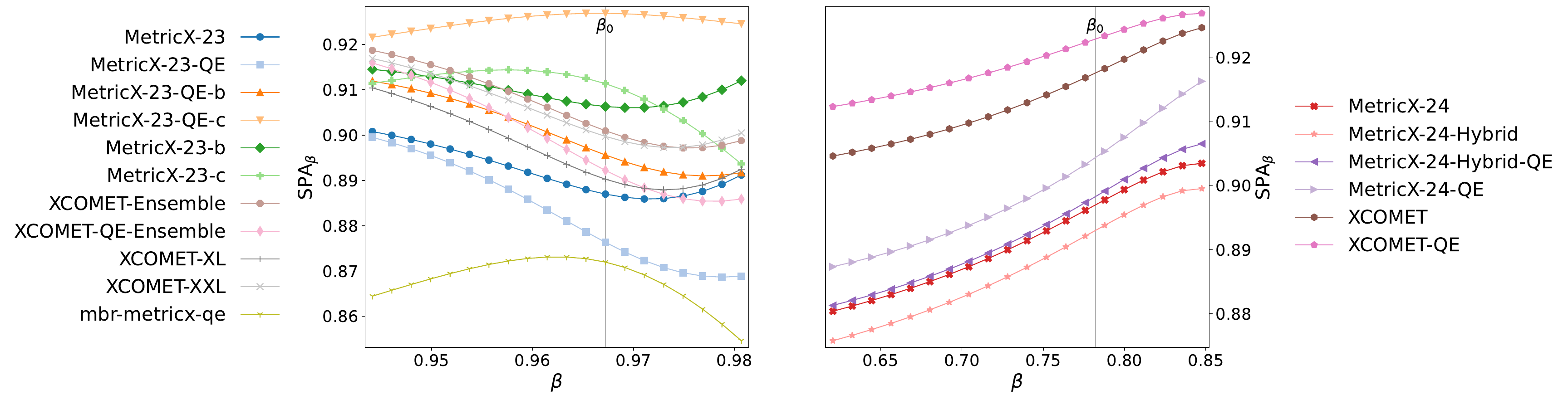}
\end{center}
\vspace{-12pt}
\caption{Weighted meta-evaluation performance (SPA$_\beta$) of MetricX and xCOMET variants as a function of the adequacy--fluency balance parameter $\beta$. In zhen23 (left), scorer rankings change depending on the target balance. In ende24 (right), relative performance remains stable across the reachable range. $\beta$ interval is bounded by $\mathrm{ESS} \ge K-1$, where $K$ is the total number of systems ($15$ for zhen23, $16$ for ende24).}
\label{fig:metrics}
\vspace{-5pt}
\end{figure*}

Figures\footnote{Results for some datasets are deferred to Appendix~\ref{app:all_heatmaps} due to space limits. Claims are made with respect to all datasets.}~\ref{fig:heatmap} and~\ref{fig:all_heatmaps} report the results. Each cell represents a pair of $\beta$ and $\mathrm{ESS}$ buckets. Reading down any column, preference increases with~$\beta$ at every $\mathrm{ESS}$. This trend holds for $99.1\%$ of vertically-adjacent cell pairs pooled across all datasets. Reading across rows, higher $\mathrm{ESS}$  sharply reduces variance: averaged over all $554$ horizontally-adjacent cell pairs, higher $\mathrm{ESS}$ exhibits $18.3\%$ lower same-$\beta$ standard deviation, and the standard deviation decreases in $99.8\%$ of those pairs. Lastly, comparing means to standard deviations confirms that same-$\beta$ variance does not obscure the $\beta$ signal. Despite partial overlap between cells, the expected differences remain substantial. A one-sided Welch's $t$-test~\citep{a967ba42-d0b9-3e65-976f-80cc6086b406} on every vertically-adjacent cell pair confirms this: $94.8\%$ of pairs show that the higher-$\beta$ cell has a significantly higher preference~($p{<}0.01$), and no pair is significant in the opposite direction.

The mapping from $\beta$ to the adequacy-over-fluency preference is dataset-dependent: not only does the preference measure preclude cross-dataset comparison by construction, identical $\beta$ values also yield empirically different preference scores across datasets. Despite this limitation, the underlying relationship between $\beta$ and the adequacy–fluency balance remains universal. Its direction, monotonicity, and robustness to same-$\beta$ variation hold across every dataset tested, licensing $\beta$ as a reliable control parameter.

\paragraph{Weighted Meta-Evaluation.}
Having validated our reweighting mechanism, we now apply it to compare variants of the popular scorers MetricX~\citep{juraska-etal-2024-metricx} and xCOMET~\citep{10.1162/tacl_a_00683} across a sweep of $\beta$. Figure~\ref{fig:metrics} reports the weighted soft pairwise accuracy (SPA) for the zhen23 and ende24 datasets.\footnote{We selected the two datasets empirically to best show two different possible scenarios: when the scorer ranking is or is not influenced by adequacy–fluency balance.}

The rankings in ende24 remain largely stable across the reachable $\beta$ range. In contrast, the performance trajectories in zhen23 intersect a lot, demonstrating that the relative ranking of scorers \textbf{can} be highly sensitive to the adequacy--fluency balance. Our proposed parameter and reweighting tool explicitly describe and control this behavior.

\paragraph{More results in the appendix.} In Appendix~\ref{sec:comp}, we empirically evaluate the computational efficiency of our search algorithm, ablate its pruning mechanisms, and compare it against a numerical optimization baseline. In Appendix~\ref{sec:loo}, we show that meta-evaluation is sensitive to random system removal, and that adequacy--fluency imbalance is not the main cause of this instability.

\section{Conclusion}

Machine translation meta-evaluation inherits the adequacy–fluency tradeoff, with its balance dictated by the specific combination of translation systems in the meta-evaluation data. Because this system set is a small, unrepresentative sample that fluctuates across datasets, the resulting adequacy–fluency balance is uncontrolled and subject to change.

We exposed this balance as a tunable control parameter, $\beta$. To achieve a target $\beta$, we reweighted existing systems while minimizing distortion from uniform weighting, ensuring the evaluation pool remains real and representative. We supported this with an exact optimization algorithm, equipped with theoretical guarantees and pruning mechanisms, to efficiently compute these weights.

To provide an internal diagnostic for meta-evaluation behavior, we introduced a scorer-augmentation framework that establishes a known relative identity for near-real scorers. Our experiments demonstrated that our reweighting method successfully controls the adequacy–fluency balance and preserves evaluation integrity, outperforming system-synthesis approaches. Finally, our analysis of popular scorers showed that scorer rankings can be sensitive to~$\beta$.

\paragraph{Calibration of $\boldsymbol\beta$ in practice.}
A specific numerical value of $\beta$ does not yield an identical adequacy-over-fluency preference across different datasets. This stems from inherent variations in language pairs, the correlation between Adequacy~MQM and Fluency~MQM, base scorer biases, and other unaccounted variables. Thus, $\beta$ functions as a directional control knob rather than an absolute metric. To determine the target $\beta$ for a specific real-world scenario (e.g., technical versus creative translation), we recommend practitioners craft a few translation pairs that exhibit a marginal adequacy–fluency tradeoff, where one translation is slightly more adequate and the other slightly more fluent. After establishing which direction their specific use case prefers, they can sweep $\beta$ until the meta-evaluation prioritizes scorers aligned with that preference.

\section*{Limitations}

Our reweighting mechanism is strictly bounded by the intrinsic diversity of the underlying dataset. If a dataset lacks systems that exhibit isolated adequacy or fluency errors (e.g., jazh24), the reachable range of $\beta$ is narrow. Forcing $\beta$ beyond this range requires sparse weights, which reduces the effective sample size ($\mathrm{ESS}$) and compromises the representativeness of the evaluation pool. Further expanding the admissible range for those datasets with limited flexibility remains a direction for future work.

Moreover, our framework depends on the provided MQM annotations and the chosen adequacy--fluency taxonomy, introducing two unaccounted sources of variance. First, WMT human labels contain noise and individual annotator bias. Our method does not model this uncertainty, meaning it propagates directly into system-level scores and~$\beta$. Second, categorizing errors as adequacy or fluency is subjective. A different taxonomy alters the score distributions, so the same numerical $\beta$ will behave differently under a different decomposition. Accounting for annotation uncertainty and taxonomy dependence requires further research.

Finally, controlling the adequacy--fluency balance does not resolve the broader instability of meta-evaluation. As demonstrated in Appendix~\ref{sec:loo}, scorer rankings remain sensitive to system selection even when $\beta$ is fixed. Future work must identify the root causes of system-selection instability. We would like to investigate other axes of variation across translation systems that disrupt scorer rankings.

\section*{Acknowledgements}
We thank Lili Mou for the insightful discussions that provided the initial motivation for this research, and Jan-Thorsten Peter for their valuable feedback on the final manuscript. We also extend our gratitude to the WMT reviewers, both for this manuscript and for our prior work last year, whose constructive feedback significantly shaped the direction of this study. This research was conducted independently and received no external funding.

\bibliography{custom}

\appendix

\section{Proofs}

\subsection{Proof of Theorem~\ref{thm:range}}
\label{app:proof-range}

\theoremrange*

\begin{proof}
The forward direction is given by Theorem~\ref{thm:admissible} applied to the full support. We now prove the reverse direction: Any $\hat\beta \in [\beta_{\min}, \beta_{\max}]$ is a feasible target for Problem~\ref{problem}.

$\beta(w)$ is continuous on the relative interior of $S$'s simplex face and that face is connected. Moreover, for two systems $k_1, k_2: \beta_{\min}(S){=}\beta_{k_1k_2}$, a weight vector $\hat w = \{k\neq k_1,k_2:  w_k{=}0\}$ drives $\beta(w){=}\beta_{\min}(S)$; symmetrically for $\beta_{\max}(S)$. So the image is a continuous interval with infimum $\beta_{\min}(S)$ and supremum $\beta_{\max}(S)$.
\end{proof}

\subsection{Proof of Theorem~\ref{thm:candidates}}
\label{app:proof-candidates}

\theoremcandidates*

\begin{proof}[]
\textbf{[Part (a)]} We prove this part by deriving the explicit polynomial to bound the stationary points.
By \eqref{eq:b} and Proposition~\ref{lem:pairwise-gap}:
\begin{align}
  b(w,\hat\beta) = &(1-\hat\beta)\sum_{k<k'} w_k w_{k'} (a_k - a_{k'})^2 \notag\\
  &- \hat\beta\sum_{k<k'} w_k w_{k'} (f_k - f_{k'})^2\label{eq:bpairwise}\\
  = &(1-\hat\beta)  (\sum_{k} w_k a_k^2 -w^\top aa^\top w) \notag\\
  &- \hat\beta (\sum_{k} w_k f_k^2 -w^\top ff^\top w)\\
  = &w^\top (\hat\beta ff^\top - (1-\hat\beta)aa^\top ) w \notag\\
  &+ \sum_k w_k\big( (1-\hat\beta) a_k^2 - \hat\beta f_k^2 \big)\\
  = &w^\top H w + w^\top g \label{eq:binhg}
\end{align}
where $H = \hat\beta ff^\top - (1-\hat\beta)aa^\top$ and $g_k = (1{-}\hat\beta) a_k^2 - \hat\beta f_k^2$, yielding a quadratic in $w$.

Eliminate the budget constraint $\mathbf 1^\top w=1$ using the affine transformation
\begin{align}
    w = u_0 + Nu\label{eq:affinetransform}
\end{align}
where $u_0$ is the uniform weighting on $S$ and $N$ has orthonormal columns spanning $\{v: \mathbf 1^\top v=0\}$.
Substituting~\eqref{eq:affinetransform} into~\eqref{eq:binhg} yields a single quadratic in the free variable $u$:
\begin{align}
    b(w, \hat\beta) = \tilde{b}(u, \hat\beta) = u^\top M u + u^\top r + c = 0 \label{eq:tildeb}
\end{align}
where $M = N^\top H N$, $r = N^\top (2Hu_0+g)$, and $c=b(u_0,\hat\beta)$.
We can restate the problem in the free variable~$u$:
\begin{align}
  \min_u&\ \tfrac12\ \|u\|^2 + \mathrm{const.} \;\; \text{s.t.}\;\;\tilde{b}(u, \hat\beta)=0\quad\label{eq:probinu}\tag{$\mathcal{P}_S^{(u)}$}
\end{align}

A point $u$ is a \emph{stationary point} of~\eqref{eq:probinu} if it is feasible and satisfies
\begin{align}
\nabla_u{\big(\tfrac12\|u\|^2 + \eta\cdot\tilde b(u, \hat\beta)\big)} = 0
\end{align}
for a multiplier $\eta \in \mathbb R$. Taking the derivative:
\begin{align}
u + \eta(2 M u + r) &= 0 \\
(I+2\eta M)\,u &= -\eta r \label{eq:stat-u}
\end{align}
We call the support $S$ \emph{degenerate} if all its pairwise balances $\beta_{kk'}$, $k,k' \in S$, coincide. We consider two scenarios:

\paragraph{[Case 1: $S$ is degenerate]} Let $\bar\beta$ be the common value of the pairwise balances in $S$. Because $S$ is admissible, Theorem~\ref{thm:admissible} forces $\hat\beta = \beta_{\min}(S)=\beta_{\max}(S)=\bar\beta$. For any pair $k,k' \in S$ we have $(a_k{-}a_{k'})^2 + (f_k{-}f_{k'})^2 > 0$, as no two systems share the same $(a_k,f_k)$; together with $\hat\beta = \beta_{kk'}$ and~\eqref{eq:bkk} this gives
\begin{align}
    (1-\hat\beta)(a_k{-}a_{k'})^2 - \hat\beta (f_k{-}f_{k'})^2 = 0
\end{align}
Substituting into~\eqref{eq:bpairwise}
\begin{align}
    &b(w,\hat\beta) = \notag\\ 
    &\sum_{k<k'} w_k w_{k'} \big( (1-\hat\beta)(a_k - a_{k'})^2 - \hat\beta(f_k - f_{k'})^2 \big) \notag\\
    &=\sum_{k<k'} w_k w_{k'} \cdot 0 = 0
\end{align}
for {every} $w$. So the quadratic $\tilde b(\cdot,\hat\beta)$ vanishes identically and $M = 0$, $r = 0$, $c = 0$ in~\eqref{eq:tildeb}. The stationarity condition~\eqref{eq:stat-u} reads $u = 0$, so $S$ carries exactly one stationary point: the uniform weighting $u_0$. It is feasible, minimizes the strictly convex objective $\tfrac12\|u\|^2$, and is the unique global optimum on $S$.

\paragraph{[Case 2: $S$ is non-degenerate]} The pairwise balances in $S$ are not all equal. Let $\tilde a$ and $\tilde f$ be the score vectors restricted to $S$ and centered by their uniform means over $S$:
\begin{align}
    \tilde a_k = a_k - \mu(a; u_0)\\
    \tilde f_k = f_k - \mu(f; u_0)
\end{align}
for $k \in S$. These are linearly independent: were $\tilde f = \kappa \tilde a$ for some $\kappa \in \mathbb R$, every pair would satisfy $(f_k{-}f_{k'}) = \kappa (a_k{-}a_{k'})$ and hence $\beta_{kk'} = (1+\kappa^2)^{-1}$, contradicting non-degeneracy; the same holds with the roles of $\tilde a$ and $\tilde f$ exchanged.

Within the zero-sum subspace defined by $N$, $M$ acts as the difference of two independent rank-1 outer products: $\hat\beta \tilde f \tilde f^\top$ and $(1-\hat\beta)\tilde a \tilde a^\top$. 
By definition of $\beta$, we have $0 < \hat\beta < 1$. Therefore, this difference provides one strictly positive and one strictly negative direction. Consequently, $M$ is indefinite with rank two, possessing exactly two nonzero eigenvalues:
\begin{align}
    \theta_1 > 0 > \theta_2 \label{eq:eigenvalues}
\end{align}

We use this structure to solve for the stationary candidates. Projecting~\eqref{eq:stat-u} onto the eigen basis $\{v_j\}_{j=1}^{|S|-1}$ of $M$, and defining the coordinate components $u_j = v_j^\top u$ and $r_j = v_j^\top r$, splits it into independent scalar equations of the form
\begin{align}
u_j (1+2\eta\theta_j) = -\eta r_j
\end{align}
The system is singular exactly at the two poles $\lambda_j = -1/(2\theta_j)$ for $j \in \{1,2\}$. We categorize the candidates by resolving across three subcases for~$\eta$:

\paragraph{[Case 2.1: $\eta = \lambda_j$ and $r_j \neq 0$]} The $j$-th equation reads $u_j \cdot 0 = -\lambda_j r_j \neq 0$. The system is inconsistent, yielding zero candidates.

\paragraph{[Case 2.2: $\eta = \lambda_j$ and $r_j = 0$]} The $j$-th equation is $0=0$, leaving the coordinate $u_j$ free.
For every other coordinate $m \neq j$, the coefficient $(1 + 2\lambda_j \theta_m) = 1 - \frac{\theta_m}{\theta_j}$ is strictly non-zero: $\theta_m{=}0$ for flat directions $m\ge3$ and $\frac{\theta_{3-j}}{\theta_j}<0$ by~\eqref{eq:eigenvalues}. Therefore, every $u_m:m \neq j$ is uniquely fixed to
\begin{align}
u_m = \frac{-\lambda_j r_m}{1 + 2\lambda_j \theta_m}
\end{align}
Defining $u^\circ$ as this baseline vector with $u_j^\circ = 0$, the full quadratic constraint separates into:
\begin{align}
    \tilde{b}(u, \hat\beta) = \theta_j u_j^2 + r_j u_j + \tilde{b}(u^\circ, \hat\beta) = 0
\end{align}
Because $r_j = 0$, the linear term vanishes, leaving $\theta_j u_j^2 = -\tilde{b}(u^\circ, \hat\beta)$. This yields at most two candidates:
\begin{align}
    u_j = \pm\sqrt{-\tilde{b}(u^\circ, \hat\beta)/\theta_j}\label{eq:case21}
\end{align}

\paragraph{[Case 2.3: $\eta \neq \lambda_j$]} In this case, $1+2\eta\theta_j \neq 0$, which gives
\begin{align}
    u_j(\eta) =
    \dfrac{-\eta {r}_j}{1 + 2\eta\theta_j}
\end{align}
which simplifies to $-\eta {r}_j$ for $j \ge 3$.
We substitute these components into \eqref{eq:tildeb}.

When $j \ge 3$, the quadratic terms vanish by $\theta_j = 0$, and the linear terms sum to:
\begin{align}
    \sum_{j \ge 3} {r}_j(-\eta {r}_j) = -\eta \sum_{j \ge 3} {r}_j^2\label{eq:flatpart}
\end{align}

For $j \in \{1,2\}$, substituting $u_j(\eta)$ into its quadratic and linear terms yields:
\begin{align}
    \theta_j &u_j(\eta)^2 + {r}_j u_j(\eta) \notag\\&= \theta_j \left(\frac{-\eta {r}_j}{1 + 2\eta\theta_j}\right)^2 + {r}_j \left(\frac{-\eta {r}_j}{1 + 2\eta\theta_j}\right) \\
    &= -\eta{r}_j^2\ \frac{ (1 + \eta\theta_j)}{(1 + 2\eta\theta_j)^2}\label{eq:quadpart}
\end{align}

Combining the constant $c$, \eqref{eq:flatpart}, and \eqref{eq:quadpart} for the two $j$ values, \eqref{eq:tildeb} becomes:
\begin{align}
    c - \eta \left[ \sum_{j=1,2} \frac{{r}_j^2 (1 + \theta_j \eta)}{(1 + 2\eta\theta_j)^2} + \sum_{j \ge 3} {r}_j^2 \right] = 0
\end{align}
Clearing denominators by multiplying through by the squared characteristic polynomial $(1+2\eta\theta_1)^2(1+2\eta\theta_2)^2$ yields a polynomial equation $\Phi_{S,\hat\beta}(\eta)=0$, where
\begin{align}
    \Phi_{S,\hat\beta}(\eta) = \;& c\,(1{+}2\eta\theta_1)^2(1{+}2\eta\theta_2)^2 \notag\\
    &- \eta\, {r}_1^2\,(1{+}\theta_1\eta)(1{+}2\eta\theta_2)^2 \notag\\
    &- \eta\, {r}_2^2\,(1{+}\theta_2\eta)(1{+}2\eta\theta_1)^2 \notag\\
    &- \eta \sum_{j \ge 3} {r}_j^2\,(1{+}2\eta\theta_1)^2(1{+}2\eta\theta_2)^2 \label{eq:phidef}
\end{align}
has degree at most five.

It remains to bound the number of candidates. We proceed in three steps. First, we evaluate $\Phi_{S,\hat\beta}$ at $\lambda_1$ and $\lambda_2$. We show that $\lambda_j$ is a root only when ${r}_j = 0$ (Case~2.2), and that it is then a root of multiplicity two. Second, we count the candidates when $\Phi_{S,\hat\beta}$ is not the zero polynomial. A root $\eta \notin \{\lambda_1,\lambda_2\}$ gives one candidate (Case~2.3) and costs one of the five roots, while a root at $\lambda_j$ gives two candidates and costs two of them by the first step. Hence the bound of five holds. Third, we treat the case where $\Phi_{S,\hat\beta}$ is the zero polynomial. In this case, counting roots does not apply and we solve the stationarity condition~\eqref{eq:stat-u} directly.

\paragraph{[Poles of $\Phi_{S,\hat\beta}$]} At $\eta = \lambda_1$ the factor $(1+2\eta\theta_1)$ vanishes, killing every term of~\eqref{eq:phidef} except the second one, so
\begin{align}
    \Phi_{S,\hat\beta}(\lambda_1) &= -\lambda_1 {r}_1^2 (1+\theta_1\lambda_1)(1+2\lambda_1\theta_2)^2 \notag\\
    &= -\frac{\lambda_1 {r}_1^2}{2}\Big(1 - \frac{\theta_2}{\theta_1}\Big)^2 \label{eq:polevalue}
\end{align}
using $\lambda_1 = -1/(2\theta_1)$, and symmetrically for $\lambda_2$ with the indices $1$ and $2$ exchanged. Because $\lambda_j{\neq}0$ and $\theta_1\theta_2 < 0$ by~\eqref{eq:eigenvalues}, $\lambda_j$ is a root of $\Phi_{S,\hat\beta}$ if and only if ${r}_j = 0$, which is exactly Case~2.2; and when ${r}_j = 0$, every remaining term of~\eqref{eq:phidef} carries the factor $(1+2\eta\theta_j)^2$, making $\lambda_j$ a double root.

\paragraph{[Counting when $\Phi_{S,\hat\beta} \not\equiv 0$]} The polynomial has at most five roots counted with multiplicity. Each root $\eta \notin \{\lambda_1,\lambda_2\}$ contributes exactly one candidate through Case~2.3. Each $\lambda_j$ that is a root uses two roots and contributes at most two candidates through~\eqref{eq:case21}. Summing over all configurations, the number of candidates is at most five.

\paragraph{[Counting when $\Phi_{S,\hat\beta}{\equiv}0$]} We have $c = \Phi_{S,\hat\beta}(0) = 0$ and ${r}_1={r}_2=0$ by~\eqref{eq:polevalue}. The remaining (last) term of~\eqref{eq:phidef} forces $\sum_{j\ge3}{r}_j^2=0$ and hence $r=0$. The stationarity condition~\eqref{eq:stat-u} becomes $(I+2\eta M)u=0$, which in the eigen coordinates reads $u_j(1+2\eta\theta_j) = 0$ for every $j$. If $\eta \notin \{\lambda_1,\lambda_2\}$, all the coefficients $1+2\eta\theta_j$ are nonzero and $u = 0$. If $\eta = \lambda_j$ for some $j \in \{1,2\}$, the $j$-th coefficient vanishes and leaves $u_j$ free, while every other coefficient is nonzero as shown in Case~2.2; the solution is thus $u = u_j v_j$, a multiple of the $j$-th eigenvector of $M$. Feasibility~\eqref{eq:tildeb} then gives
\begin{align}
    \tilde b(u_j v_j, \hat\beta) = \theta_j u_j^2 + u_j\, {r}_j + c = \theta_j u_j^2 = 0
\end{align}
using $r = 0$ and $c = 0$. $\theta_j \neq 0$ forces $u_j = 0$. So $u = 0$ in every case, and $S$ carries the single stationary point $u_0$.

In all scenarios $S$ carries at most five stationary points, each obtained in closed form from the roots of a polynomial of degree at most five.

\paragraph{[Part (b)]}
By assumption, $w^\star$ is a global optimum of~\eqref{problem} with support $S=S(w^\star)$, and $S$ is admissible by Theorem~\ref{thm:admissible}. Because $w^\star_k > 0$ for all $k \in S$, the non-negativity constraints are inactive: $w^\star$ locally minimizes $\tfrac12\|w\|^2$ subject to $\mathbf 1^\top w = 1$ and $b(w,\hat\beta)=0$ on~$S$. Write $u^\star$ for its reduced coordinate, $w^\star = u_0 + Nu^\star$. We consider two scenarios:

\paragraph{[The constraint gradients are linearly independent]} The linear independence constraint qualification holds at $w^\star$, so $u^\star$ is a stationary point of~\eqref{eq:probinu} and is enumerated by Part~(a).

\paragraph{[The constraint gradients are linearly dependent]} Dependence forces $\nabla_w b(w^\star, \hat\beta) = \mu \mathbf 1$ for some $\mu \in \mathbb R$. The chain rule gives the reduced gradient
\begin{align}
    \nabla_u \tilde b(u^\star, \hat\beta) &= \nabla_u\big[b(u_0 + Nu, \hat\beta)\big]\big|_{u = u^\star} \notag\\
    &= N^\top \nabla_w b(w^\star, \hat\beta) = \mu N^\top \mathbf 1\label{eq:redgradzero}
\end{align}
which is equal to zero since the columns of $N$ are orthogonal to~$\mathbf 1$.
Together with feasibility $\tilde b(u^\star,\hat\beta)=0$, expanding the quadratic~\eqref{eq:tildeb} around $u^\star$ along a direction $\delta$ with step size $t$ gives
\begin{align}
    &\tilde b(u^\star + t\delta, \hat\beta) \notag\\
    &= \underbrace{\tilde b(u^\star,\hat\beta)}_{=0} + \underbrace{t\,\nabla_u \tilde b(u^\star,\hat\beta)^\top \delta}_{=0} + \tfrac{t^2}{2}\delta^\top \underbrace{\nabla_u^2 \tilde b(u^\star,\hat\beta)}_{=2M}\delta \notag\\
    &= t^2 \delta^\top M \delta
\end{align}
Call $\delta$ a {null direction} if $\delta^\top M \delta = 0$. Along such a direction the line $u(t) = u^\star + t\delta$ satisfies the constraint $\tilde b(\cdot,\hat\beta)=0$ for every $t \in \mathbb R$. Because $w^\star_k > 0$ for all $k \in S$, the points of this line with $|t|$ small are feasible for~\eqref{problem}, so $t=0$ is a local minimizer of
\begin{align}
    \varphi(t) &= \|u^\star + t\delta\|^2 \\&= \|u^\star\|^2 + 2t\,(u^{\star\top}\delta) + t^2\|\delta\|^2
\end{align}
This function is a convex quadratic, so $\varphi'(0) = 0$, which asserts
\begin{align}
    u^{\star\top}\delta = 0 \label{eq:orthnull}
\end{align}
for every null direction $\delta$.

If $S$ is degenerate, then $M = 0$ and every direction is null, so~\eqref{eq:orthnull} forces $u^\star = 0$. Otherwise, $M$ is indefinite of rank two by~\eqref{eq:eigenvalues}, and its null directions span the reduced space: they contain the kernel directions $\{v_j\}_{j\ge3}$ together with $\sqrt{-\theta_2}\,v_1 \pm \sqrt{\theta_1}\,v_2$, whose span is all of $\mathbb R^{|S|-1}$. Orthogonality~\eqref{eq:orthnull} to a spanning set again forces $u^\star = 0$.
In both cases $w^\star$ is the uniform weighting on $S$. Moreover $r = -2Mu^\star = 0$ by~\eqref{eq:redgradzero} and $c = \tilde b(0,\hat\beta) = 0$ by feasibility, so every term of~\eqref{eq:phidef} vanishes and $\Phi_{S,\hat\beta} \equiv 0$. Part~(a) enumerates $u^\star = 0$ as the single stationary point on $S$ in this situation, both when $S$ is degenerate and when $\Phi_{S,\hat\beta} \equiv 0$.
The global optimum is therefore contained within the candidate set enumerated in Part~(a).
\end{proof}

\subsection{Proof of Theorem~\ref{thm:admissible}}
\label{app:proof-admissible}

\theoremadmissible*

\begin{proof}
For any $w_S$ with at least two positive values, \eqref{eq:beta} gives
\begin{align}
    \hat\beta = \beta(w_S) = \frac{\sigma(a; w_S)^2}{\sigma(a;w_s)^2+\sigma(f;w_s)^2}
\end{align}
Given $w_S = \{w_k\}_k$, by Proposition~\ref{lem:pairwise-gap} we have
\begin{align}
    \hat\beta = \frac{\sum_{k<k'}w_kw_{k'}(a_k{-}a_k')^2}{\sum_{k<k'} w_kw_{k'}\big[(a_k{-}a_k')^2{+}(f_k{-}f_k')^2 \big]}
\end{align}
Letting
\begin{align}
    p_{kk'}= \frac{w_kw_{k'} \big[ (a_k{-}a_k')^2 + (f_k{-}f_k')^2 \big]}{\sum_{l<l'} w_lw_{l'}\big[ (a_l{-}a_l')^2 + (f_l{-}f_l')^2 \big]}
\end{align}
with $p_{kk'}\ge0$ and $\sum_{k<k'} p_{kk'}=1$, we substitute~\eqref{eq:bkk}:
\begin{align}
    \hat\beta = \sum_{k<k'} p_{kk'} \beta_{kk'}
\end{align}
Therefore, $\hat\beta$ is a strict convex combination of pairwise balances over non-zero weights $\{\beta_{kk'}\}_{w_kw_{k'}>0}$, hence strictly between their smallest and their largest, unless they all coincide, in which case it equals their common value. We have
\begin{align}
    \{\beta_{kk'}\}_{w_kw_{k'}>0} \subseteq \{\beta_{kk'}\}_{kk'} \in [\beta_{\min}(S), \beta_{\max}(S)]
\end{align}
thus $\beta_{\min}(S) \le \hat\beta \le \beta_{\max}(S)$.
\end{proof}

\subsection{Proof of Theorem~\ref{thm:size-pruning}}
\label{app:proof-size-bound}

\theoremsizebound*

\begin{proof}
$S$ is the support of $w$, meaning $w_k = 0$ for $k \notin S$. The collision probability is defined as:
\begin{align}
\mathrm{P}_\mathrm{collis}(w) = \sum_{k} w_k^2 = \sum_{k \in S} w_k^2
\end{align}
By Cauchy–Schwarz inequality:
\begin{align}
\big(\underbrace{\sum_{k\in S} w_k}_{=1} \cdot 1\big)^2 \le \big(\underbrace{\sum_{k\in S} w_k^2}_{\mathrm{P}_\mathrm{collis}(w)}\big) \big(\underbrace{\sum_{k\in S} 1^2}_{|S|}\big)\label{eq:Cauchy–Schwarz}
\end{align}
Equality holds in the Cauchy–Schwarz inequality iff the vectors $(w_k)_{k\in S}$ and $(1)_{k\in S}$ are proportional, meaning $w_k$ is a constant across all $k \in S$. Given $\sum_{k\in S} w_k = 1$, this forces $w_k = 1/|S|$ for all $k \in S$. Therefore, equality holds iff $w$ is uniform on $S$. 
Dividing both sides of~\eqref{eq:Cauchy–Schwarz} by $|S|$ yields the claim.
\end{proof}

\begin{table*}[t]
\centering
\resizebox{\textwidth}{!}{%
\begin{tabular}{llll rrrr rrrr}
\toprule
\multicolumn{4}{l}{$\beta$-sweep range:} & \multicolumn{4}{c}{$\hat\beta \in [\beta_0 - 0.05,\ \beta_0 + 0.05]$} & \multicolumn{4}{c}{Full range $\hat\beta \in [\beta_{\min}, \beta_{\max}]$} \\
\cmidrule(lr){5-8} \cmidrule(lr){9-12}
& \multicolumn{3}{c}{Config} & \# Supports & \# Candidates & Success \% & \# FLOPs & \# Supports & \# Candidates & Success \% & \# FLOPs \\
\midrule
& \multicolumn{3}{l}{Theorem}\\
& \ref{thm:admissible} & \ref{thm:size-pruning} & \ref{thm:certificate}\\
\multirow{8}{*}{\rotatebox[origin=c]{90}{Our algorithm}}
& \checkmark & \checkmark & \checkmark & 50 & 50 & 100\% & 1.3M & 951.4K & 2.1M & 100\% & 26{,}304.8M \\
& \checkmark & \checkmark & -- & 50 & 50 & 100\% & 1.3M & 951.7K & 2.1M & 100\% & 26{,}313.2M \\
& -- & \checkmark & \checkmark & 50 & 50 & 100\% & 1.3M & 1.0M & 2.3M & 100\% & 28{,}146.9M \\
& \checkmark & -- & \checkmark & 50 & 50 & 100\% & 1.3M & 1.8M & 3.9M & 100\% & 48{,}911.1M \\
& -- & \checkmark & -- & 50 & 50 & 100\% & 1.3M & 1.0M & 2.3M & 100\% & 28{,}155.4M \\
& -- & -- & \checkmark & 50 & 50 & 100\% & 1.3M & 1.9M & 4.3M & 100\% & 52{,}520.6M \\
& \checkmark & -- & -- & 3.2M & 4.8M & 100\% & 87.4B & 3.0M & 5.8M & 100\% & 82{,}239.1M \\
& -- & -- & -- & 3.3M & 4.9M & 100\% & 88.8B & 3.1M & 6.2M & 100\% & 86{,}307.4M \\
\midrule
& N & \multicolumn{2}{l}{Tolerance}\\
\multirow{4}{*}{\rotatebox[origin=c]{90}{Numerical}}
& $2$ & \multicolumn{2}{c}{$10^{-3}$} & --- & --- & 18\% & 1.6M & --- & --- & 16\% & 1.6M \\
& $3$ & \multicolumn{2}{c}{$10^{-3}$} & --- & --- & 80\% & 9.3M & --- & --- & 62\% & 9.3M \\
& $4$ & \multicolumn{2}{c}{$10^{-3}$} & --- & --- & 100\% & 44.4M & --- & --- & 88\% & 44.4M \\
& $5$ & \multicolumn{2}{c}{$10^{-3}$} & --- & --- & 100\% & 177.5M & --- & --- & 100\% & 177.5M \\
\bottomrule
\end{tabular}%
}
\caption{Ablation study of our algorithm on the ende24 dataset ($K=16$). We sweep $\hat\beta$ across 50 values in each of two distinct regimes. \# Supports (\# Candidates) is the number of supports (candidates) visited prior to termination.}
\label{tab:solve-exact-operation-counts}
\end{table*}

\begin{table}[t]
\centering
\begin{tabular}{l rrr}
\toprule
& \multicolumn{3}{c}{Theorem} \\
\cmidrule(lr){2-4}
$\beta$-sweep range & \ref{thm:admissible} & \ref{thm:size-pruning} & \ref{thm:certificate} \\
\midrule
$[\beta_0{-}0.05,\ \beta_0{+}0.05]$ & 0 & 0 & 50 \\
$[\beta_{\min}, \beta_{\max}]$ & 65.5K & 29 & 19 \\
\bottomrule
\end{tabular}
\caption{Number of times each of Theorems~\ref{thm:admissible}--\ref{thm:certificate} fires over a 50-point $\beta$-sweep.}
\label{tab:solve-exact-theorem-invocations}
\end{table}

\subsection{Proof of Theorem~\ref{thm:certificate}}
\label{app:proof-certificate}

\theoremcertificate*

\begin{proof}
The balance constraint $b(w, \hat\beta) = 0$ is governed by an indefinite quadratic matrix $H = \hat\beta ff^\top - (1-\hat\beta)aa^\top$, according to~\eqref{eq:binhg} in the proof of Theorem~\ref{thm:candidates}. This makes the constraint domain non-convex.
Therefore, we verify the Karush–Kuhn–Tucker~\citep[KKT;][]{WU200746} conditions, as well as the second-order sufficiency criteria adapted for non-convex quadratically constrained quadratic program~\citep[QCQPs;][]{park2017general}. 

\paragraph{[KKT Conditions]}
Consider the Lagrangian for Problem~\eqref{problem}:
\begin{align}
\mathcal L&(w, \eta, \lambda, \gamma) =\tag{$\mathcal L$}\label{eq:Lagrangian}\\&\tfrac12 \Vert{}w\Vert{}^2 + \eta \, b(w, \hat\beta) + \lambda \mathbf 1^\top w=1 - \sum_{k=1}^K \gamma_k w_k\notag
\end{align}
where $\eta,\lambda \in \mathbb R$, and $\gamma_k \ge 0$ are the dual variables for the non-negativity constraints $w_k \ge 0$ so that $\gamma_k w_k=0$ for all $k$ (complementary slackness).
Differentiating~\eqref{eq:Lagrangian} with respect to each weight $w_k$ yields the general stationarity condition
\begin{align}
    \frac{\partial \mathcal L}{\partial w_k} = w_k + \eta \frac{\partial b(w, \hat\beta)}{\partial w_k} + \lambda - \gamma_k = 0\label{eq:stationarity}
\end{align}
for all $k$. To validate this, we split the analysis into two parts based on whether $k \in S(w)$:

\paragraph{[Case $k \in S(w)$]}
By definition, $w_k > 0$. Complementary slackness forces the corresponding dual variable $\gamma_k = 0$. Substituting this into \eqref{eq:stationarity}:
\begin{align}
w_k + \eta \frac{\partial b(w, \hat\beta)}{\partial w_k} + \lambda = 0 \quad \text{for }\ k \in S(w)\label{eq:cond1out}
\end{align}
which is given by condition~(i).

\paragraph{[Case $k \not\in S(w)$]}
By substituting $w_k = 0$ into \eqref{eq:stationarity} and keeping only $\gamma_k$ on one side:
\begin{align}
    \gamma_k =\lambda + \eta \frac{\partial b(w, \hat\beta)}{\partial w_k} \ge 0 \quad \text{for } k \not\in S(w)\label{eq:cond2out}
\end{align}
which is given by condition~(ii). Note that the inequality is given by $\gamma_k$'s definition. 

Together, \eqref{eq:cond1out} and \eqref{eq:cond2out} satisfy all KKT conditions for Problem~\eqref{problem}, which are necessary but not sufficient. Next, we verify second-order sufficiency criteria.

\paragraph{[Second-order sufficiency criteria]}
We have the Hessian of~\eqref{eq:Lagrangian}
\begin{align}
    \nabla_w^2 \mathcal L = I + 2\eta H \label{eq:laghessian}
\end{align}
where $H$ is the indefinite matrix $H = \hat\beta ff^\top - (1-\hat\beta)aa^\top$, following~\eqref{eq:binhg}.
For $w$ to be a global minimizer, the Lagrangian Hessian must be positive semi-definite along all directions $v$ that preserve the equality constraints $\mathbf 1^\top v = 0$ and $\nabla b(w, \hat\beta)^\top v = 0$~\citep{more1983computing}. Through standard null-space elimination~\citep{nocedal2006numerical}, this requirement projects onto the zero-sum subspace, demanding
\begin{align}
    v^\top \big(I + 2\eta H\big) v \ge 0 \label{eq:subspacecurv}
\end{align}
for all vector directions $v$ such that $\mathbf 1^\top v = 0$.

To validate this condition without explicitly constructing high-dimensional sweeps, we invoke the structural reduction detailed in~(\ref{eq:affinetransform}--\ref{eq:tildeb}) in the proof of Theorem~\ref{thm:candidates}. As shown in that proof, eliminating the budget constraint via the affine transformation $w = u_0 + Nu$ restricts the quadratic form to the zero-sum subspace via the reduced matrix $M = N^\top H N$.
According to Theorem~\ref{thm:candidates}'s proof, Case 2, $M$ acts as the difference of two rank-1 outer products and has rank two on the relevant subspace with two non-zero eigenvalues $\theta_1 > 0 > \theta_2$. Projecting the quadratic form onto the orthonormal basis of this two-dimensional span $\{\tilde{a}, \tilde{f}\}$ transforms the test of~\eqref{eq:subspacecurv} into the explicit $2\times2$ symmetric matrix test defined in condition~(iii). That is, condition (iii) requires that $H(\eta, \hat\beta) \succeq 0$, and requiring $H(\eta, \hat\beta) \succeq 0$ is strictly equivalent to $v^\top(I + 2\eta H)v \ge 0$ for all $v$ with $\mathbf 1^\top v = 0$, which satisfies the second-order sufficiency criteria.

Since $w$ satisfies KKT conditions across both active and inactive supports via (i) and (ii), and the second-order subspace sufficiency condition via (iii), $w$ is certified to be a global optimum of Problem~\eqref{problem}.
\end{proof}

\section{Efficiency and Ablation Analysis}\label{sec:comp}

We evaluate the computational efficiency of Algorithm~\ref{alg:reweight} and ablate its three pruning mechanisms using the ende24 dataset\footnote{The choice of this dataset is due to its large $K$. The trends and conclusions are similar across the datasets.} ($K=16$). We sweep $\hat\beta$ across 50 values in two regimes: a narrow band around the natural balance $\beta_0$, and the full reachable range $[\beta_{\min}, \beta_{\max}]$. The former represents the practical use case, as performance is more reliable here (\S\ref{sec:results}) and minor adjustments reflect realistic demands. The latter serves as an extreme regime to stress-test the algorithm. We compare against a numerical optimization baseline using varying grid resolutions $N$ and a fixed tolerance of $10^{-3}$. Tables~\ref{tab:solve-exact-operation-counts} and~\ref{tab:solve-exact-theorem-invocations} summarize the results.

\paragraph{Near natural balance.}
When the target $\hat\beta$ is close to $\beta_0$, the optimal weighting remains close to uniform. Table~\ref{tab:solve-exact-theorem-invocations} shows that Theorem~\ref{thm:certificate} (exit certificate) triggers immediately on the full support for all 50 targets. Consequently, the search evaluates exactly one support per target. Table~\ref{tab:solve-exact-operation-counts} demonstrates the resulting efficiency: solving the narrow sweep requires only 1.3M FLOPs, substantially faster than the numerical baseline required to guarantee 100\% success, while strictly guaranteeing global optimality. Without early exit, the algorithm exhaustively evaluates over 3.2M supports, increasing FLOPs by nearly five orders of magnitude.

\paragraph{Full range.}
Reaching extreme $\hat\beta$ targets requires sparse weight vectors. In this regime, Theorem~\ref{thm:admissible} (admissible supports) drives the pruning, firing over 65.5K times (Table~\ref{tab:solve-exact-theorem-invocations}) to safely skip infeasible branches. The exact algorithm demands more computation than the numerical baseline across the full range, which is the cost of guaranteeing the global minimum. The ablation confirms that combining all pruning theorems reduces the full-range computational cost by more than $3\times$ compared to the unpruned exact search.

\begin{figure}[t]
\begin{center}
\includegraphics[width=1\linewidth]{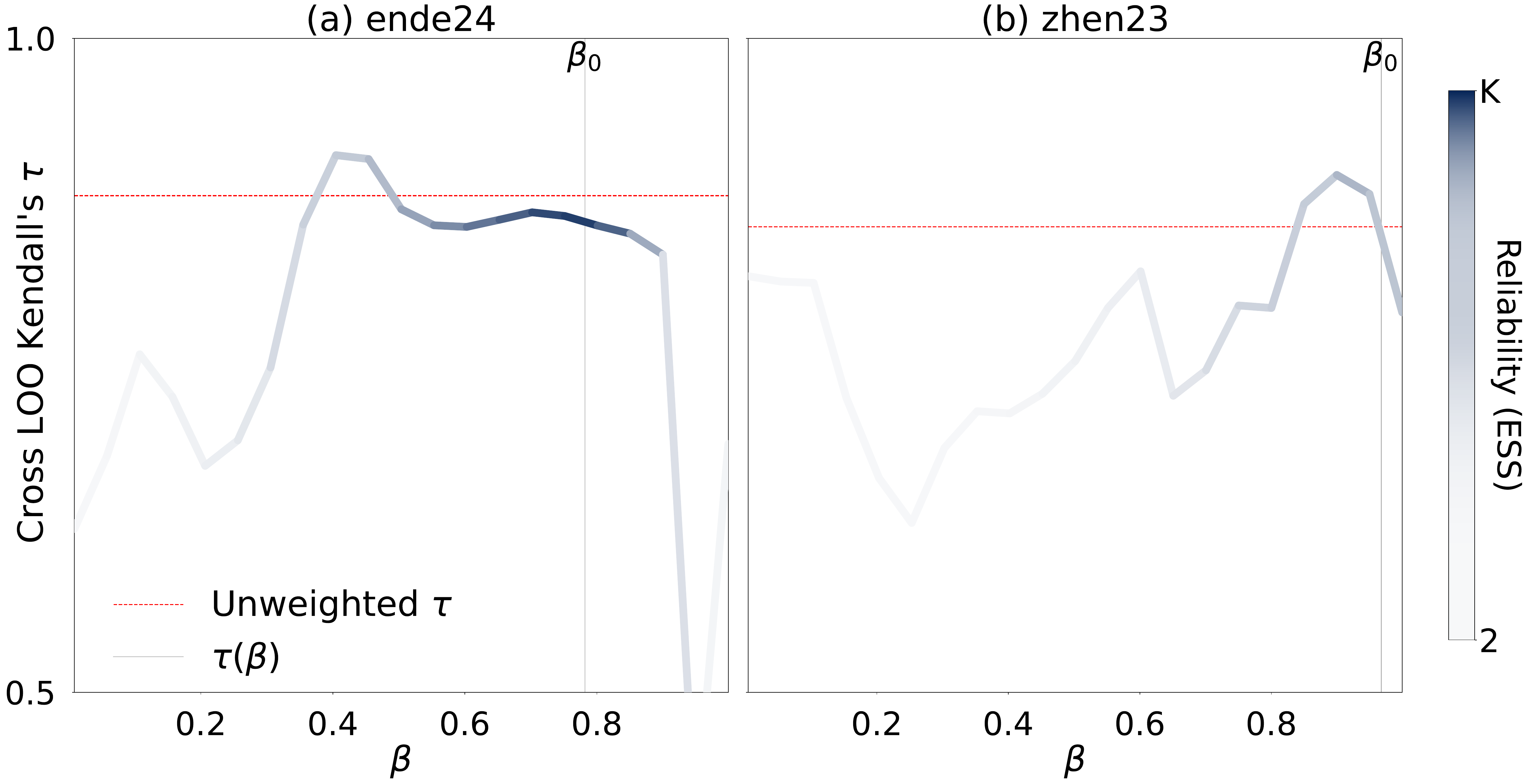}
\end{center}
\caption{Leave-one-out stability with and without fixing $\beta$. The horizontal red line corresponds to the baseline values in Table~\ref{tab:dataset_tau}. The blue curve shows the pooled pairwise Kendall's $\tau$ correlation of scorer rankings after reweighting each subset to the fixed $\beta$. Curve intensity represents $\mathrm{ESS}$ (reliability).}
\label{fig:instability_placeholder}
\end{figure}

\begin{table}[t]
\centering
\resizebox{0.85\linewidth}{!}{%
\begin{tabular}{lcc}
\toprule
Dataset & \# Systems & Cross-LOO Kendall's $\tau$ \\
\midrule
ende21 & 13 & 0.8241 \\
zhen21 & 13 & 0.6659 \\
ted\_ende & 13 & 0.8634 \\
ted\_zhen & 13 & 0.8774 \\
ende22 & 13 & 0.7497 \\
zhen22 & 14 & 0.8864 \\
zhen23 & 15 & 0.8560 \\
heen23 & 12 & 0.8470 \\
ende24 & 16 & 0.8797 \\
enes24 & 13 & 0.8815 \\
jazh24 & 13 & 0.7931 \\
\bottomrule
\end{tabular}
}
\caption{Pooled pairwise Kendall's $\tau$ correlation of scorer rankings across leave-one-out (LOO) system subsets for each dataset. The sub-optimal correlation values indicate that meta-evaluation rankings are sensitive to system selection.}
\label{tab:dataset_tau}
\end{table}

\section{Does Controlling Adequacy--Fluency Balance Stabilize Meta-Evaluation?}\label{sec:loo}

Meta-evaluation rankings vary across datasets. While intrinsic differences partially explain this, instability also stems directly from system selection within a single dataset. We quantify this by computing the pooled pairwise Kendall's $\tau$~\citep{kendall1938new} of scorer rankings across all leave-one-out (LOO) system subsets for each dataset. Table~\ref{tab:dataset_tau} reports these results. These values confirm that meta-evaluation outcomes are sensitive to the specific MT systems included in the pool.

A reasonable hypothesis is that this within-dataset instability is partly driven by varying adequacy--fluency balances across subsets. To test this, we evaluate whether fixing $\beta$ stabilizes the scorer rankings across the LOO subsets.\footnote{In preliminary studies, we also controlled the adequacy-over-fluency preference directly (see \S\ref{sec:experiments}) to account for different subsets possessing different balance--$\beta$ relations. This alternative control also failed to resolve the instability.}
Our results reject this hypothesis. As illustrated for two representative datasets in Figure~\ref{fig:instability_placeholder}, precisely controlling $\beta$ does not resolve the ranking variance. While the adequacy--fluency balance may contribute to the instability of specific datasets, it is not the primary driver of the system-selection sensitivity observed in recent WMT evaluations. Identifying the root cause of this instability remains an open problem beyond the scope of this work.

\begin{table*}[t]
\centering
\begin{tabular}{lrrrrrrr}
\toprule
Dataset & \#Systems & \#Scorers & \#Segments & $\beta_0$ & $\beta_{\min}$ & $\beta_{\max}$ & $\beta_{\max}-\beta_{\min}$ \\
\midrule
ende24 & 16 & 30 & 487 & 0.7821 & 0.0074 & 1.0000 & 0.9926 \\
enes24 & 13 & 30 & 623 & 0.9657 & 0.0005 & 1.0000 & 0.9994 \\
jazh24 & 13 & 31 & 560 & 0.9954 & 0.4059 & 1.0000 & 0.5941 \\
heen23 & 12 & 43 & 820 & 0.9868 & 0.0022 & 1.0000 & 0.9978 \\
zhen23 & 15 & 43 & 1177 & 0.9672 & 0.0034 & 1.0000 & 0.9966 \\
ende22 & 13 & 31 & 1315 & 0.9338 & 0.0010 & 1.0000 & 0.9990 \\
zhen22 & 14 & 31 & 1875 & 0.9767 & 0.0071 & 1.0000 & 0.9928 \\
ende21 & 13 & 27 & 527 & 0.2273 & 0.0000 & 0.9996 & 0.9996 \\
zhen21 & 13 & 28 & 650 & 0.7054 & 0.0008 & 1.0000 & 0.9991 \\
ted\_ende & 13 & 27 & 529 & 0.3802 & 0.0034 & 1.0000 & 0.9966 \\
ted\_zhen & 13 & 28 & 529 & 0.4921 & 0.0063 & 0.9995 & 0.9932 \\
\bottomrule
\end{tabular}%
\caption{Dataset size and reachable balance range. \#Scorers is the number of base scorers used for augmentation, listed by name in Table~\ref{tab:dataset-scorers}.}
\label{tab:dataset-core}
\end{table*}

\begin{table*}[t]
\centering
\begin{tabular}{lrrrrr}
\toprule
Dataset & \#Pairs & Mean & Median & Std & IQR \\
\midrule
ende24 & 120 & 0.6223 & 0.6883 & 0.3083 & 0.5019 \\
enes24 & 78 & 0.8444 & 0.9705 & 0.2642 & 0.1335 \\
jazh24 & 78 & 0.9317 & 0.9922 & 0.1538 & 0.0292 \\
heen23 & 66 & 0.8971 & 0.9919 & 0.2065 & 0.0917 \\
zhen23 & 105 & 0.8392 & 0.9508 & 0.2416 & 0.1835 \\
ende22 & 78 & 0.8049 & 0.9115 & 0.2523 & 0.2376 \\
zhen22 & 91 & 0.8576 & 0.9499 & 0.2160 & 0.1371 \\
ende21 & 78 & 0.3543 & 0.3105 & 0.2916 & 0.4552 \\
zhen21 & 78 & 0.5984 & 0.7165 & 0.3347 & 0.4885 \\
ted\_ende & 78 & 0.4672 & 0.3872 & 0.3667 & 0.7545 \\
ted\_zhen & 78 & 0.5278 & 0.5340 & 0.3273 & 0.5987 \\
\bottomrule
\end{tabular}%
\caption{Statistics describing the distribution of the pairwise balance $\beta_{kk'}$ across system pairs. IQR: their interquartile range, i.e. the $75$th percentile minus the $25$th.}
\label{tab:dataset-betaij}
\end{table*}

\begin{table*}[t]
\centering
\begin{tabular}{lrrrrrr}
\toprule
Dataset & $\mu(a)$ & $\sigma(a)^2$ & $\mu(f)$ & $\sigma(f)^2$ & $\rho$ & $\Sigma(a,f)$ \\
\midrule
ende24 & -1.3745 & 0.4475 & -1.1275 & 0.1247 & 0.7210 & 0.1703 \\
enes24 & -0.4547 & 0.1167 & -0.1262 & 0.0042 & 0.5531 & 0.0122 \\
jazh24 & -2.5512 & 2.3894 & -0.3477 & 0.0110 & 0.5749 & 0.0933 \\
heen23 & -2.0684 & 0.7275 & -0.3771 & 0.0097 & 0.4214 & 0.0355 \\
zhen23 & -2.7212 & 1.5735 & -1.4967 & 0.0534 & 0.2707 & 0.0784 \\
ende22 & -0.8175 & 0.0494 & -0.4051 & 0.0035 & 0.6182 & 0.0081 \\
zhen22 & -2.6217 & 0.8798 & -0.6556 & 0.0210 & 0.8144 & 0.1108 \\
ende21 & -0.5341 & 0.0364 & -1.1930 & 0.1238 & 0.6377 & 0.0428 \\
zhen21 & -3.3916 & 0.1510 & -1.7004 & 0.0631 & 0.0837 & 0.0082 \\
ted\_ende & -0.7215 & 0.0262 & -0.8389 & 0.0426 & 0.2913 & 0.0097 \\
ted\_zhen & -1.1248 & 0.0461 & -1.0927 & 0.0476 & 0.7116 & 0.0333 \\
\bottomrule
\end{tabular}%
\caption{System-level Adequacy and Fluency MQM. Every statistic is taken over the $K$ systems of a single dataset, at the uniform weighting $w=\tfrac1K\mathbf 1$, on the system-level score vectors $a$ (Adequacy~MQM) and $f$ (Fluency~MQM). $\rho$ is the Pearson correlation between $a$ and $f$ across systems. $\Sigma(a,f)$: the covariance between $a$ and $f$ across those same systems. Means are negative because MQM is defined as a negated error weight.}
\label{tab:dataset-mqm}
\end{table*}

\section{Dataset Statistics}
\label{app:datasetstats}

Tables~\ref{tab:dataset-core}--\ref{tab:dataset-scorers} report descriptive
statistics for the eleven datasets used throughout this work.
Table~\ref{tab:dataset-core} covers pool size and the balance parameter. Table~\ref{tab:dataset-betaij} summarizes the
distribution of pairwise balances~$\beta_{kk'}$.
Table~\ref{tab:dataset-mqm} reports system-level Adequacy~MQM and Fluency~MQM statistics, together with their correlation and
covariance across systems. Table~\ref{tab:dataset-scorers} lists by name the
base scorers augmented for each dataset.

\begin{table*}[t]
\centering
\footnotesize
\begin{tabular}{l p{0.84\textwidth}}
\toprule
Dataset & Base scorers \\
\midrule
ende24 & \texttt{BERTScore}, \texttt{BLCOM}, \texttt{BLCOM\_1}, \texttt{BLEU}, \texttt{BLEURT-20}, \texttt{COMET-22}, \texttt{CometKiwi}, \texttt{CometKiwi-XXL}, \texttt{MEE4}, \texttt{MetricX-24}, \texttt{MetricX-24-Hybrid}, \texttt{MetricX-24-Hybrid-QE}, \texttt{MetricX-24-QE}, \texttt{PrismRefMedium}, \texttt{PrismRefSmall}, \texttt{XCOMET}, \texttt{XCOMET-QE}, \texttt{XLsimDA}, \texttt{YiSi-1}, \texttt{bright-qe}, \texttt{chrF}, \texttt{chrfS}, \texttt{damonmonli}, \texttt{gemba\_esa}, \texttt{metametrics\_mt\_mqm\_hybrid\_kendall}, \texttt{metametrics\_mt\_mqm\_kendall}, \texttt{metametrics\_mt\_mqm\_qe\_kendall.seg.s}, \texttt{monmonli}, \texttt{sentinel-cand-mqm}, \texttt{spBLEU} \\
\addlinespace
enes24 & \texttt{BERTScore}, \texttt{BLCOM}, \texttt{BLCOM\_1}, \texttt{BLEU}, \texttt{BLEURT-20}, \texttt{COMET-22}, \texttt{CometKiwi}, \texttt{CometKiwi-XXL}, \texttt{MEE4}, \texttt{MetricX-24}, \texttt{MetricX-24-Hybrid}, \texttt{MetricX-24-Hybrid-QE}, \texttt{MetricX-24-QE}, \texttt{PrismRefMedium}, \texttt{PrismRefSmall}, \texttt{XCOMET}, \texttt{XCOMET-QE}, \texttt{XLsimDA}, \texttt{YiSi-1}, \texttt{bright-qe}, \texttt{chrF}, \texttt{chrfS}, \texttt{damonmonli}, \texttt{gemba\_esa}, \texttt{metametrics\_mt\_mqm\_hybrid\_kendall}, \texttt{metametrics\_mt\_mqm\_kendall}, \texttt{metametrics\_mt\_mqm\_qe\_kendall.seg.s}, \texttt{monmonli}, \texttt{sentinel-cand-mqm}, \texttt{spBLEU} \\
\addlinespace
jazh24 & \texttt{BERTScore}, \texttt{BLCOM}, \texttt{BLCOM\_1}, \texttt{BLEU}, \texttt{BLEURT-20}, \texttt{COMET-22}, \texttt{CometKiwi}, \texttt{CometKiwi-XXL}, \texttt{MEE4}, \texttt{MetricX-24}, \texttt{MetricX-24-Hybrid}, \texttt{MetricX-24-Hybrid-QE}, \texttt{MetricX-24-QE}, \texttt{PrismRefMedium}, \texttt{PrismRefSmall}, \texttt{XCOMET}, \texttt{XCOMET-QE}, \texttt{XLsimDA}, \texttt{YiSi-1}, \texttt{bright-qe}, \texttt{chrF}, \texttt{chrfS}, \texttt{damonmonli}, \texttt{gemba\_esa}, \texttt{metametrics\_mt\_mqm\_hybrid\_kendall}, \texttt{metametrics\_mt\_mqm\_qe\_kendall.seg.s}, \texttt{metametrics\_mt\_mqm\_qe\_same\_source\_t}, \texttt{metametrics\_mt\_mqm\_same\_source\_targ}, \texttt{monmonli}, \texttt{sentinel-cand-mqm}, \texttt{spBLEU} \\
\addlinespace
heen23 & \texttt{BERTscore}, \texttt{BLEU}, \texttt{BLEURT-20}, \texttt{COMET}, \texttt{Calibri-COMET22}, \texttt{Calibri-COMET22-QE}, \texttt{CometKiwi}, \texttt{CometKiwi-XL}, \texttt{CometKiwi-XXL}, \texttt{GEMBA-MQM}, \texttt{KG-BERTScore}, \texttt{MEE4}, \texttt{MS-COMET-QE-22}, \texttt{MaTESe}, \texttt{MetricX-23}, \texttt{MetricX-23-QE}, \texttt{MetricX-23-QE-b}, \texttt{MetricX-23-QE-c}, \texttt{MetricX-23-b}, \texttt{MetricX-23-c}, \texttt{Random-sysname}, \texttt{XCOMET-Ensemble}, \texttt{XCOMET-QE-Ensemble}, \texttt{XCOMET-XL}, \texttt{XCOMET-XXL}, \texttt{XLsim}, \texttt{YiSi-1}, \texttt{chrF}, \texttt{cometoid22-wmt21}, \texttt{cometoid22-wmt22}, \texttt{cometoid22-wmt23}, \texttt{docWMT22CometDA}, \texttt{docWMT22CometKiwiDA}, \texttt{eBLEU}, \texttt{embed\_llama}, \texttt{f200spBLEU}, \texttt{instructscore}, \texttt{mbr-metricx-qe}, \texttt{mre-score-labse-regular}, \texttt{prismRef}, \texttt{prismSrc}, \texttt{sescoreX}, \texttt{tokengram\_F} \\
\addlinespace
zhen23 & \texttt{BERTscore}, \texttt{BLEU}, \texttt{BLEURT-20}, \texttt{COMET}, \texttt{Calibri-COMET22}, \texttt{Calibri-COMET22-QE}, \texttt{CometKiwi}, \texttt{CometKiwi-XL}, \texttt{CometKiwi-XXL}, \texttt{GEMBA-MQM}, \texttt{KG-BERTScore}, \texttt{MEE4}, \texttt{MS-COMET-QE-22}, \texttt{MaTESe}, \texttt{MetricX-23}, \texttt{MetricX-23-QE}, \texttt{MetricX-23-QE-b}, \texttt{MetricX-23-QE-c}, \texttt{MetricX-23-b}, \texttt{MetricX-23-c}, \texttt{Random-sysname}, \texttt{XCOMET-Ensemble}, \texttt{XCOMET-QE-Ensemble}, \texttt{XCOMET-XL}, \texttt{XCOMET-XXL}, \texttt{XLsim}, \texttt{YiSi-1}, \texttt{chrF}, \texttt{cometoid22-wmt21}, \texttt{cometoid22-wmt22}, \texttt{cometoid22-wmt23}, \texttt{docWMT22CometDA}, \texttt{docWMT22CometKiwiDA}, \texttt{eBLEU}, \texttt{embed\_llama}, \texttt{f200spBLEU}, \texttt{instructscore}, \texttt{mbr-metricx-qe}, \texttt{mre-score-labse-regular}, \texttt{prismRef}, \texttt{prismSrc}, \texttt{sescoreX}, \texttt{tokengram\_F} \\
\addlinespace
ende22 & \texttt{BERTScore}, \texttt{BLEU}, \texttt{BLEURT-20}, \texttt{COMET-20}, \texttt{COMET-22}, \texttt{COMET-QE}, \texttt{COMETKiwi}, \texttt{Cross-QE}, \texttt{HWTSC-TLM}, \texttt{HWTSC-Teacher-Sim}, \texttt{KG-BERTScore}, \texttt{MATESE}, \texttt{MATESE-QE}, \texttt{MEE}, \texttt{MEE2}, \texttt{MEE4}, \texttt{MS-COMET-22}, \texttt{MS-COMET-QE-22}, \texttt{REUSE}, \texttt{SEScore}, \texttt{UniTE}, \texttt{UniTE-ref}, \texttt{UniTE-src}, \texttt{YiSi-1}, \texttt{chrF}, \texttt{f101spBLEU}, \texttt{f200spBLEU}, \texttt{metricx\_xl\_DA\_2019}, \texttt{metricx\_xl\_MQM\_2020}, \texttt{metricx\_xxl\_DA\_2019}, \texttt{metricx\_xxl\_MQM\_2020} \\
\addlinespace
zhen22 & \texttt{BERTScore}, \texttt{BLEU}, \texttt{BLEURT-20}, \texttt{COMET-20}, \texttt{COMET-22}, \texttt{COMET-QE}, \texttt{COMETKiwi}, \texttt{Cross-QE}, \texttt{HWTSC-TLM}, \texttt{HWTSC-Teacher-Sim}, \texttt{KG-BERTScore}, \texttt{MATESE}, \texttt{MATESE-QE}, \texttt{MEE}, \texttt{MEE2}, \texttt{MEE4}, \texttt{MS-COMET-22}, \texttt{MS-COMET-QE-22}, \texttt{REUSE}, \texttt{SEScore}, \texttt{UniTE}, \texttt{UniTE-ref}, \texttt{UniTE-src}, \texttt{YiSi-1}, \texttt{chrF}, \texttt{f101spBLEU}, \texttt{f200spBLEU}, \texttt{metricx\_xl\_DA\_2019}, \texttt{metricx\_xl\_MQM\_2020}, \texttt{metricx\_xxl\_DA\_2019}, \texttt{metricx\_xxl\_MQM\_2020} \\
\addlinespace
ende21 & \texttt{BERTScore}, \texttt{C-SPEC}, \texttt{C-SPECpn}, \texttt{COMET-DA\_2020}, \texttt{COMET-DA\_2021}, \texttt{COMET-MQM\_2021}, \texttt{COMET-QE-DA\_2021}, \texttt{COMET-QE-MQM\_2021}, \texttt{COMETinho-DA}, \texttt{COMETinho-MQM}, \texttt{MEE}, \texttt{MEE2}, \texttt{MTEQA}, \texttt{OpenKiwi-MQM}, \texttt{Prism}, \texttt{TER}, \texttt{YiSi-1}, \texttt{YiSi-2}, \texttt{bleurt-20}, \texttt{bleurt-21-beta}, \texttt{chrF}, \texttt{cushLEPOR(LM)}, \texttt{cushLEPOR(pSQM)}, \texttt{hLEPOR}, \texttt{regEMT}, \texttt{regEMT-baseline}, \texttt{sentBLEU} \\
\addlinespace
zhen21 & \texttt{BERTScore}, \texttt{C-SPEC}, \texttt{C-SPECpn}, \texttt{COMET-DA\_2020}, \texttt{COMET-DA\_2021}, \texttt{COMET-MQM\_2021}, \texttt{COMET-QE-DA\_2021}, \texttt{COMET-QE-MQM\_2021}, \texttt{COMETinho-DA}, \texttt{COMETinho-MQM}, \texttt{MEE}, \texttt{MEE2}, \texttt{MTEQA}, \texttt{OpenKiwi-MQM}, \texttt{Prism}, \texttt{RoBLEURT}, \texttt{TER}, \texttt{YiSi-1}, \texttt{YiSi-2}, \texttt{bleurt-20}, \texttt{bleurt-21-beta}, \texttt{chrF}, \texttt{cushLEPOR(LM)}, \texttt{cushLEPOR(pSQM)}, \texttt{hLEPOR}, \texttt{regEMT}, \texttt{regEMT-baseline}, \texttt{sentBLEU} \\
\addlinespace
ted\_ende & \texttt{BERTScore}, \texttt{C-SPEC}, \texttt{C-SPECpn}, \texttt{COMET-DA\_2020}, \texttt{COMET-DA\_2021}, \texttt{COMET-MQM\_2021}, \texttt{COMET-QE-DA\_2021}, \texttt{COMET-QE-MQM\_2021}, \texttt{COMETinho-DA}, \texttt{COMETinho-MQM}, \texttt{MEE}, \texttt{MEE2}, \texttt{MTEQA}, \texttt{OpenKiwi-MQM}, \texttt{Prism}, \texttt{TER}, \texttt{YiSi-1}, \texttt{YiSi-2}, \texttt{bleurt-20}, \texttt{bleurt-21-beta}, \texttt{chrF}, \texttt{cushLEPOR(LM)}, \texttt{cushLEPOR(pSQM)}, \texttt{hLEPOR}, \texttt{regEMT}, \texttt{regEMT-baseline}, \texttt{sentBLEU} \\
\addlinespace
ted\_zhen & \texttt{BERTScore}, \texttt{C-SPEC}, \texttt{C-SPECpn}, \texttt{COMET-DA\_2020}, \texttt{COMET-DA\_2021}, \texttt{COMET-MQM\_2021}, \texttt{COMET-QE-DA\_2021}, \texttt{COMET-QE-MQM\_2021}, \texttt{COMETinho-DA}, \texttt{COMETinho-MQM}, \texttt{MEE}, \texttt{MEE2}, \texttt{MTEQA}, \texttt{OpenKiwi-MQM}, \texttt{Prism}, \texttt{RoBLEURT}, \texttt{TER}, \texttt{YiSi-1}, \texttt{YiSi-2}, \texttt{bleurt-20}, \texttt{bleurt-21-beta}, \texttt{chrF}, \texttt{cushLEPOR(LM)}, \texttt{cushLEPOR(pSQM)}, \texttt{hLEPOR}, \texttt{regEMT}, \texttt{regEMT-baseline}, \texttt{sentBLEU} \\
\bottomrule
\end{tabular}
\caption{The base scorers augmented for each dataset, by name. This follows the selection described in~\S\ref{sec:experiments}. Names are reproduced exactly as released by WMT.}
\label{tab:dataset-scorers}
\end{table*}

\section{Main Results for All Datasets}
\label{app:all_datasets}

Figure~\ref{fig:results_all} presents the main results across all datasets, reproducing the figures from the main text for completeness. The observed trends remain consistent throughout, demonstrating our method's generalizability.

\begin{figure*}[t]
\centering 

\begin{tabular}{@{}c@{}}
    \quad ende22 \\
    \includegraphics[height=4.9cm]{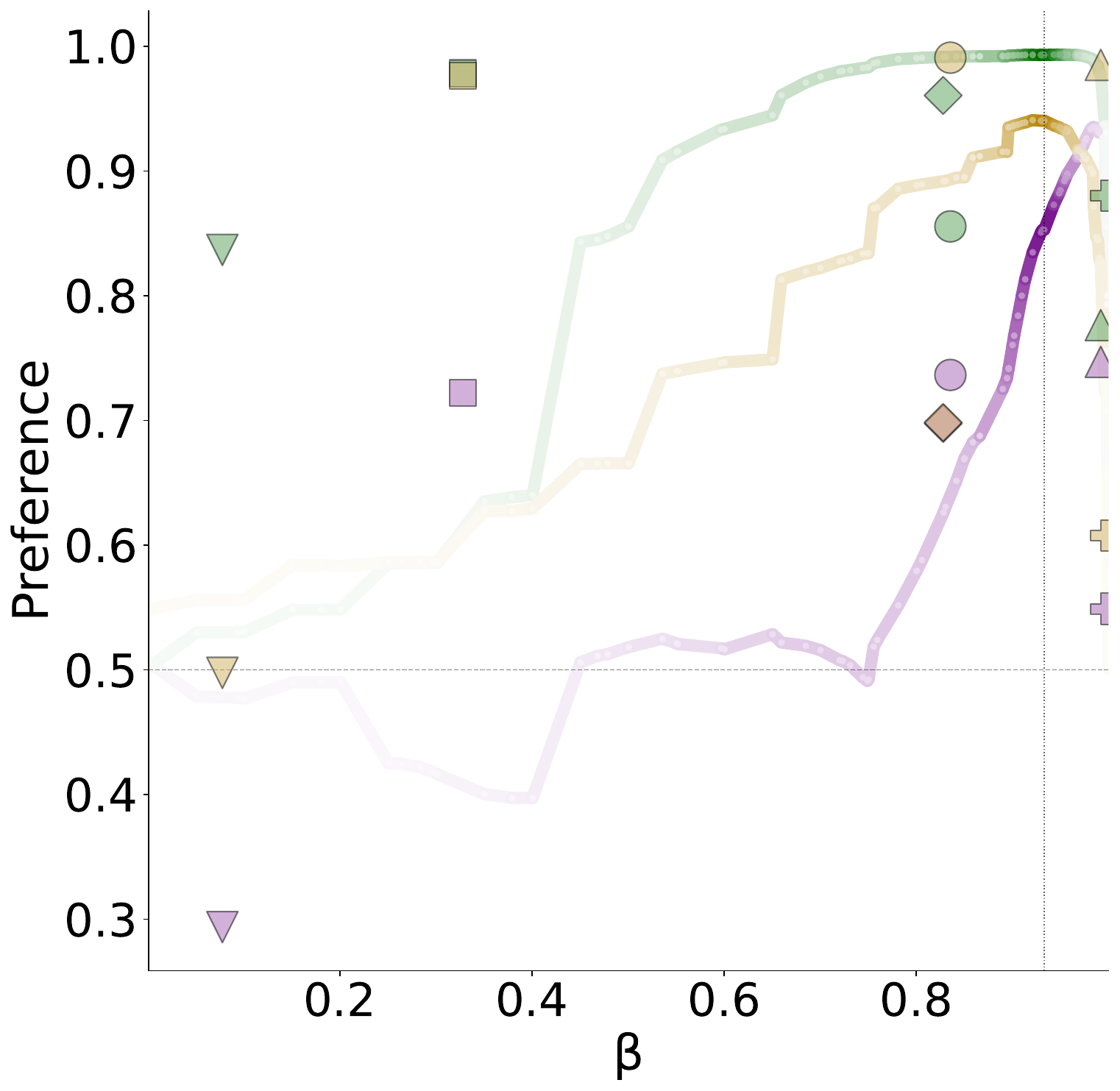}
\end{tabular}\hfill
\begin{tabular}{@{}c@{}}
    \quad zhen21 \\
    \includegraphics[height=4.9cm]{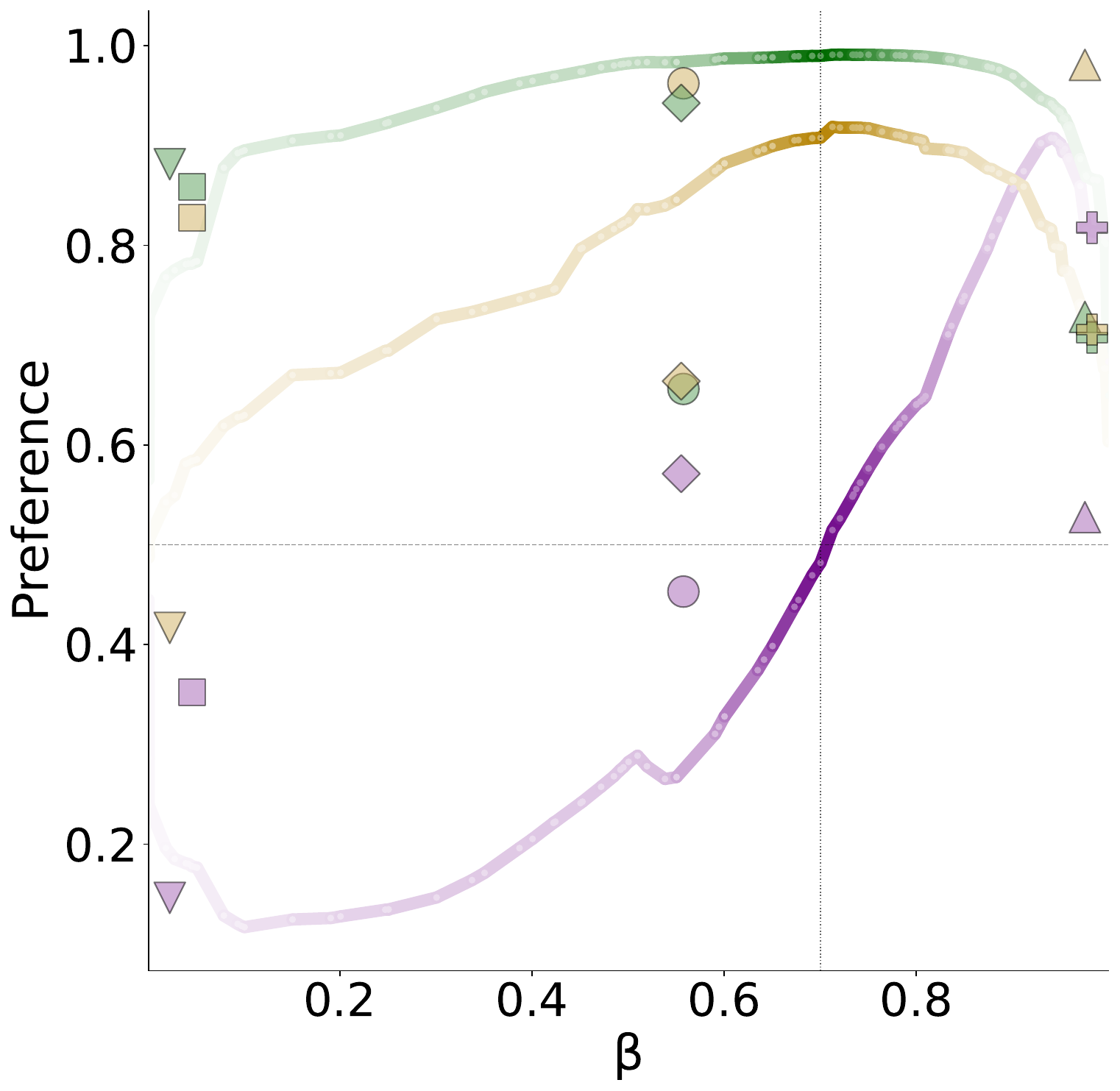}
\end{tabular}\hfill
\begin{tabular}{@{}c@{}}
    \quad enes24 \\
    \includegraphics[height=4.9cm]{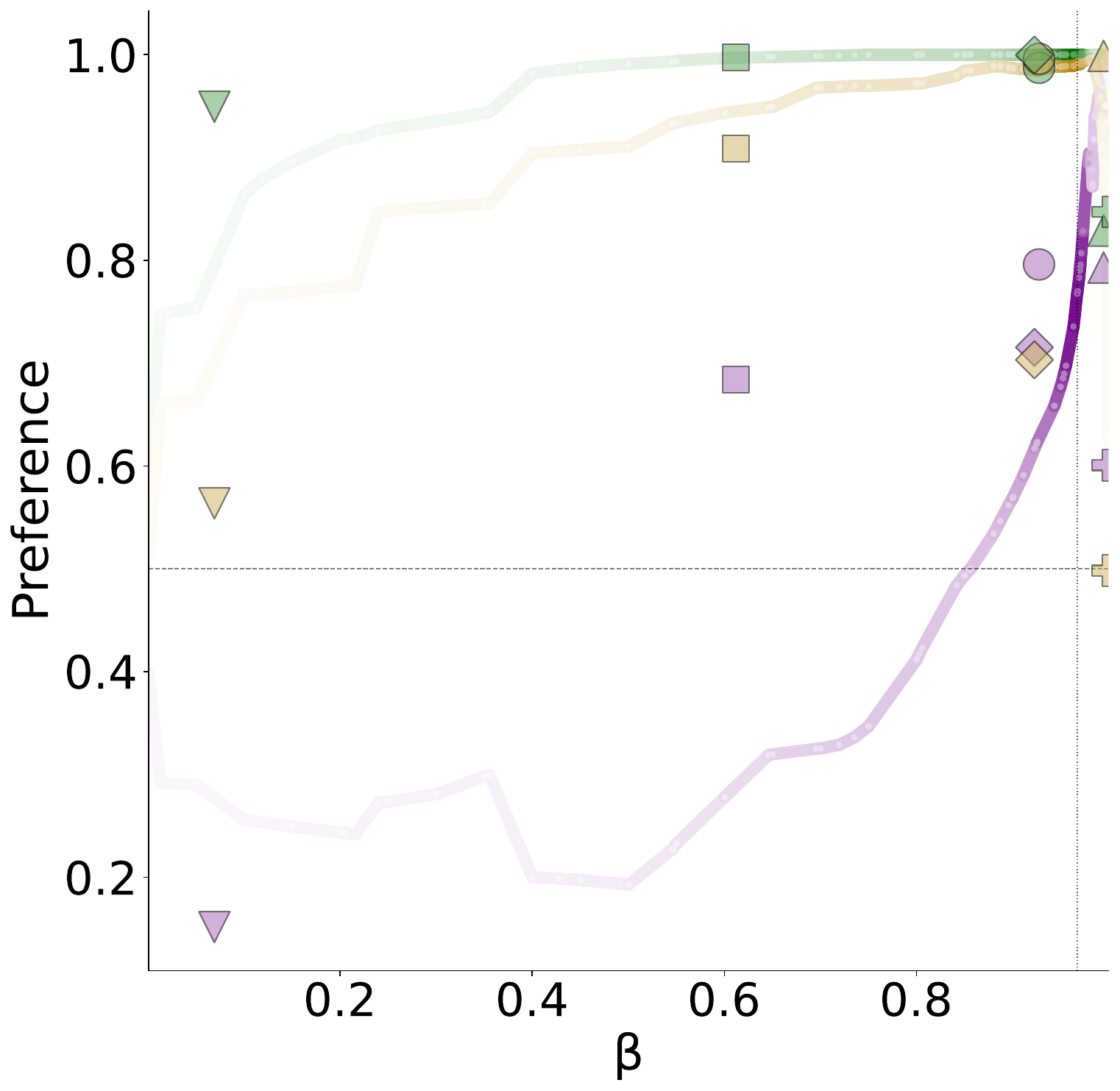}
\end{tabular}

\vspace{0.5cm} 

\begin{tabular}{@{}c@{}}
    \quad heen23 \\
    \includegraphics[height=4.9cm]{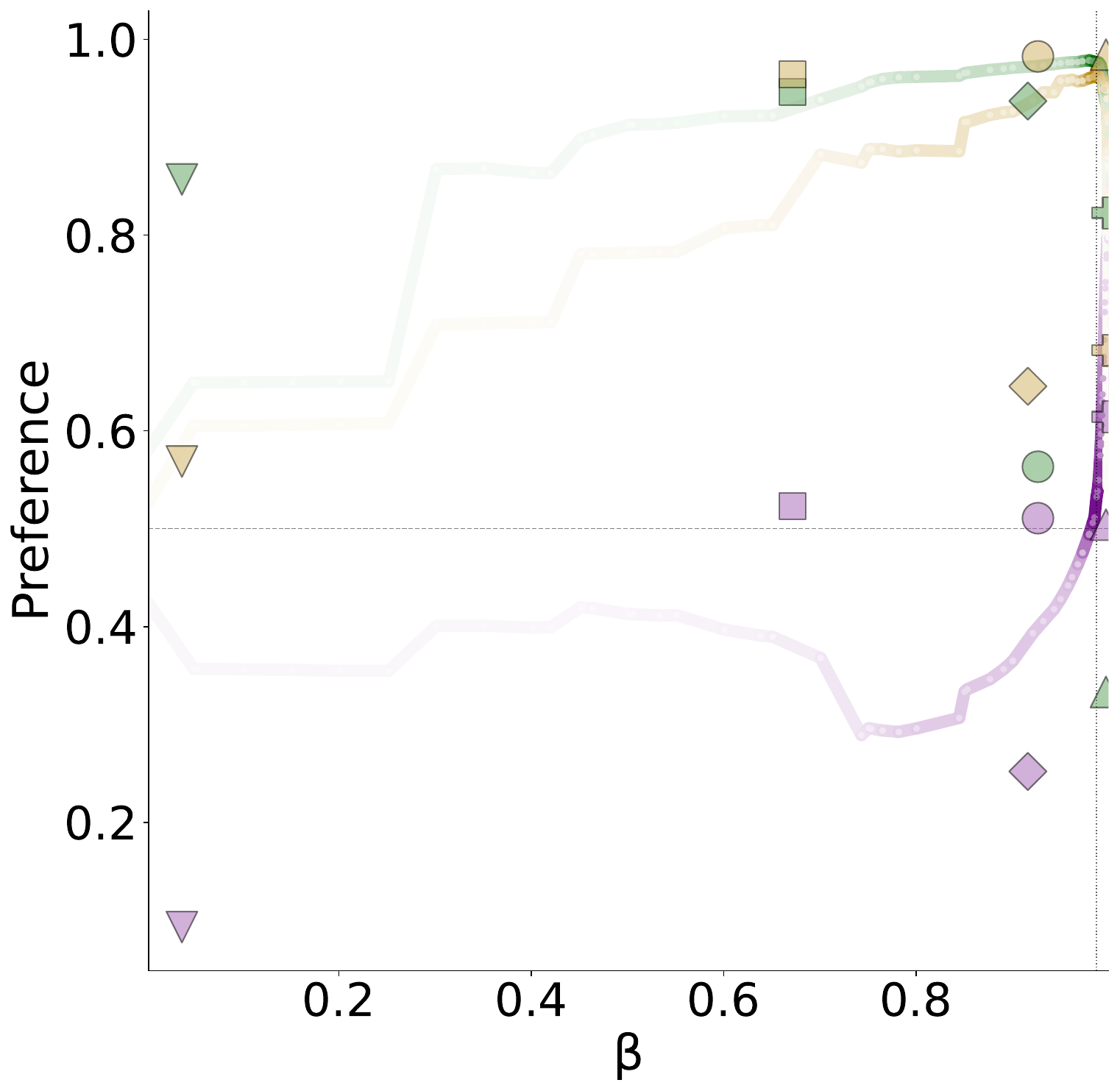}
\end{tabular}\hfill
\begin{tabular}{@{}c@{}}
    \quad jazh24 \\
    \includegraphics[height=4.9cm]{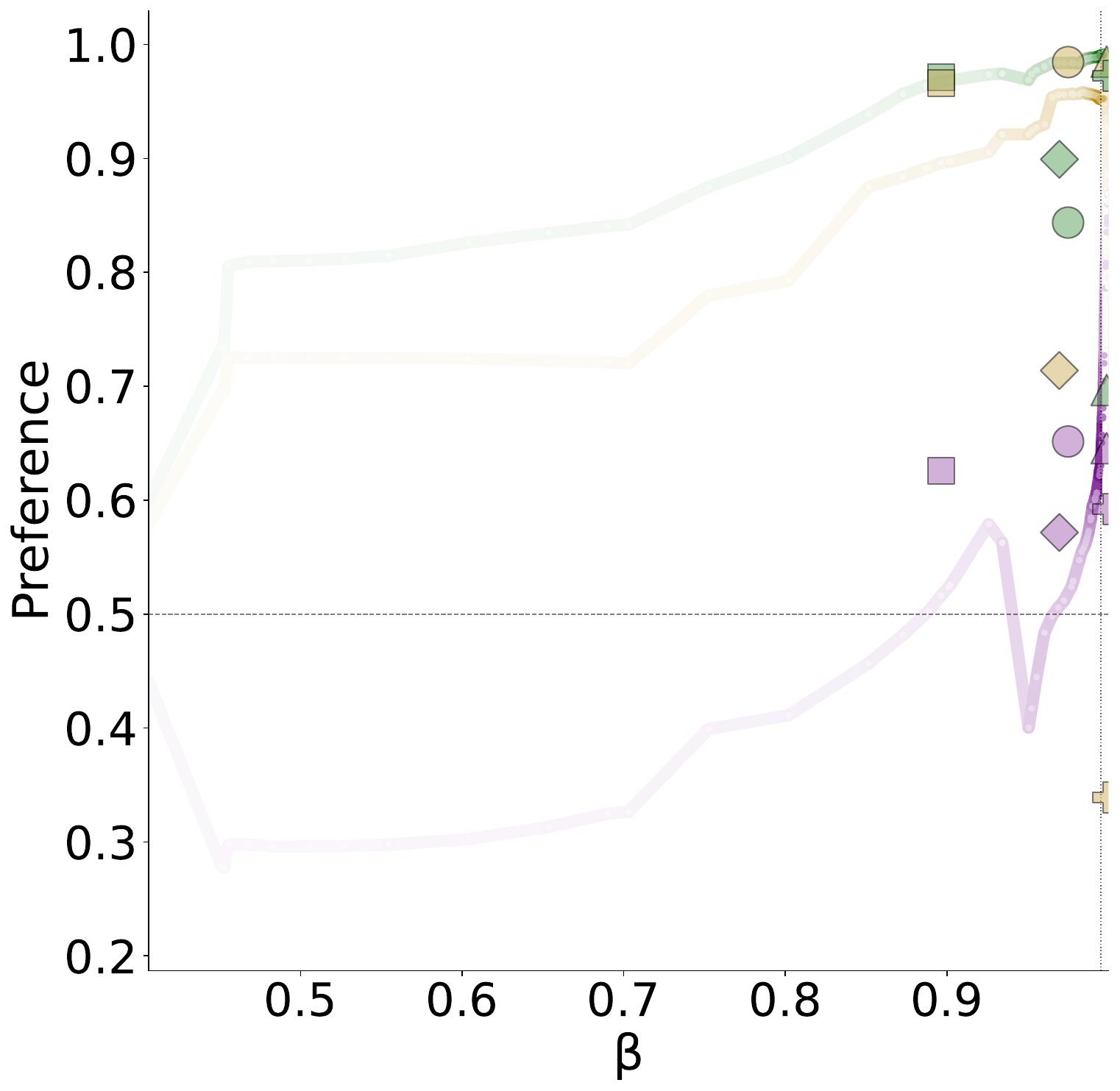}
\end{tabular}\hfill
\begin{tabular}{@{}c@{}}
    \quad ted\_ende \\
    \includegraphics[height=4.9cm]{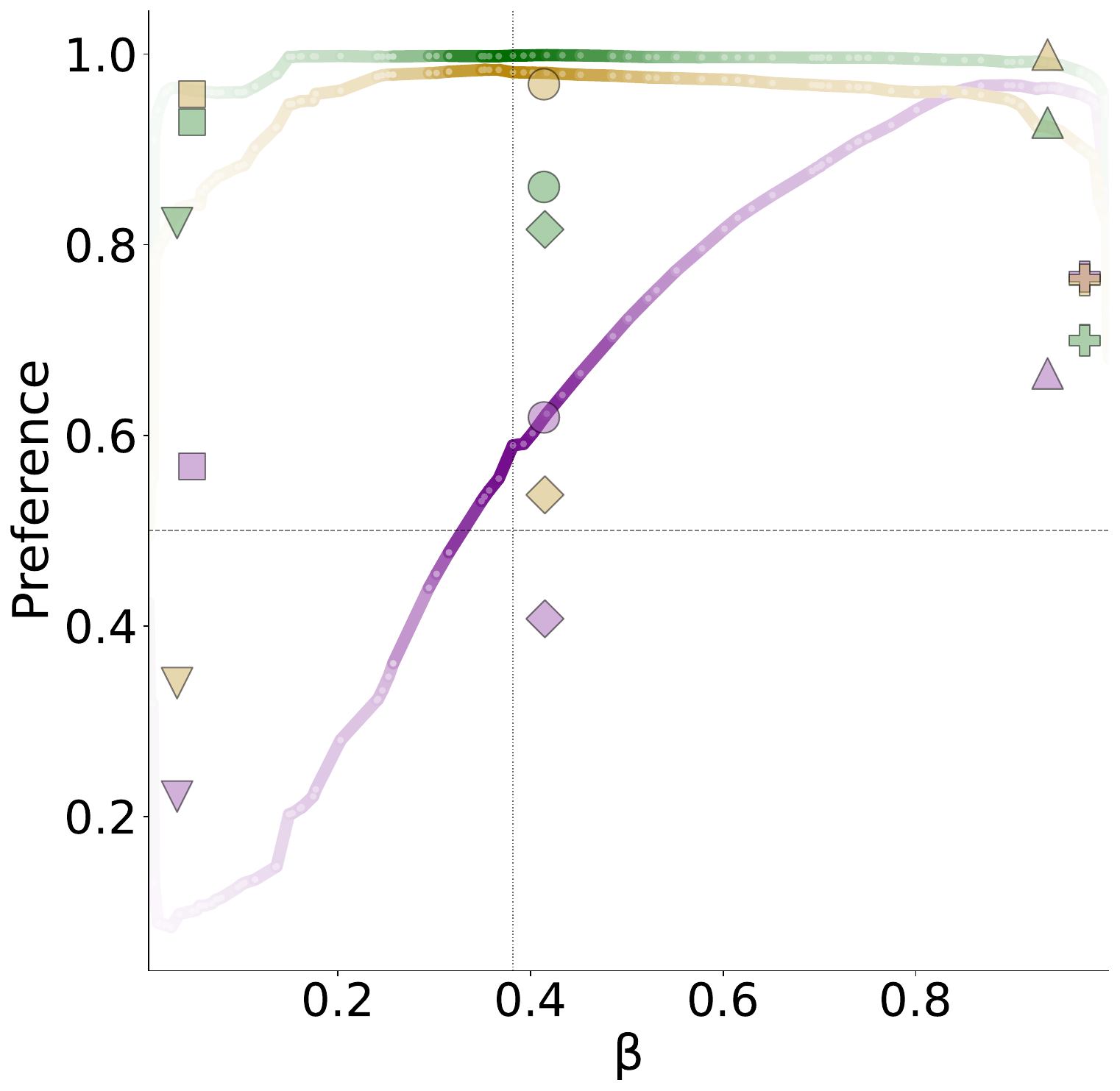}
\end{tabular}

\vspace{0.5cm} 

\begin{tabular}{@{}c@{}}
    \quad ted\_zhen \\
    \includegraphics[height=4.9cm]{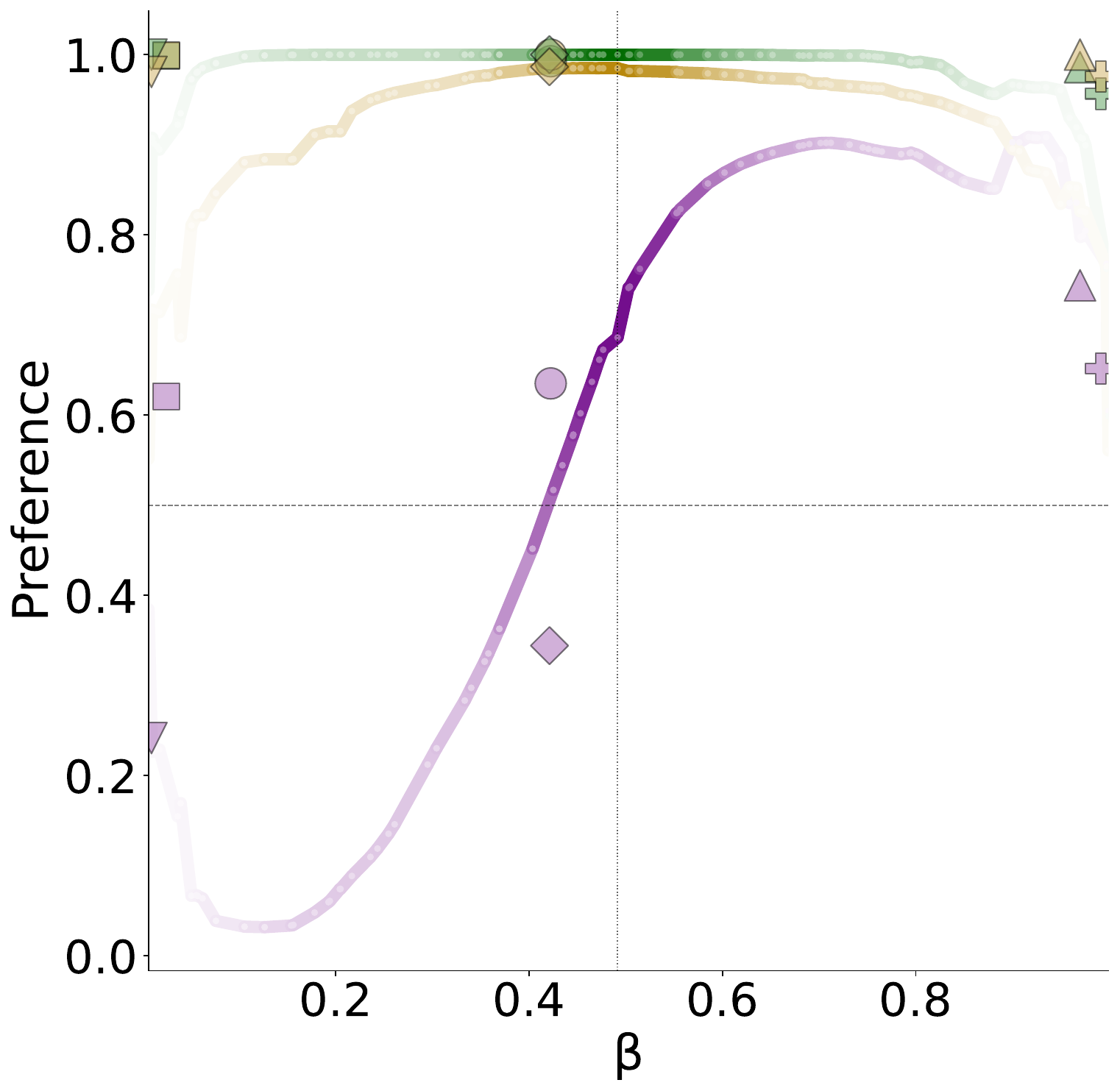}
\end{tabular}\hfill
\begin{tabular}{@{}c@{}}
    \quad zhen22 \\
    \includegraphics[height=4.9cm]{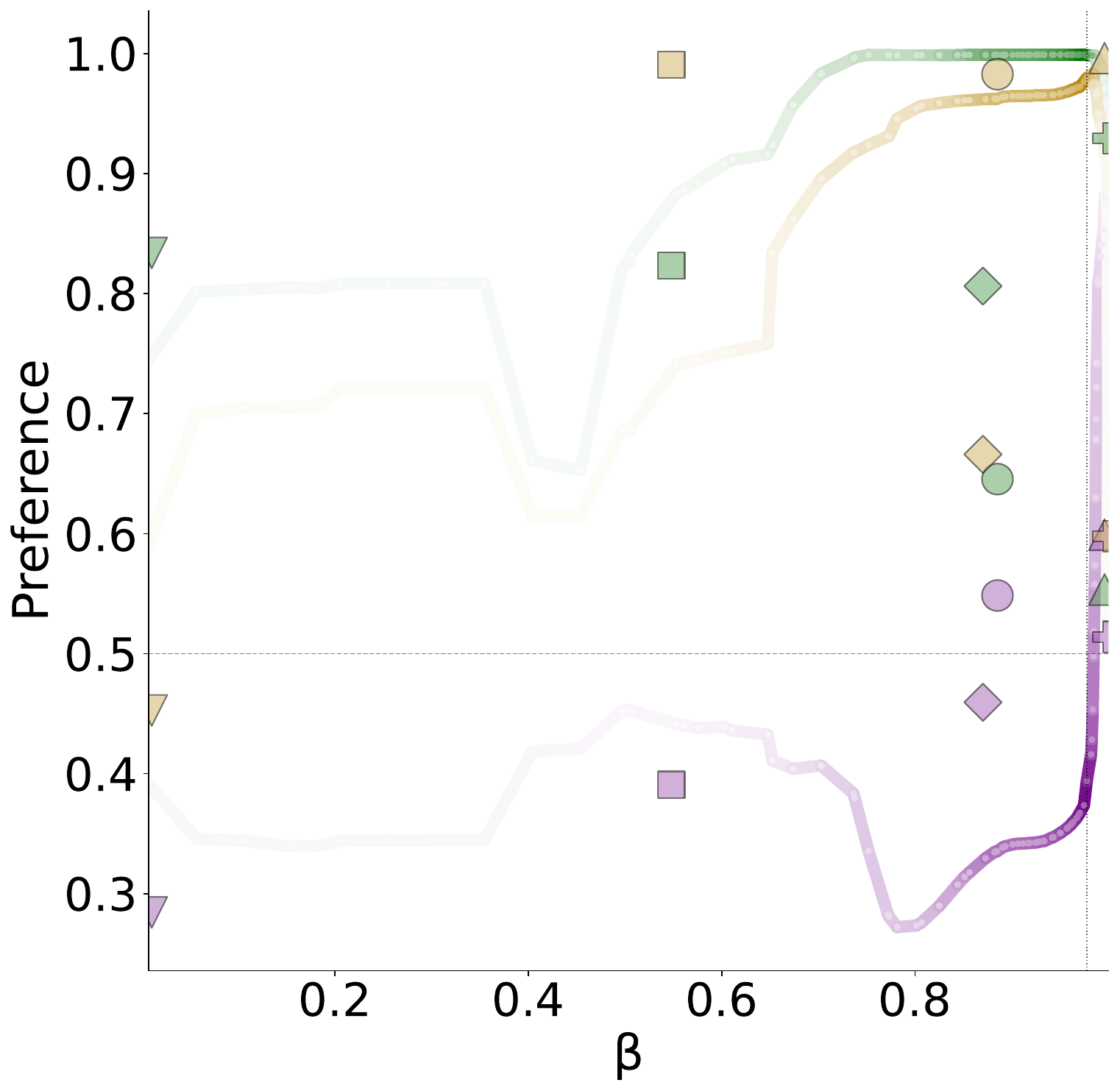}
\end{tabular}\hfill
\begin{tabular}{@{}c@{}}
    \quad zhen23 \\
    \includegraphics[height=4.9cm]{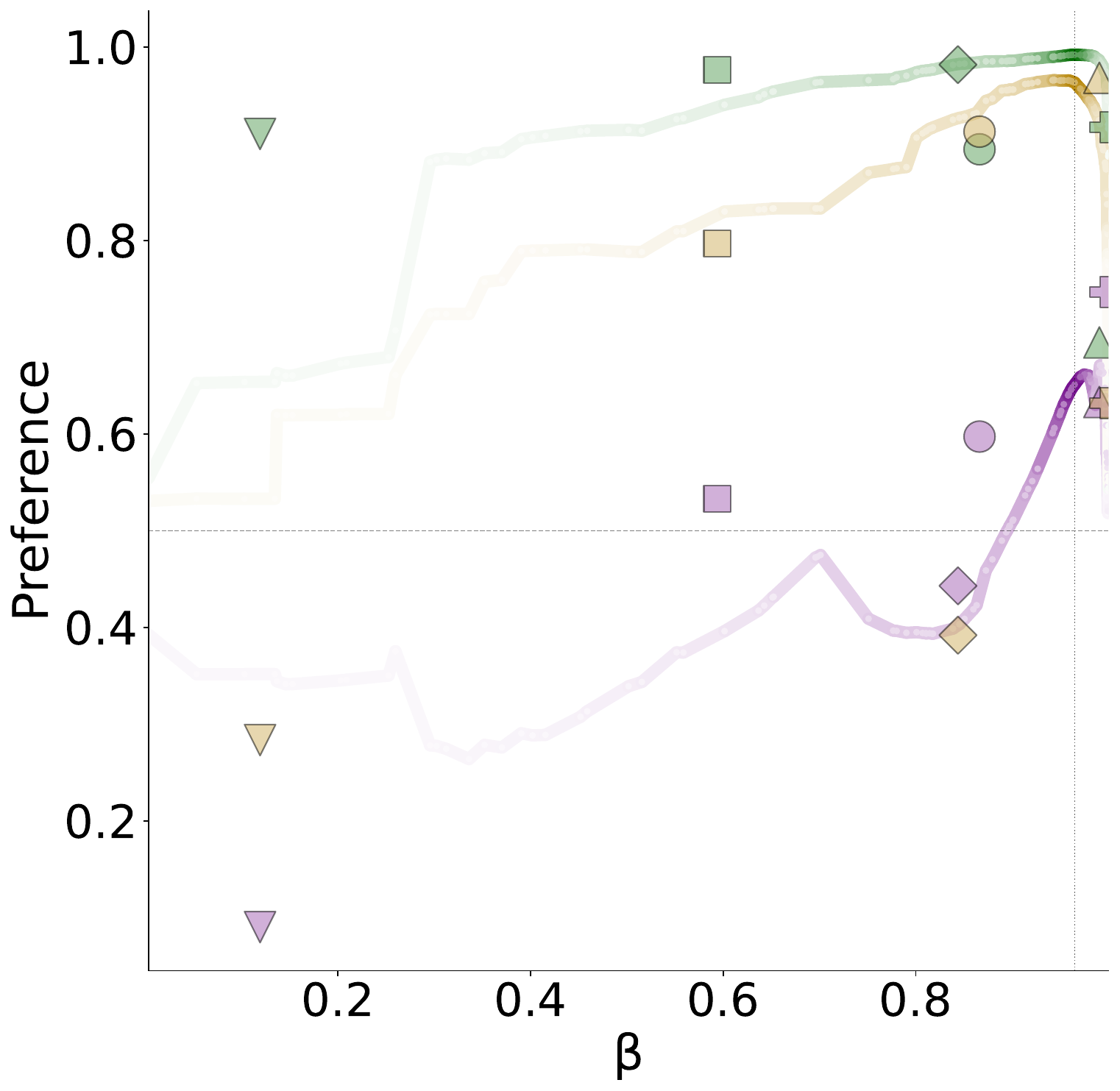}
\end{tabular}

\vspace{0.5cm} 

\begin{tabular}{@{}c@{}}
    \quad ende21 \\
    \includegraphics[height=4.9cm]{figures/preference_spa/orientation__ende21_DiagTJ_synth25_nolegend.pdf}
\end{tabular}\hfill
\begin{tabular}{@{}c@{}}
    \quad ende24 \\
    \includegraphics[height=4.9cm]{figures/preference_spa/orientation__ende24_DiagTJ_synth25.pdf}
\end{tabular}

\caption{Main results on all datasets. Layout follows Figure~\ref{fig:3}.}
\label{fig:results_all}
\end{figure*}

\section{Main Results for Pearson Correlation}
\label{app:pearson}

Figure~\ref{fig:results_pearson} mirrors the experiments from \S\ref{sec:results}, using Pearson correlation as the meta-metric rather than SPA. Markers representing \citet{shayegh-etal-2025-feeding} are omitted, as Pearson correlation falls outside the scope of their work. These findings align closely with our SPA evaluation, confirming the robustness of our approach.

\begin{figure*}[t]
\centering 

\begin{tabular}{@{}c@{}}
    \quad ende22 \\
    \includegraphics[height=4.9cm]{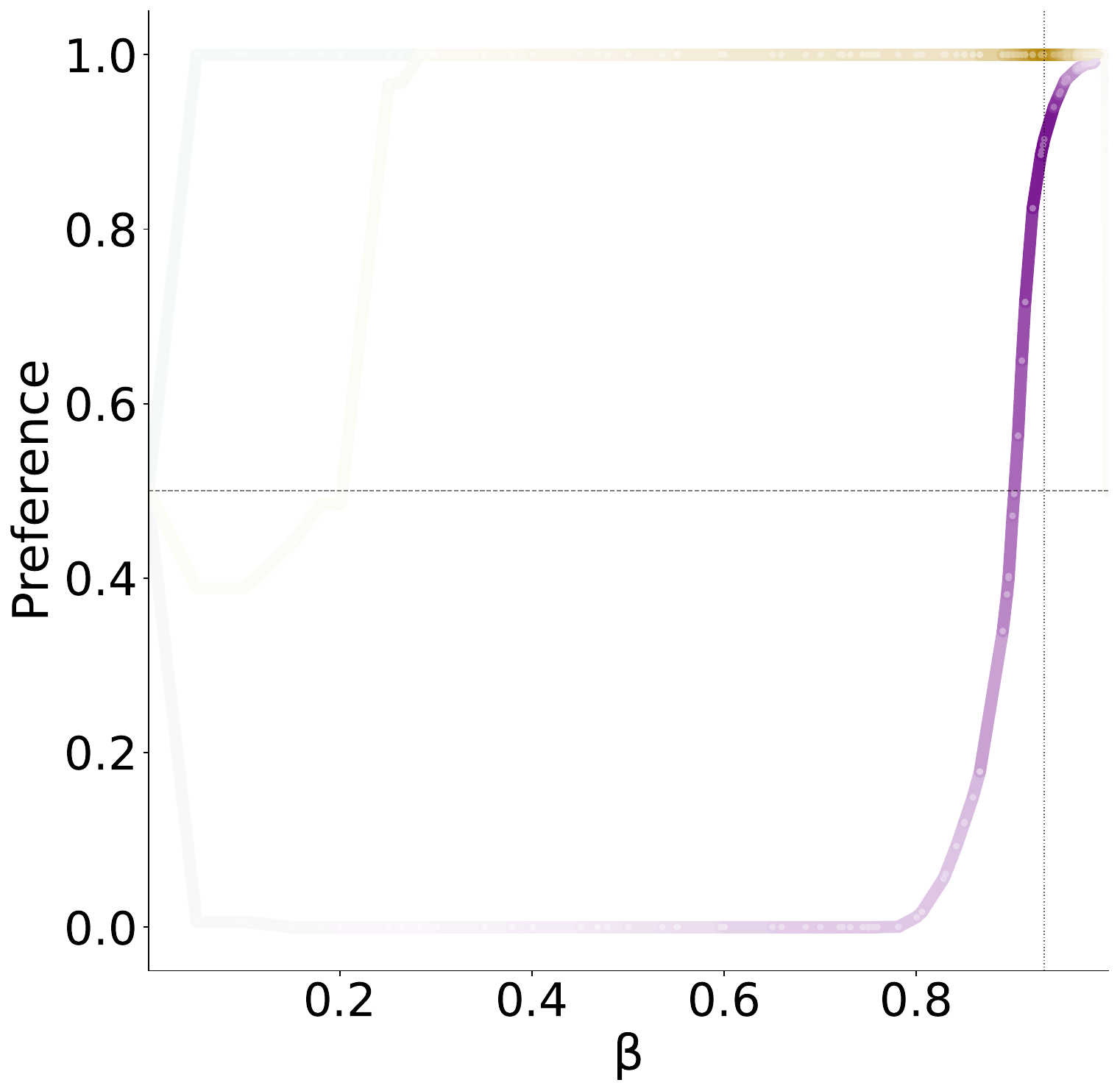}
\end{tabular}\hfill
\begin{tabular}{@{}c@{}}
    \quad zhen21 \\
    \includegraphics[height=4.9cm]{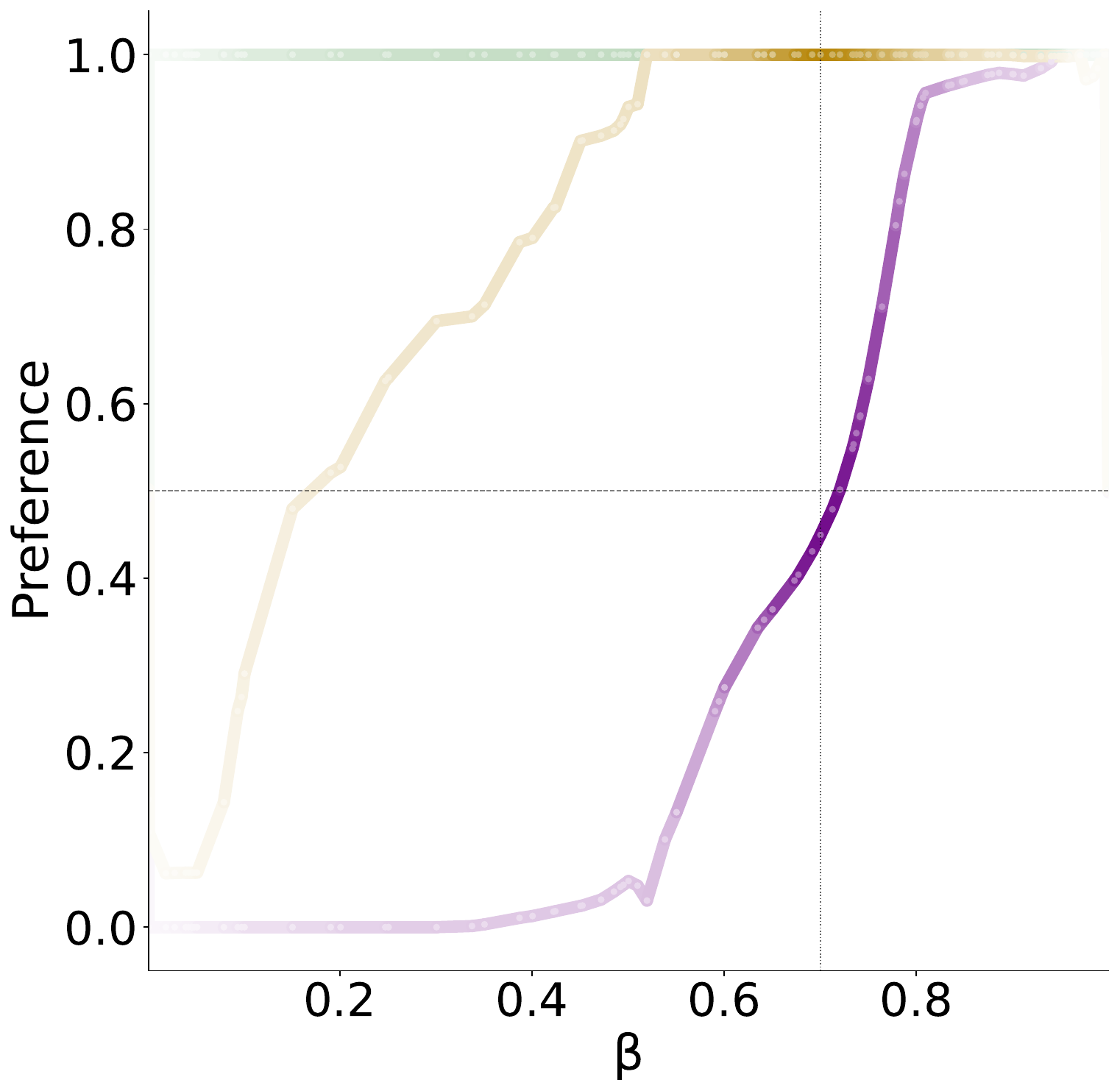}
\end{tabular}\hfill
\begin{tabular}{@{}c@{}}
    \quad enes24 \\
    \includegraphics[height=4.9cm]{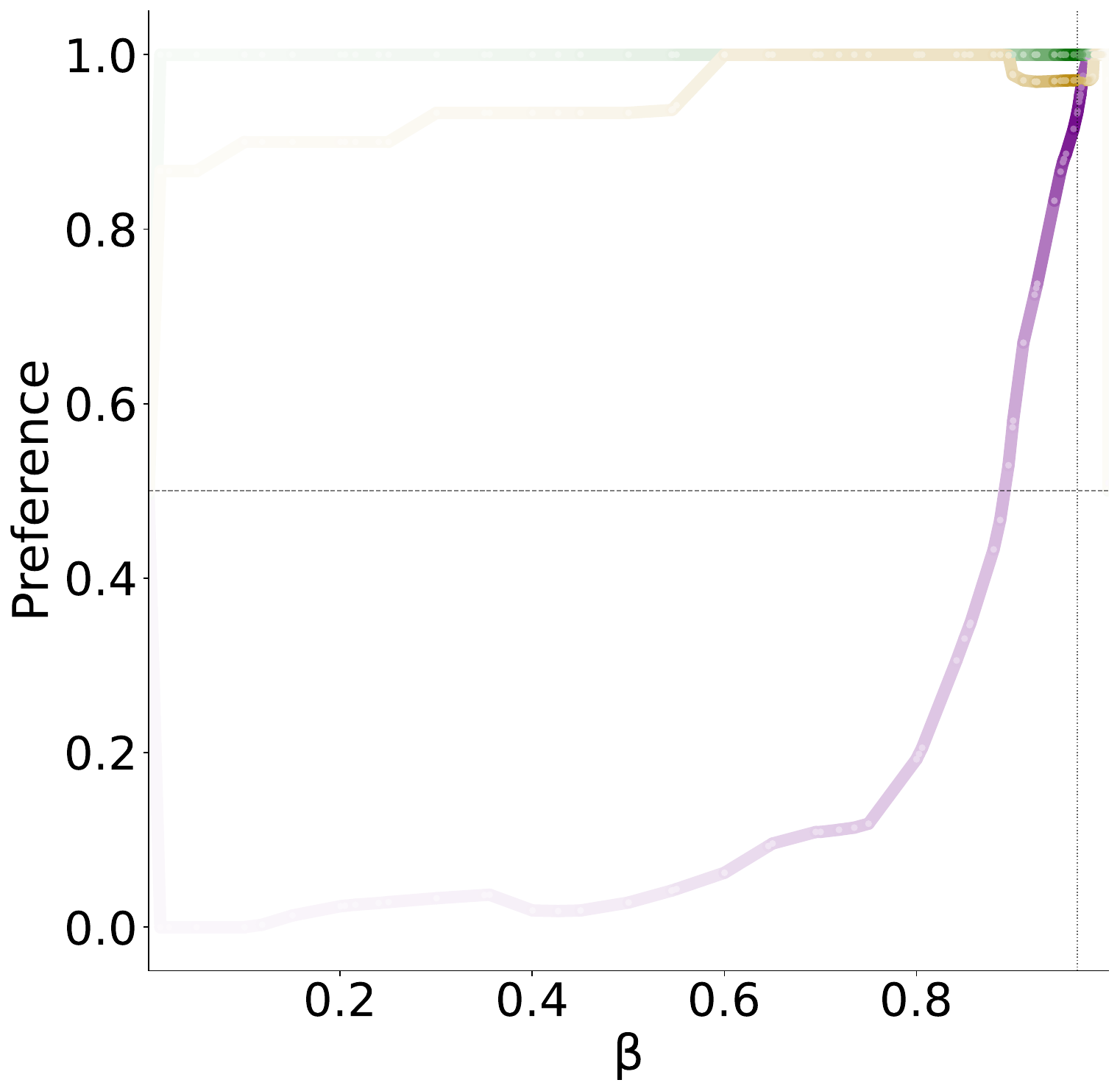}
\end{tabular}

\vspace{0.5cm} 

\begin{tabular}{@{}c@{}}
    \quad heen23 \\
    \includegraphics[height=4.9cm]{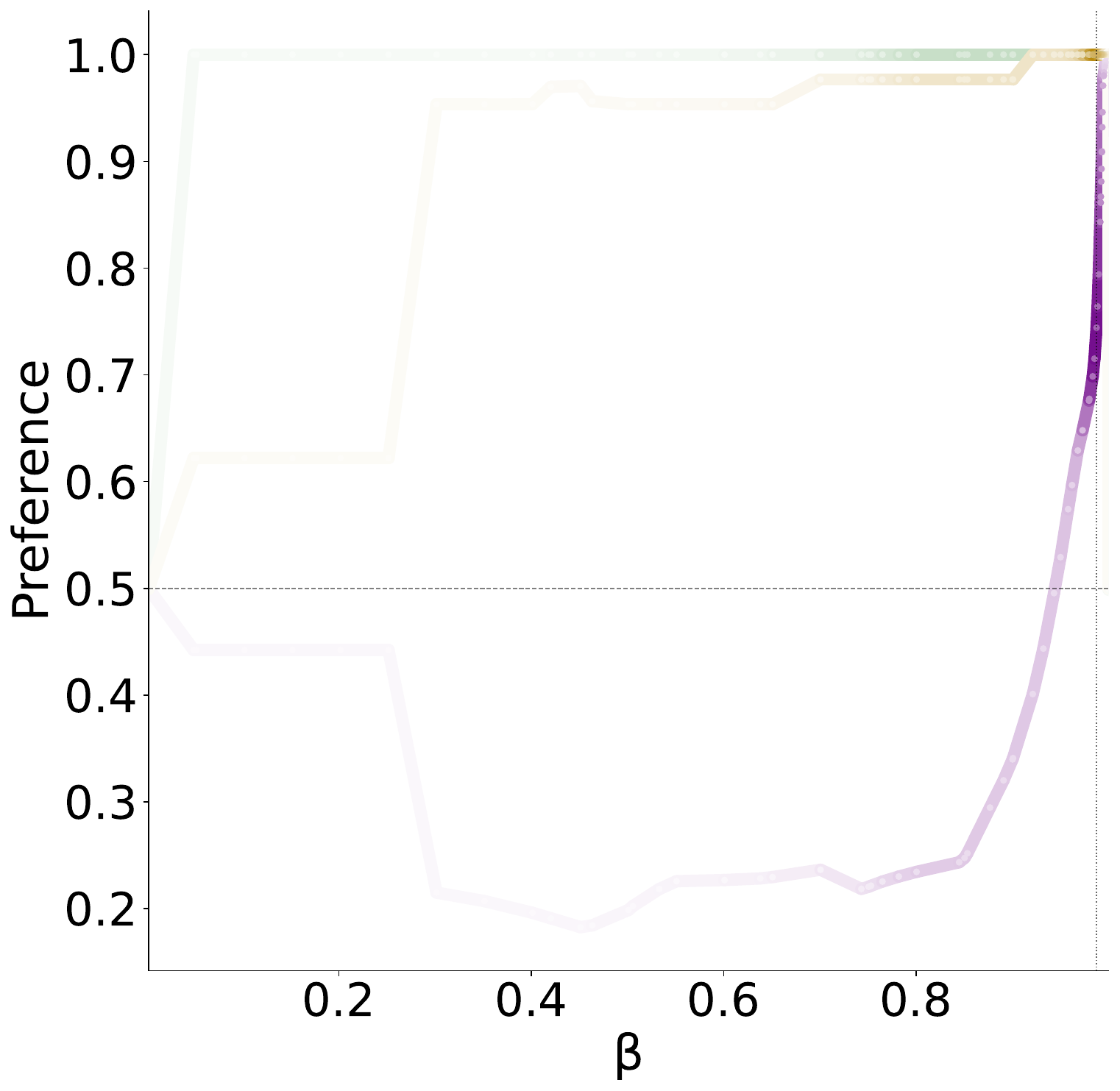}
\end{tabular}\hfill
\begin{tabular}{@{}c@{}}
    \quad jazh24 \\
    \includegraphics[height=4.9cm]{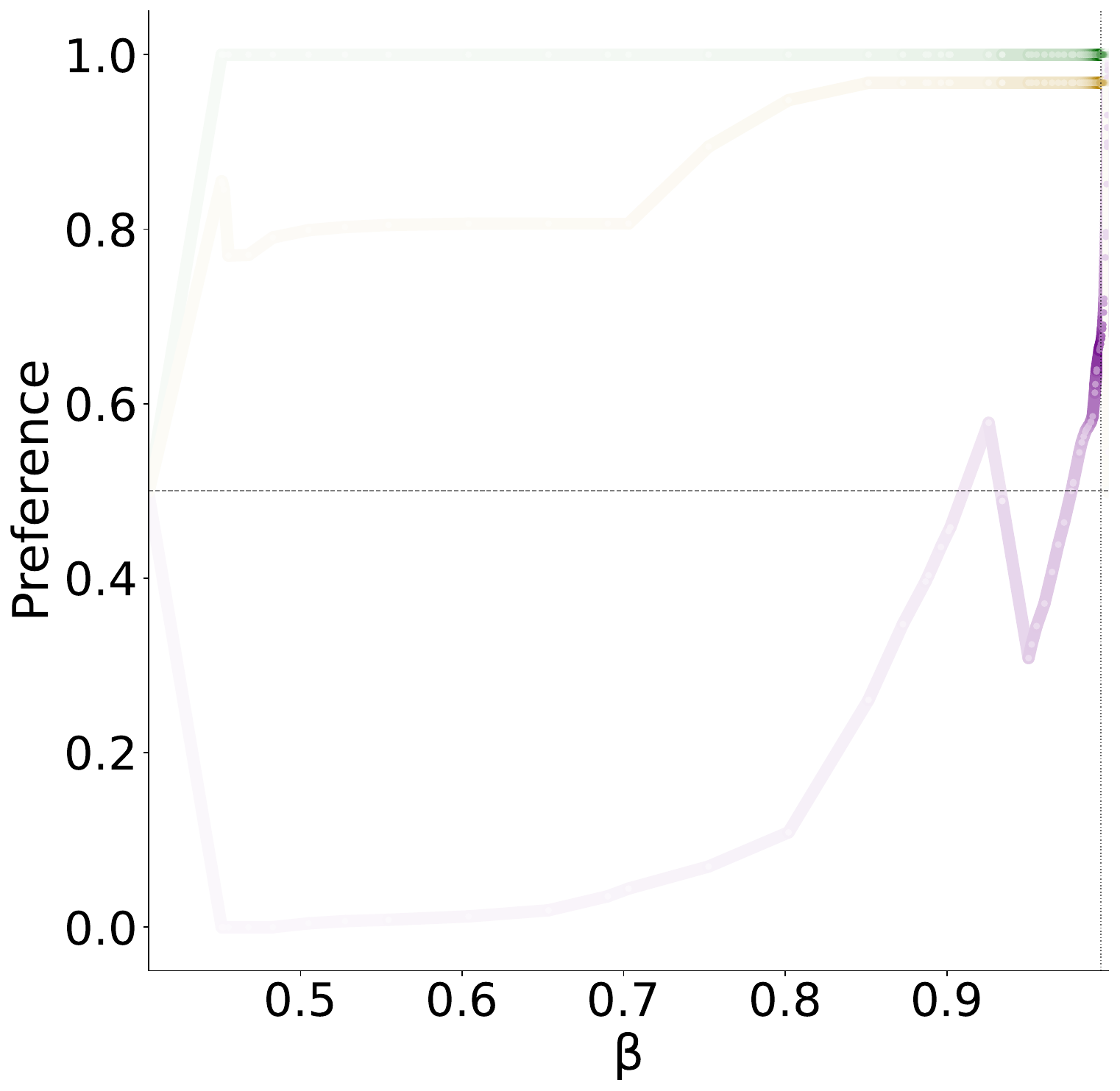}
\end{tabular}\hfill
\begin{tabular}{@{}c@{}}
    \quad ted\_ende \\
    \includegraphics[height=4.9cm]{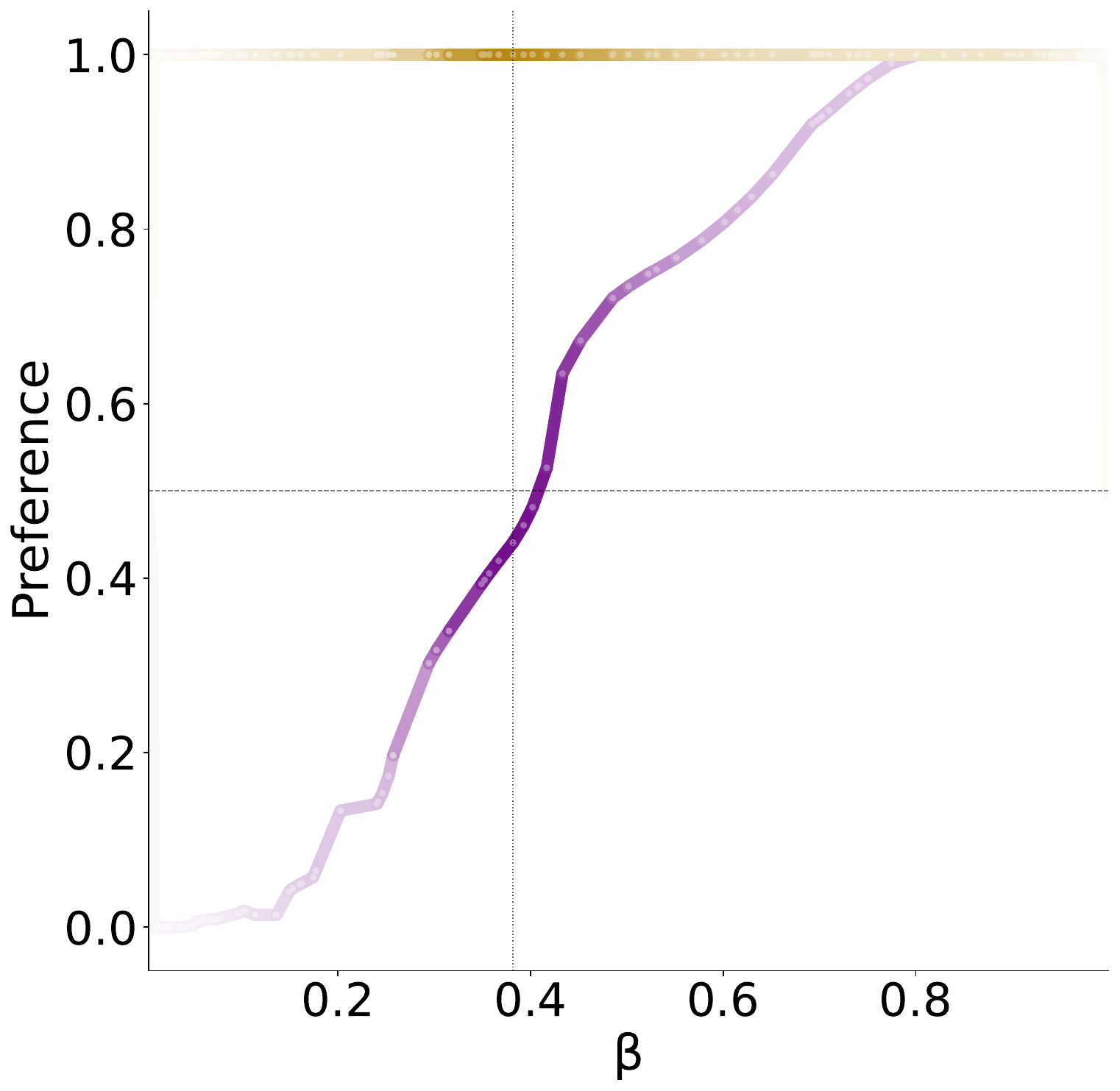}
\end{tabular}

\vspace{0.5cm} 

\begin{tabular}{@{}c@{}}
    \quad ted\_zhen \\
    \includegraphics[height=4.9cm]{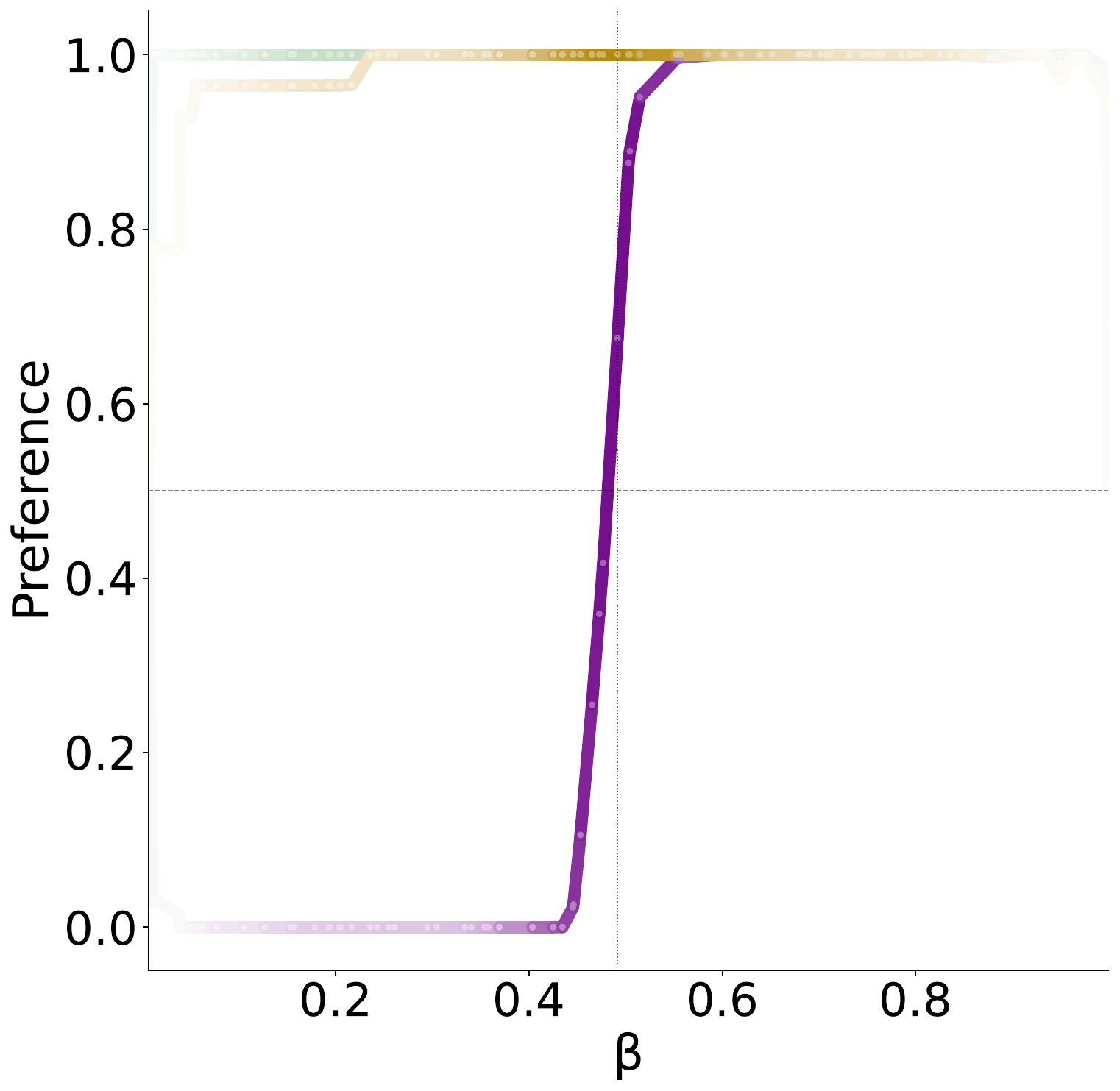}
\end{tabular}\hfill
\begin{tabular}{@{}c@{}}
    \quad zhen22 \\
    \includegraphics[height=4.9cm]{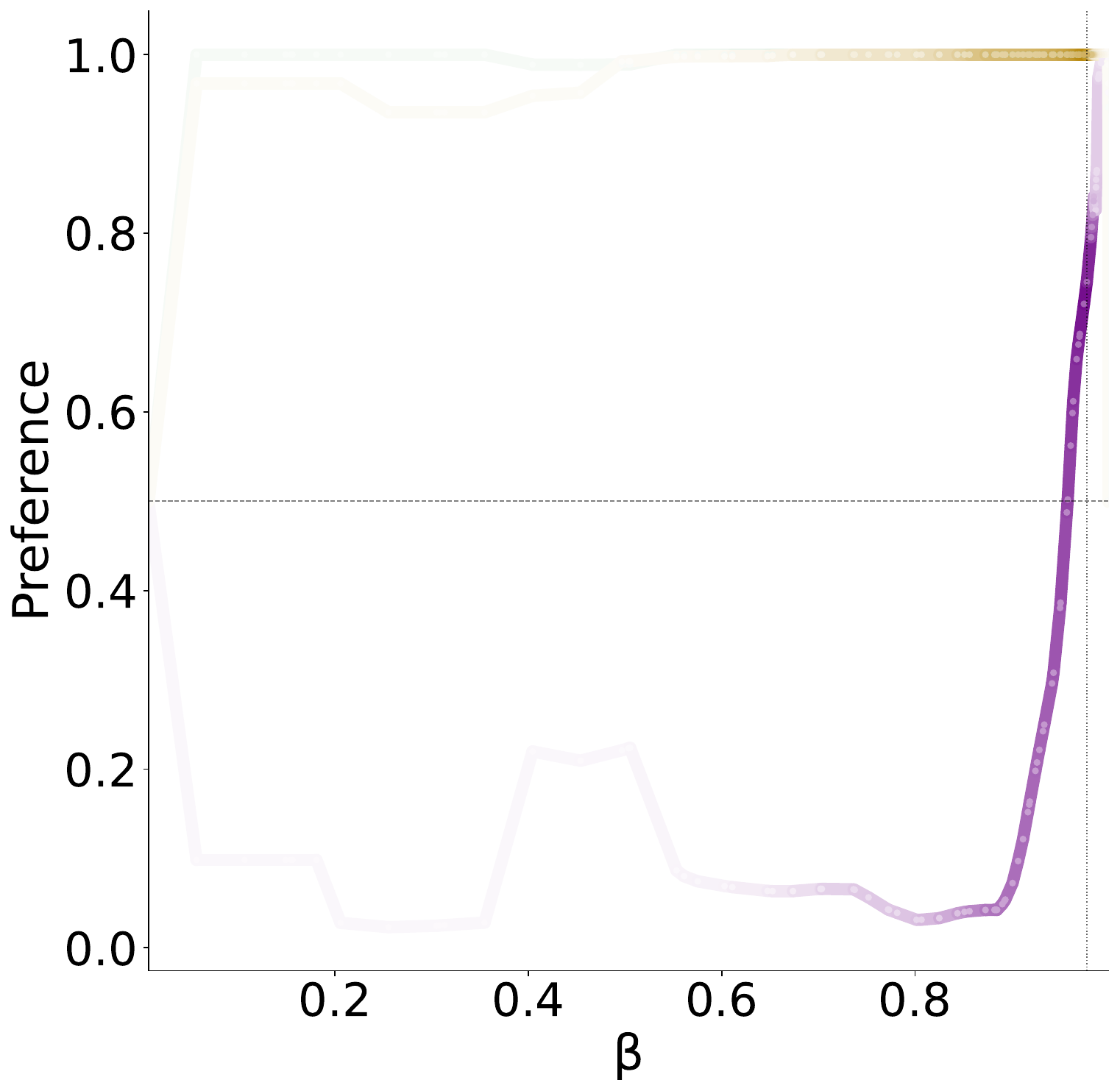}
\end{tabular}\hfill
\begin{tabular}{@{}c@{}}
    \quad zhen23 \\
    \includegraphics[height=4.9cm]{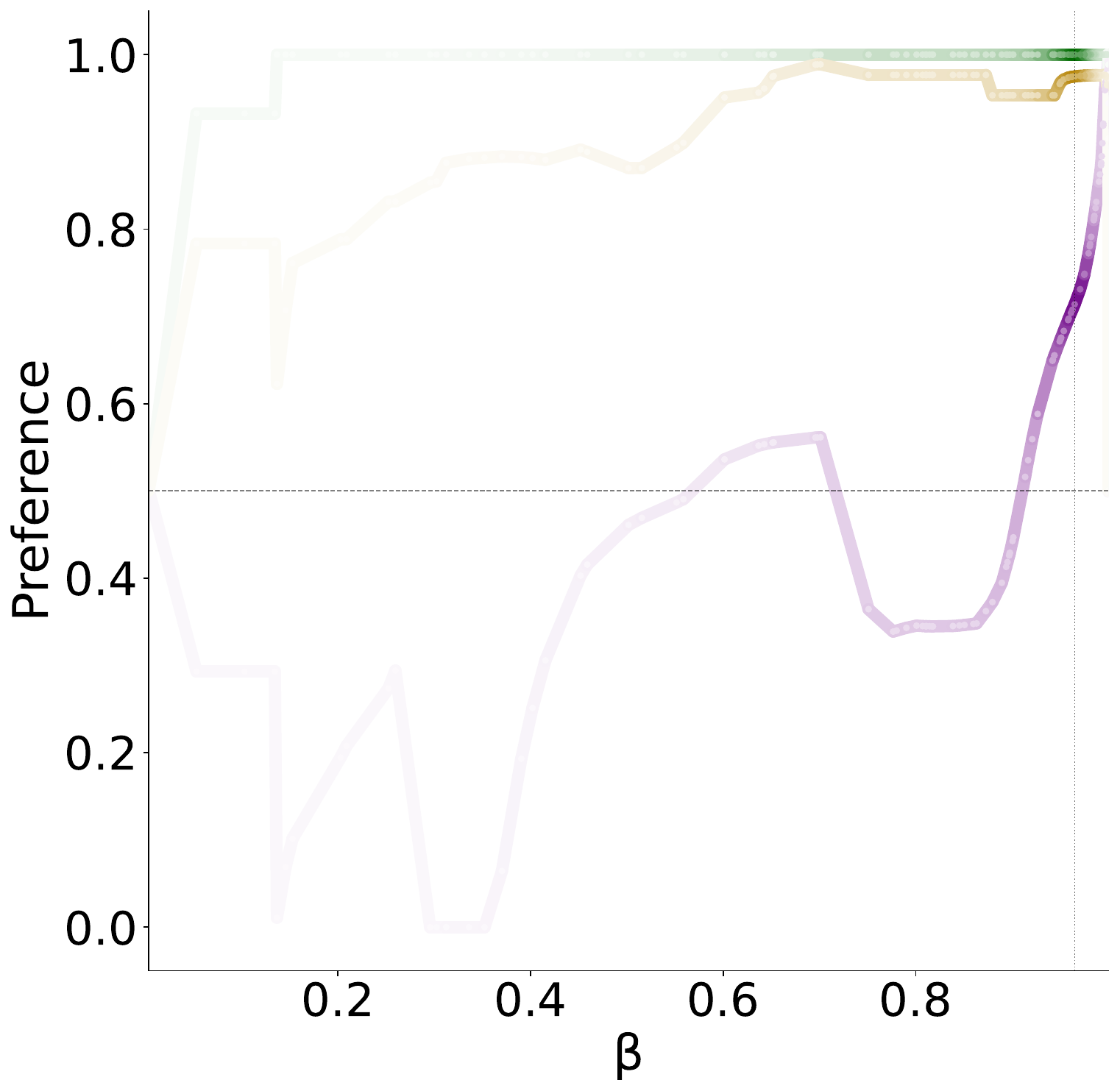}
\end{tabular}

\vspace{0.5cm} 

\begin{tabular}{@{}c@{}}
    \quad ende21 \\
    \includegraphics[height=4.9cm]{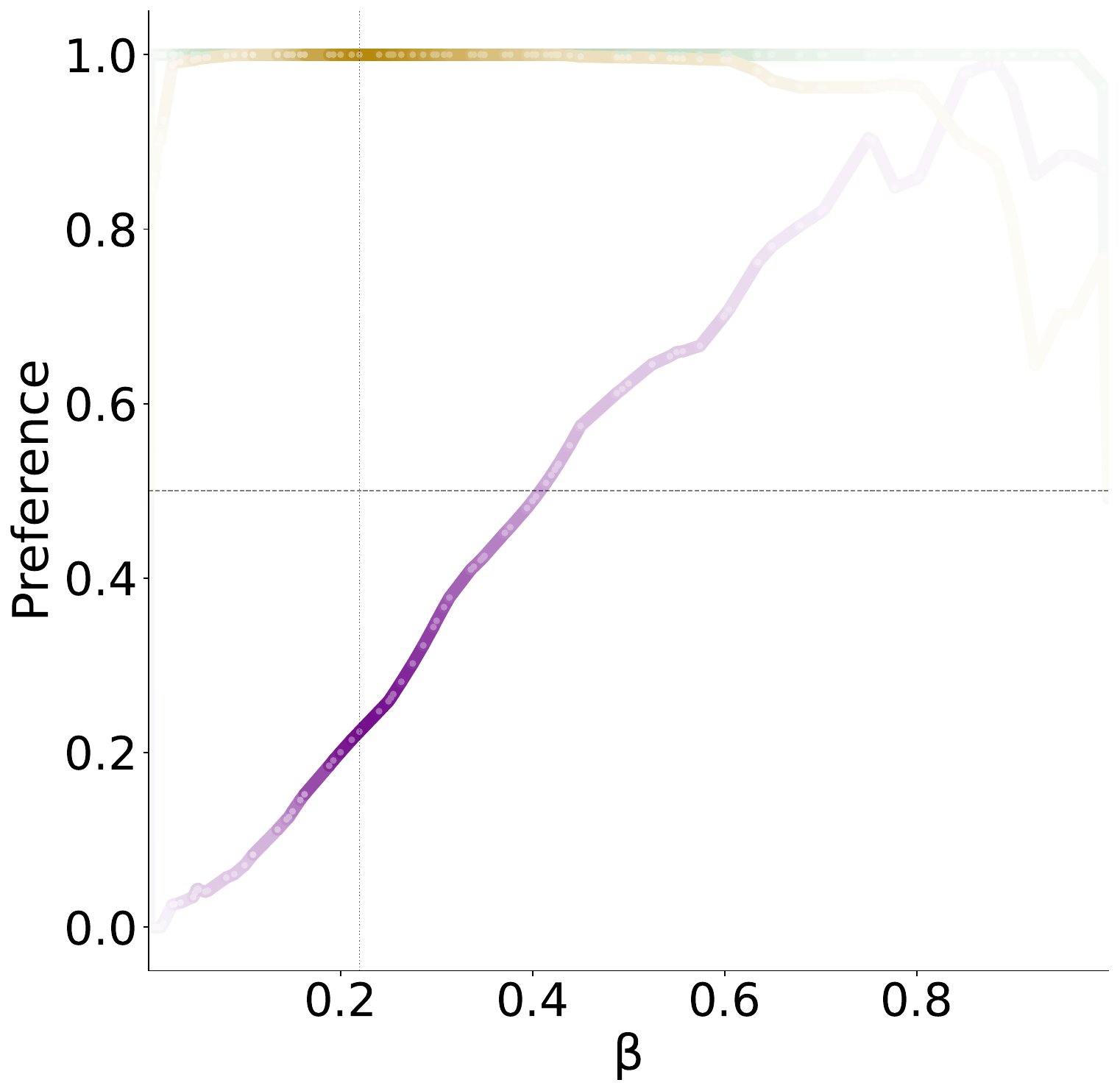}
\end{tabular}\hfill
\begin{tabular}{@{}c@{}}
    \quad ende24 \\
    \includegraphics[height=4.9cm]{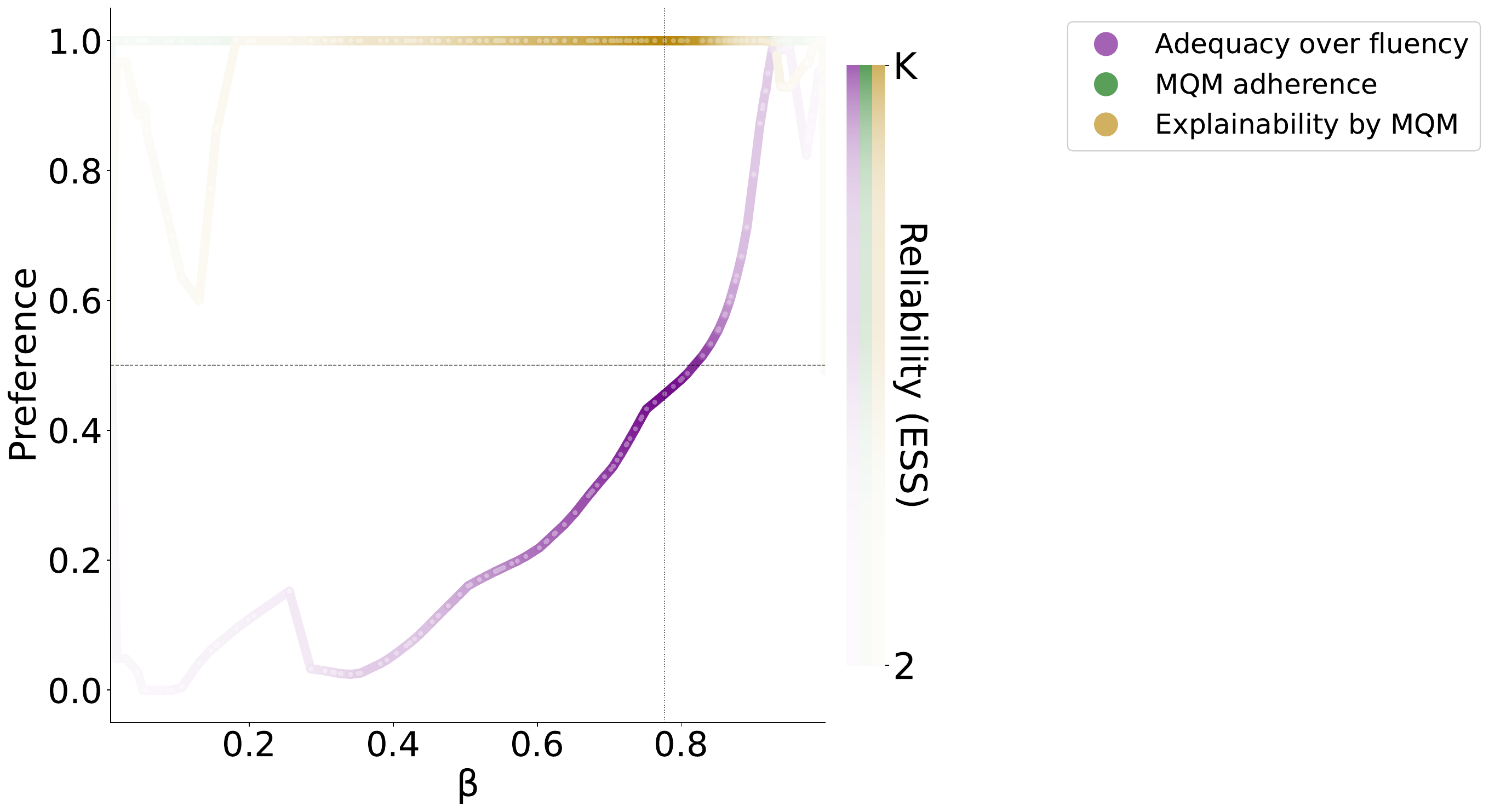}
\end{tabular}

\caption{Main results with Pearson as the meta-metric. Layout follows Figure~\ref{fig:3}.}
\label{fig:results_pearson}
\end{figure*}

\section{Variance Analysis for All Datasets}
\label{app:all_heatmaps}

Figure~\ref{fig:all_heatmaps} presents the variance analysis heatmaps across all datasets, reproducing the figures from the main text for completeness. The observed trends remain consistent throughout, demonstrating the reliability of $\beta$ as a control parameter.

\begin{figure*}[t]
\centering 

\begin{tabular}{@{}c@{}}
    \quad ende22 \\
    \includegraphics[height=7.0cm]{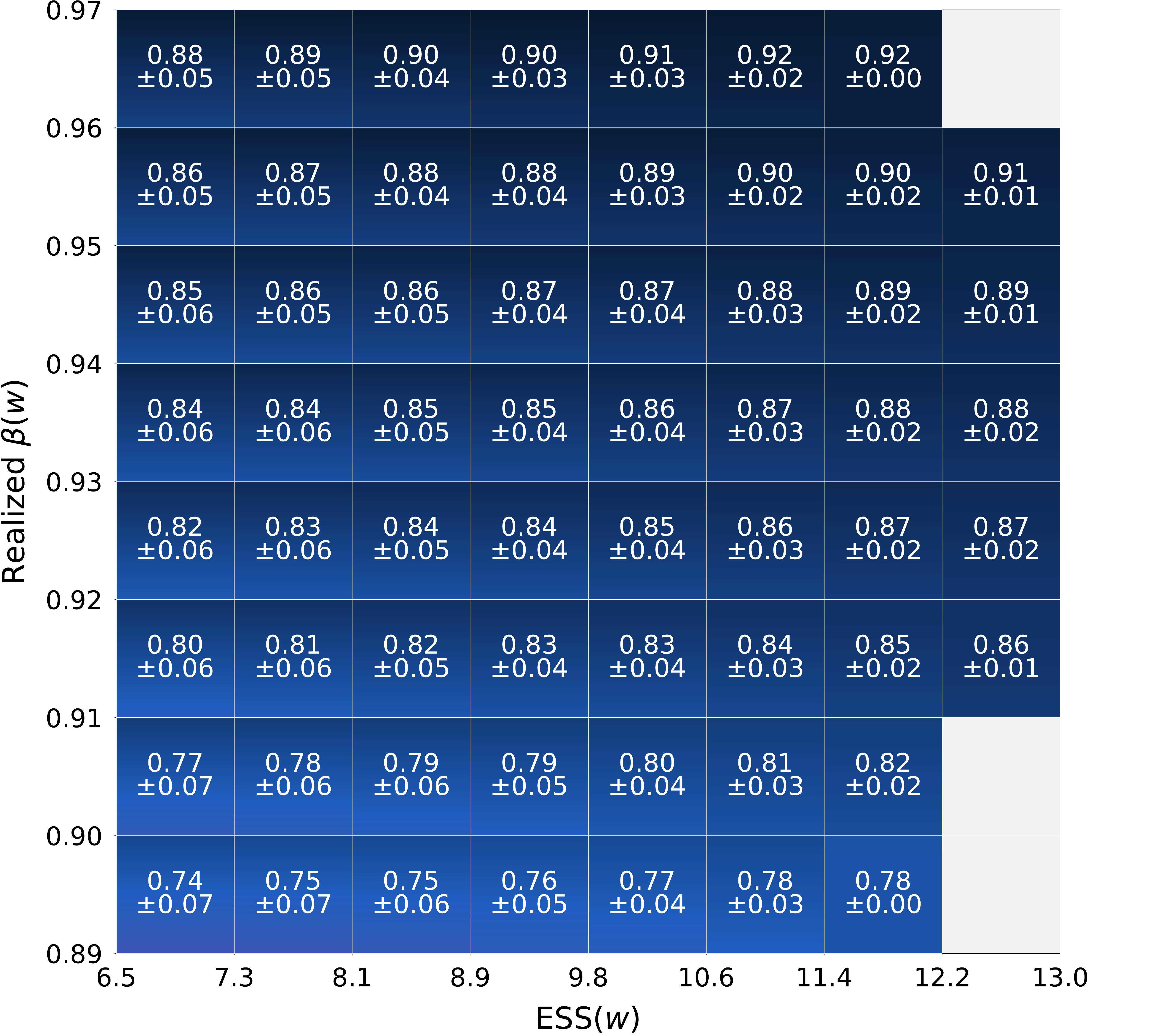}
\end{tabular}\hfill
\begin{tabular}{@{}c@{}}
    \quad zhen21 \\
    \includegraphics[height=7.0cm]{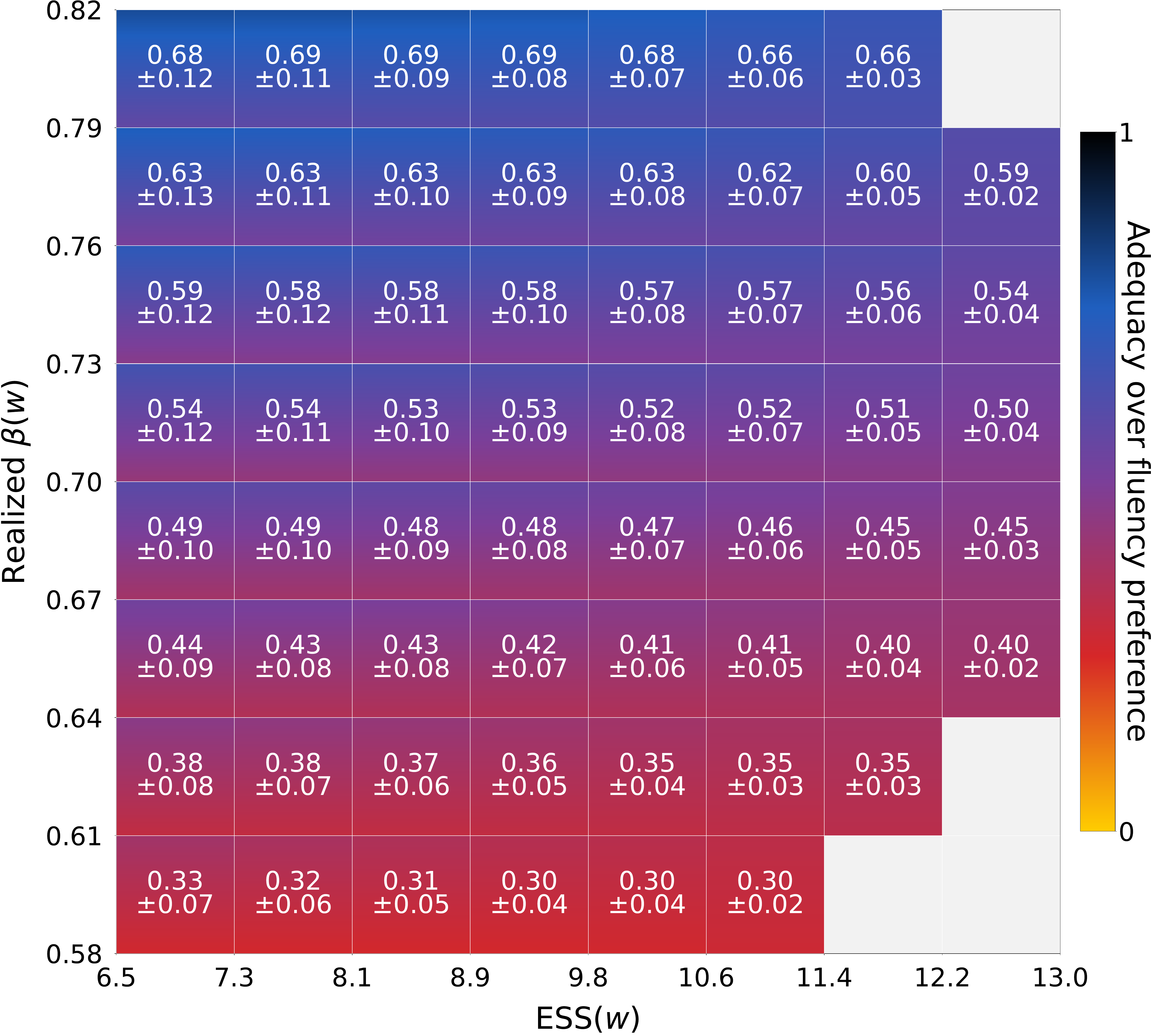}
\end{tabular}

\vspace{0.5cm} 

\begin{tabular}{@{}c@{}}
    \quad heen23 \\
    \includegraphics[height=7.0cm]{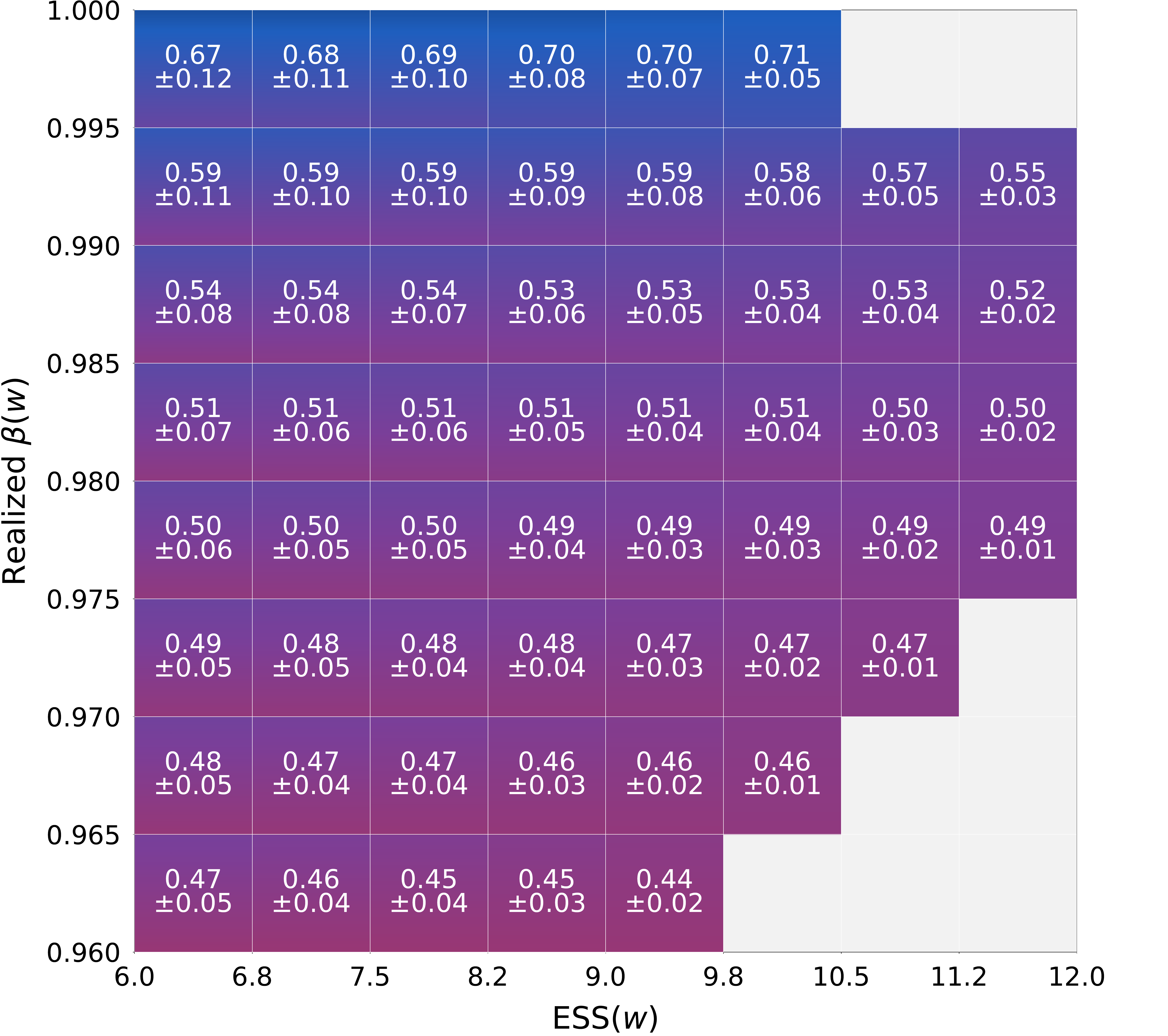}
\end{tabular}\hfill
\begin{tabular}{@{}c@{}}
    \quad jazh24 \\
    \includegraphics[height=7.0cm]{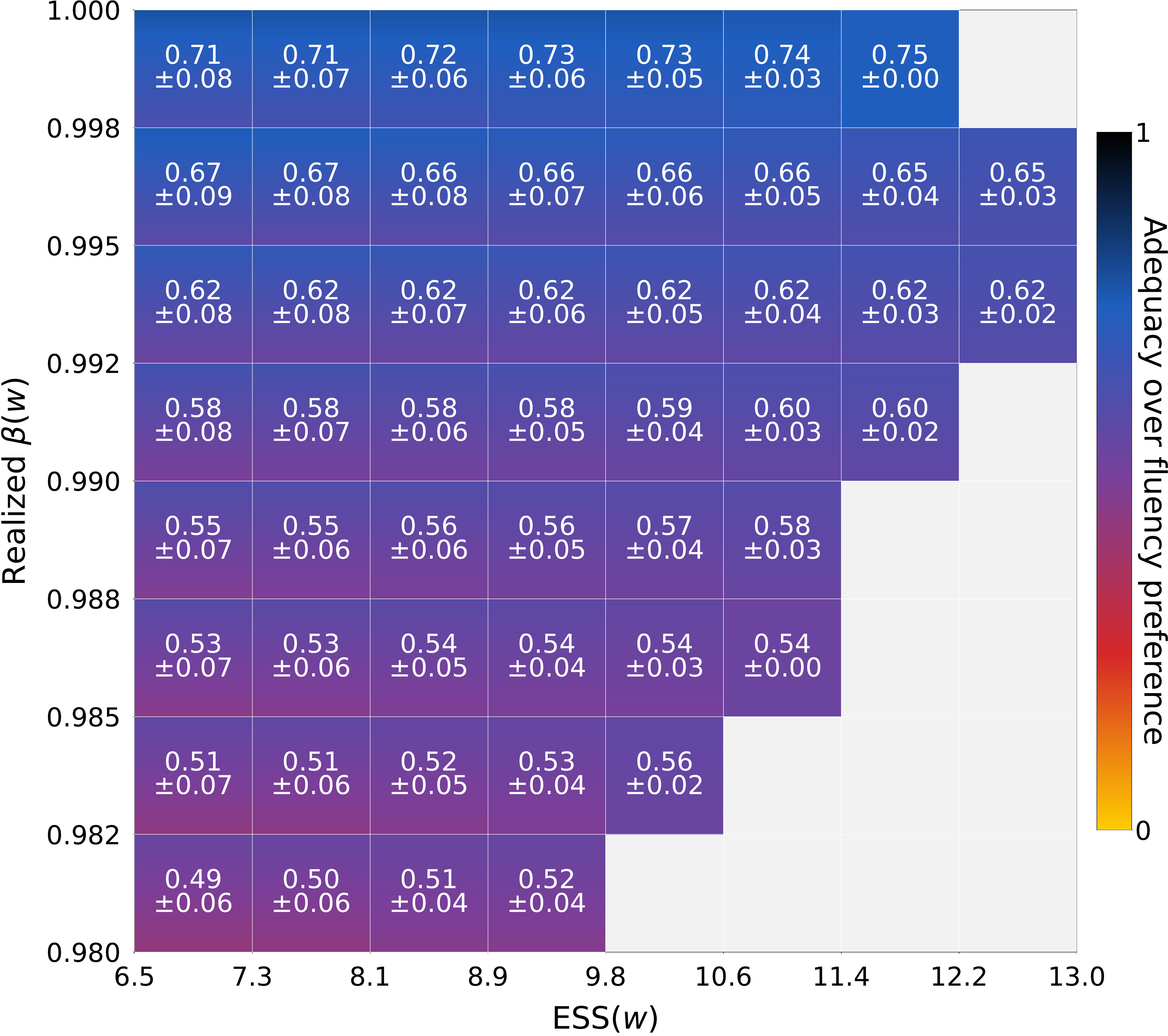}
\end{tabular}

\vspace{0.5cm} 

\begin{tabular}{@{}c@{}}
    \quad ted\_ende \\
    \includegraphics[height=7.0cm]{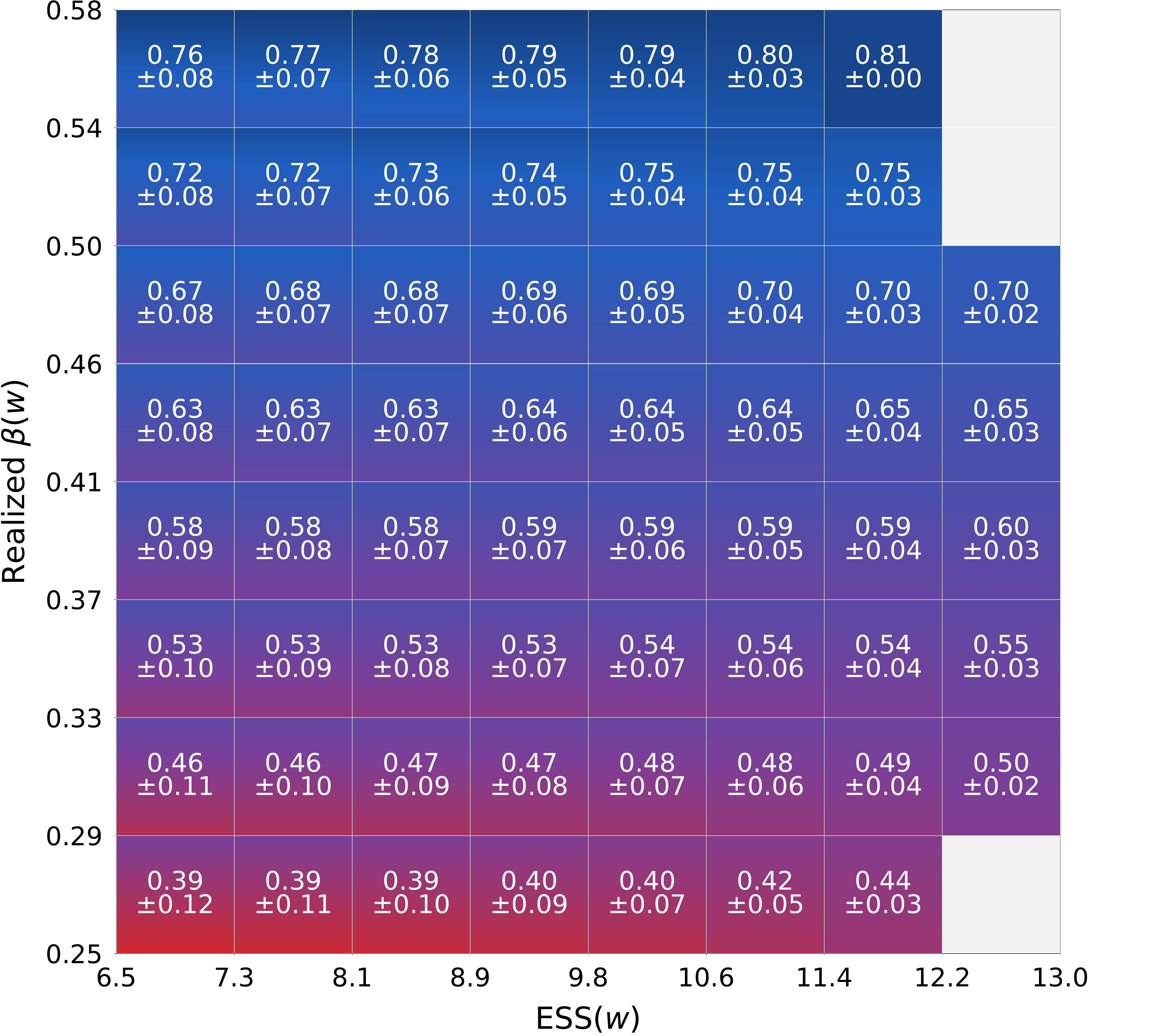}
\end{tabular}\hfill
\begin{tabular}{@{}c@{}}
    \quad enes24 \\
    \includegraphics[height=7.0cm]{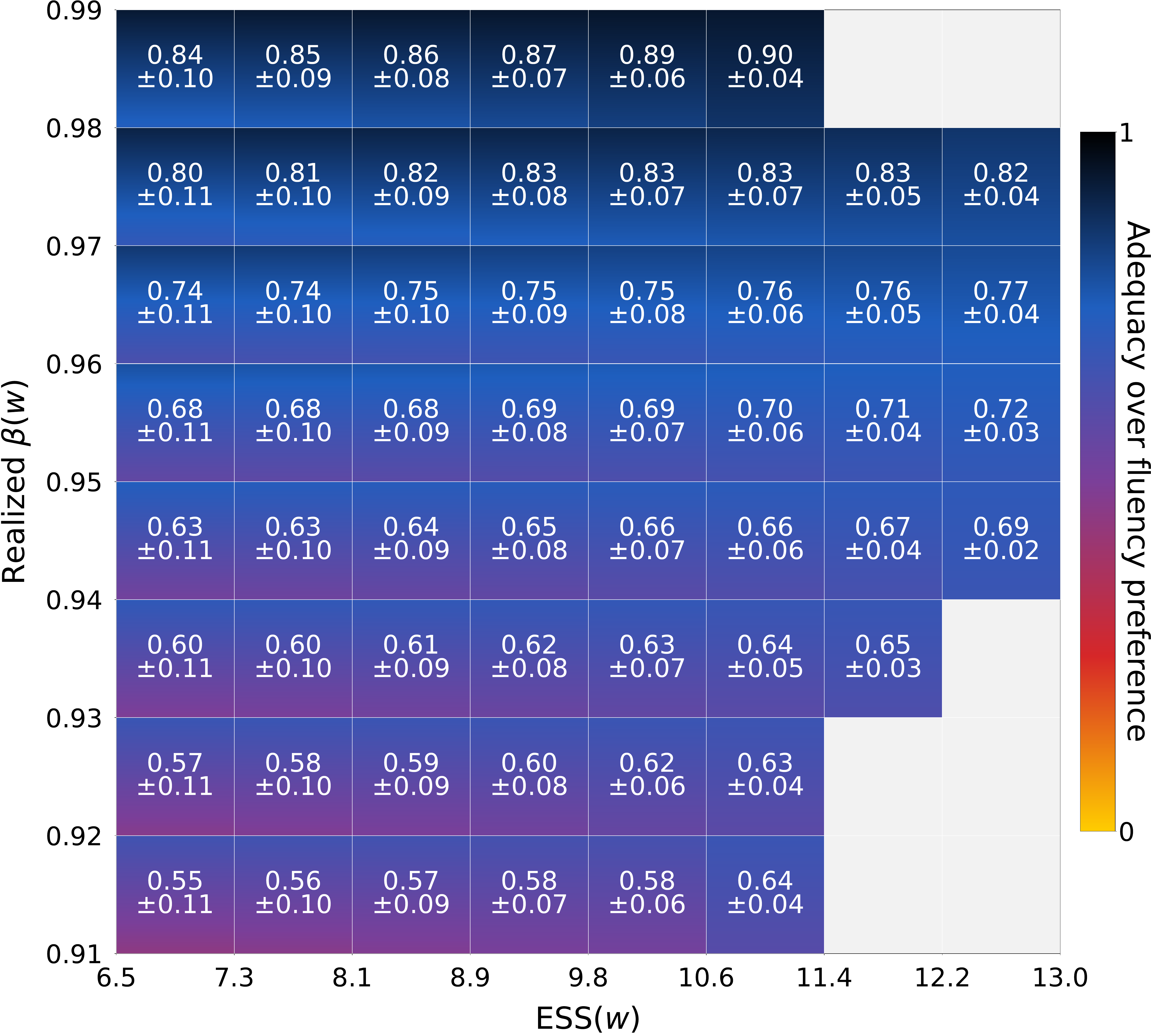}
\end{tabular}

\end{figure*}

\begin{figure*}[t]
\centering 

\begin{tabular}{@{}c@{}}
    \quad ende24 \\
    \includegraphics[height=7.0cm]{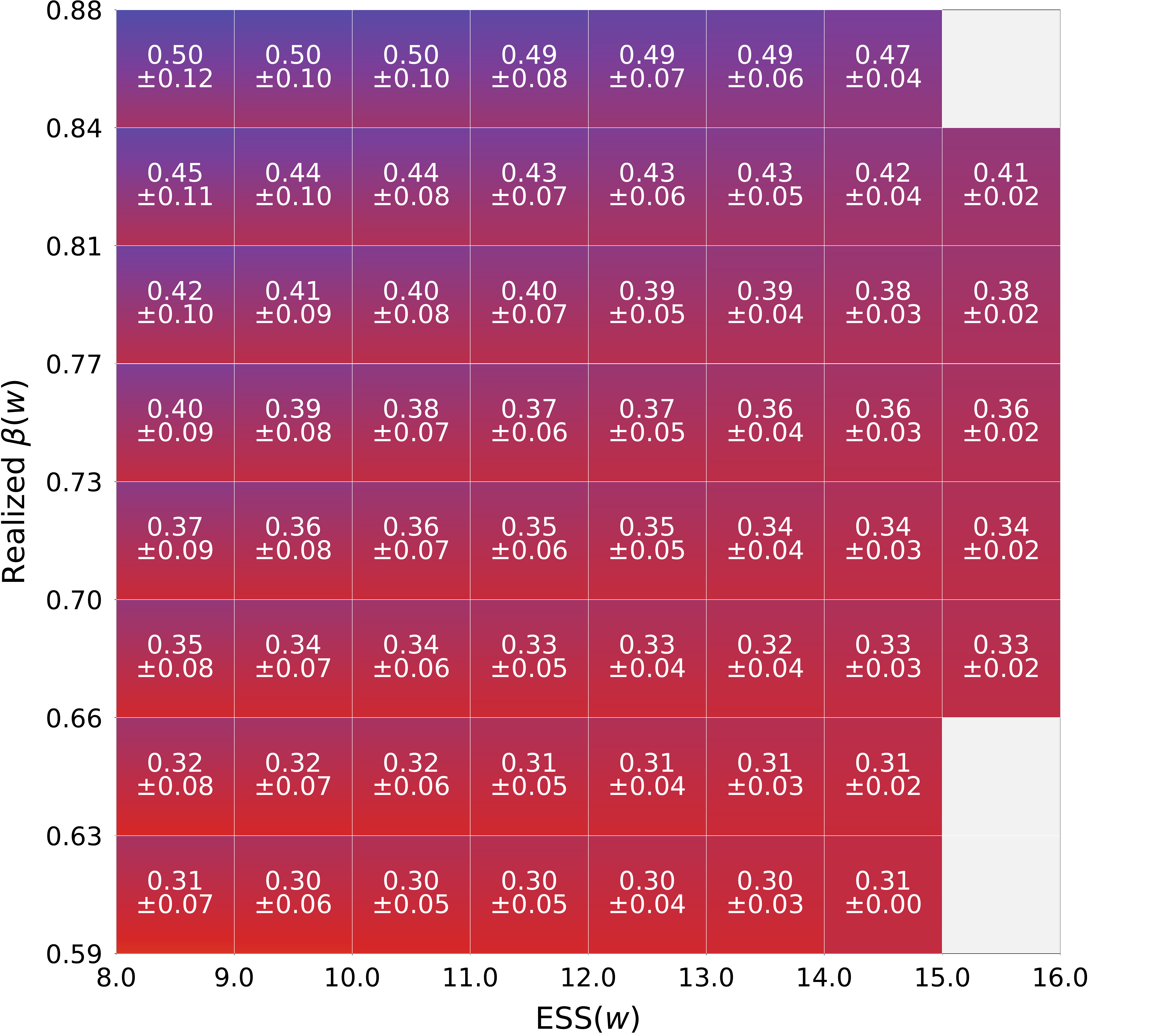}
\end{tabular}\hfill
\begin{tabular}{@{}c@{}}
    \quad zhen22 \\
    \includegraphics[height=7.0cm]{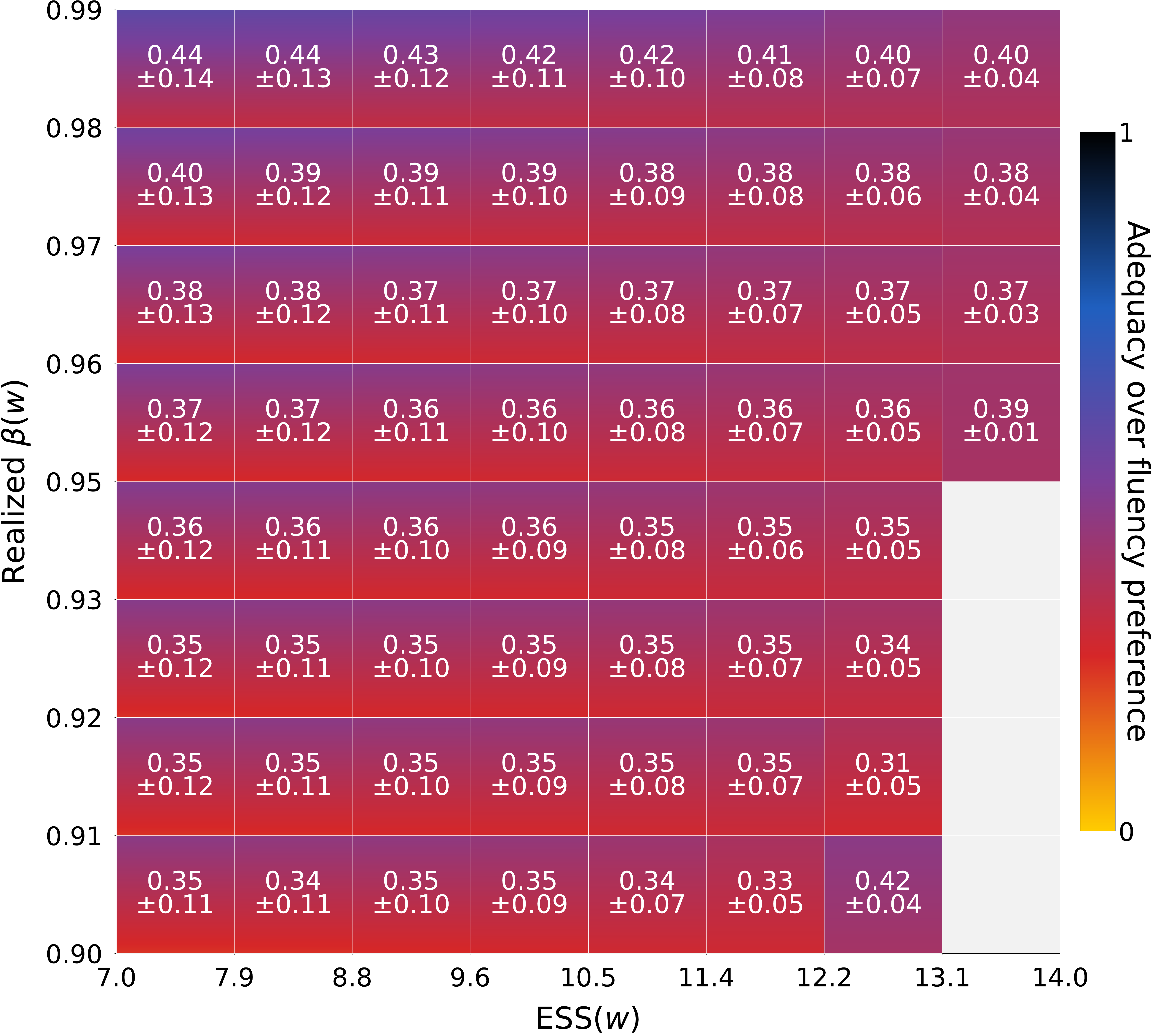}
\end{tabular}

\vspace{0.5cm} 

\begin{tabular}{@{}c@{}}
    \quad ende21 \\
    \includegraphics[height=7.0cm]{figures/beta_ess_preference_heatmap/beta_ess_preference_ende21_simplex4000k.pdf}
\end{tabular}\hfill
\begin{tabular}{@{}c@{}}
    \quad zhen23 \\
    \includegraphics[height=7.0cm]{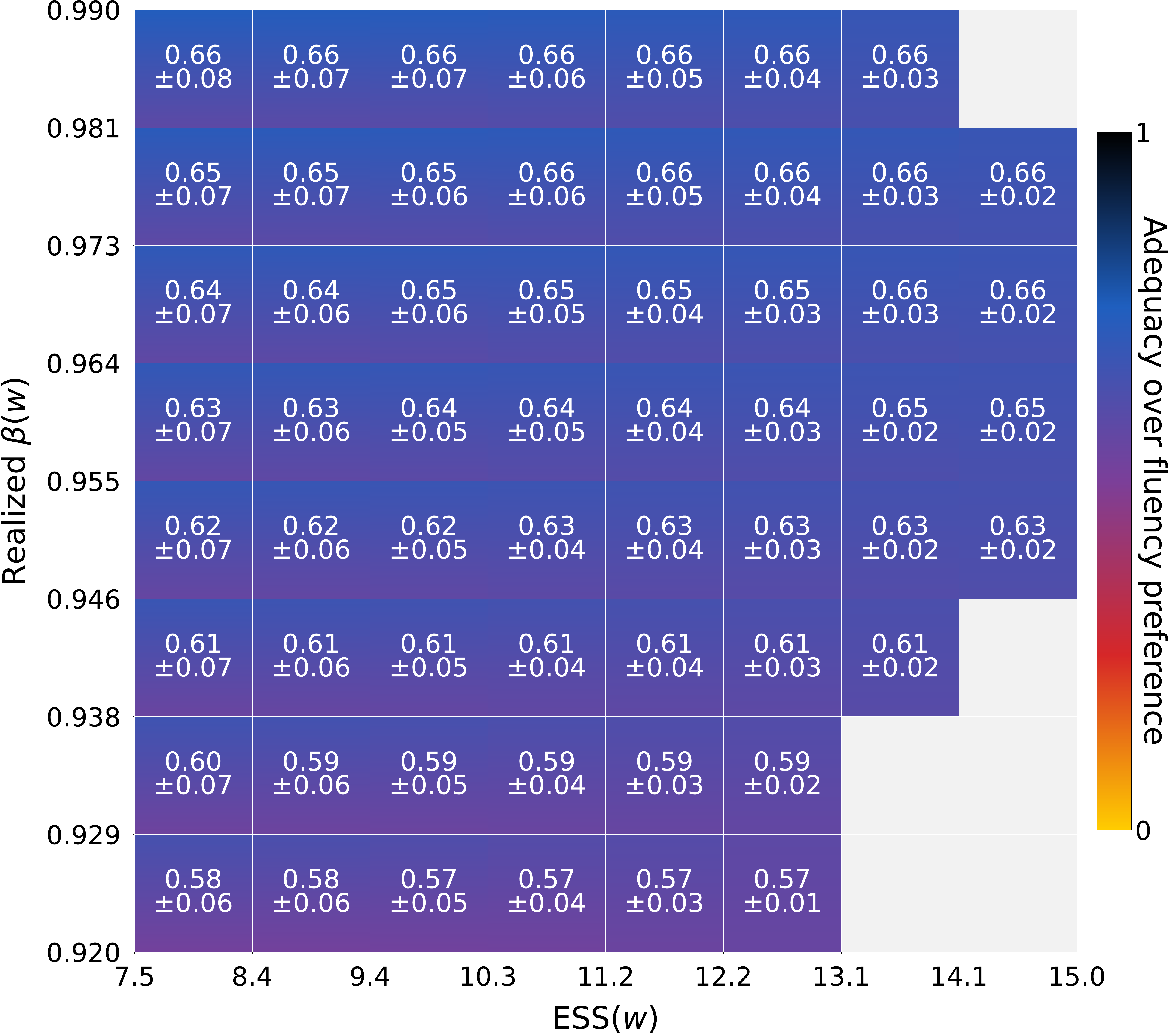}
\end{tabular}

\vspace{0.5cm} 


\begin{tabular}{@{}c@{}}
    \quad ted\_zhen \\
    \includegraphics[height=7.0cm]{figures/beta_ess_preference_heatmap/beta_ess_preference_ted_zhen_simplex4000k_colorbar.pdf}
\end{tabular}

\caption{Variance analysis results for all datasets. Layout follows Figure~\ref{fig:heatmap}.}
\label{fig:all_heatmaps}
\end{figure*}

\end{document}